\documentclass[11pt,letterpaper]{article}

\usepackage[T1]{fontenc}
\usepackage{newpxtext}   
\usepackage{newpxmath}   

\usepackage[dvipsnames]{xcolor}        
\usepackage{epsfig}
\usepackage[margin=1in]{geometry}

\usepackage{amssymb, amsmath, amsthm, graphicx, subcaption, mathtools, tikz-cd}
\usepackage[section]{algorithm}
\usepackage{enumitem}
\usepackage{nicematrix, tikz}
\usepackage{algpseudocode}

\definecolor{lightblue}{HTML}{0071bc}
\usepackage{hyperref}
\hypersetup{
  colorlinks = true,
  linkcolor  = RoyalBlue,
  citecolor  = lightblue,    
  urlcolor   = Magenta
}

\usepackage[authoryear, round]{natbib}
\DeclareMathOperator*{\argmin}{arg\,min}

\newcommand{\E}{\mathbb{E}}

\newcommand{\Var}{\mathrm{Var}}
\newcommand{\prob}{\mathbb{P}}

\newcommand{\R}{\mathbb{R}}

\newcommand{\N}{\mathbb{N}}

\newcommand{\I}{\mathbb{I}}

\newcommand{\F}{\mathcal{F}}

\newcommand{\regret}{\mathcal{R}}

\newcommand{\muhat}{\widehat{\mu}}

\newcommand{\Ncal}{\mathcal{N}}

\newcommand{\inprob}{\overset{\prob}{\longrightarrow}}
\newcommand{\indist}{\overset{\mathcal{D}}{\longrightarrow}}
\newcommand{\sigmahat}{\widehat{\sigma}}

\newcommand{\real}{\mathbf{R}}

\newcommand{\algo}{\ensuremath{\mathcal{A}}}

\newcommand{\xT}{\overline{p}_T}
\newcommand{\DE}{\Delta_\varepsilon}
\newcommand{\lmep}{{T}}

\newcommand{\Fil}{\mathcal{F}}

\newcommand{\Exs}{\mathbb{E}}

\newcommand{\Etwo}{\mathcal{E}_2(T)}

\newcommand{\stao}{\mathcal{S}_1(T)}
\newcommand{\fil}{\mathcal{F}}

\newcommand{\dist}{\operatorname{d}_{\operatorname{KS}}}

\newcommand{\grad}{\nabla}

\newcommand{\action}{\mathcal{A}}
 \newcommand{\Prob}{\mathbb{P}}
 \newcommand{\Eone}{\mathcal{E}_1(T)}

\newcommand{\indi}{\mathbb{I}}
\newcommand{\xbta}{\overline{p}_{T,a}}
\newcommand{\resiT}{\mathcal{D}_T}
\newcommand{\cuty}{\mathcal{Q}_T}
\newcommand{\hTo}{\gamma_T}
\newcommand{\hTto}{h_T}
\newcommand{\pen}{\lambda_T}
\newcommand{\vars}{\varepsilon_T}
\newcommand{\setpz}{\eta_T}
\newcommand{\xopt}{p^\star_{T}}

\newcommand{\xtas}{p^\star_{T,a}}
\newcommand{\opticlass}{\mathcal{I}}
\newcommand{\rate}{\Psi_T}
\newcommand{\kapo}{\kappa_{1,T}}
\newcommand{\MEone}{\mathcal{M}(T)}

\newcommand{\pbar}{\overline{p}_T}
\newcommand{\pbara}{\overline{p}_{T,a}}
\newcommand{\pstar}{p^\star_T}

\theoremstyle{plain}
\newtheorem{theorem}{Theorem}

\newtheorem{lemma}{Lemma}

\newtheorem{assumption}{Assumption}
\newtheorem*{lemma*}{Lemma}
\colorlet{shadecolor}{gray!10}
\colorlet{lightyellow}{yellow!10}

\title{ \textbf{Stabilizing Bandits using Regularization: \\ Precise Regret and A Quantitative Central Limit Theorem}}
\author{
\large 
Budhaditya Halder$^{\dagger}$ \quad Ishan Sengupta$^{\dagger}$ \quad Koustav Chowdhury$^{\ddagger}$ \\ 
\quad Samya Praharaj $^{\dagger}$
Koulik Khamaru$^{\dagger}$\\[0.3cm]
\normalsize $^{\dagger}$ Department of Statistics, Rutgers University\quad \normalsize $^{\ddagger}$ Indian Statistical Institute, Kolkata
}
 \date{}

\begin{document}
\maketitle

\begin{abstract}
Statistical inference with bandit data presents fundamental challenges owing to adaptive sampling, which violates the independence assumptions underlying classical asymptotic theory. Recent work has identified stability~\citep{laiwei82} as a sufficient condition for valid inference under adaptivity. This paper first provides a refined stability condition, stated in terms of the iterates of an online algorithm, and shows that a large class of regularized stochastic-mirror-descent-style algorithms satisfy it. This refined condition allows us to strengthen the asymptotic results of~\citet{laiwei82} in several ways. First, we derive a non-asymptotic Berry--Esseen bound for the empirical reward estimates under adaptive sampling. Second, we derive matching non-asymptotic upper and lower bounds on the regret of the proposed algorithm, yielding a precise characterization of its regret. Third, we show that these regularized algorithms preserve asymptotic normality and valid inference under a prescribed level of adversarial corruption. Finally, we show that regularization is necessary rather than incidental: Lai--Wei stability is incompatible with the optimal $O(\sqrt{T})$ regret rate --- the rate attained by unregularized algorithms such as EXP3 --- so that a controlled, polylogarithmic inflation in regret is the price of valid inference.
\end{abstract}

\section{Introduction}

Sequential decision-making under uncertainty~\citep{wald1992sequential,lai1985asymptotically,bastani2020online} is a fundamental problem in online learning, with the stochastic multi-armed bandit model serving as a canonical framework. Bandit algorithms are widely used in applications such as recommendation systems~\citep{bouneffouf2012contextual,elangovan2021location,islek2022hierarchical}, healthcare ~\citep{bastani2020online,trella2023reward,trella2024oralytics,ghosh2024did,trella2025deployed}, financial portfolio optimization (\citet{shen2015portfolio}, \citet{huo2017risk}) and online advertising \citep{wen2017online,vaswani2017model}, where decisions must be made sequentially while learning unknown reward distributions. The classical objective is to design algorithms that achieve low cumulative regret relative to the best fixed action in hindsight~\citep{Auer2002,bubeck2012regret}.

In many applications, regret minimization alone is insufficient. Practitioners often require reliable statistical inference about the underlying reward distributions---for instance, confidence intervals or hypothesis tests. Adaptive data collection induces dependencies that violate classical i.i.d. assumptions, causing naive estimators to be biased~\citep{nie2018why,shin2021bias} and normality-based uncertainty-quantification procedures to be invalid~\citep{zhang2020inference,praharaj2025instability,han2026thompson}. Recent work has addressed inference under adaptivity using stabilized weighting~\citep{hadad2021confidenceintervalspolicyevaluation,lin2025semiparametric,guo2025statistical}, online debiasing techniques~\citep{deshpande2020accurate,khamaru2025near,ying2023adaptive}, stability based on the number of arm pulls~\citep{han2024ucb,fan2025statistical,praharaj2025avoiding,halder2025stable,yan2026optimism}, or martingale-based corrections~\citep{waudbysmith2022estimatingmeansboundedrandom,abbasi2011improved}. While effective for inference, these approaches break down under corrupted feedback~\citep{jun2018adversarialattacksstochasticbandits}. Robustness to reward corruption is an increasingly important consideration in practice, where feedback may be unreliable owing to logging errors, delayed observations, or strategic manipulation. Existing work on stochastic bandits with adversarial corruption has focused primarily on regret guarantees~\citep{lykouris2018stochasticbanditsrobustadversarial,gupta2019better}, leaving open the question of whether adaptive algorithms can simultaneously support valid statistical inference and robustness to corruption.

\noindent In this paper, we show that an appropriate mirror map and regularizer in a stochastic mirror descent (SMD) formulation inspired by the EXP3 algorithm~\citep{Auer1995Gambling} induce controlled temporal variation in the sampling distribution of each arm. This \emph{stability} plays a central role in enabling reliable statistical inference under adaptive sampling. We further show that this stability persists under adversarial corruption, so that valid inference remains possible without substantially compromising regret.

\subsection{Related work}
\label{sec:related}

A primary source of difficulty for valid statistical inference with bandit data is the bias induced in the distribution of the losses, which can invalidate classical inferential procedures~\citep{xu2013estimation,villar2015multi,nie2018adaptively,shin2019sample,shin2021bias}. To address this problem, several methodological approaches have been proposed, including inverse-propensity-score weighting~\citep{hadad2021confidence,deshpande2018accurate,zhang2020inference,zhang2022statistical,nair2023randomization,leiner2025adaptive,guo2025statistical}, online debiasing techniques~\citep{khamaru2021near,chen2022debiasing,kim2023double}, anytime-valid confidence intervals~\citep{de2004self,de2009self,abbasi2011improved,howard2020time,waudby2024anytime}, and more recently, simulation-based approaches~\citep{cho2026simulation,praharaj2026bandit}.

A growing line of research has focused on constructing asymptotically valid Wald 
confidence intervals for model parameters, based on the notion of 
\emph{stability} of bandit algorithms. The use of stability conditions to recover classical asymptotic theory under adaptive data collection traces back to the seminal work of~\citet{laiwei82}, who identified deterministic-limit behavior of the (random) design as a sufficient condition for asymptotic normality of least-squares estimators under adaptive sampling. Subsequent work has established that a number of standard multi-arm
bandit algorithms are stable \citep{kalvit2021closer, fan2022typical, fan2024precise, han2024ucb, 
halder2025stable, yan2026optimism, fan2025statistical}. \cite{fan2025statistical} and \cite{praharaj2025avoiding} have extended these results to  \emph{linear} contextual bandits for LinUCB and (regularized) EXP4 respectively. Furthermore, it has been highlighted in literature that lack of stability may lead 
to non-normality in the limiting distribution of 
these estimators, which can lead to under-coverage when Wald-type intervals are 
applied \citep{zhang2020inference, praharaj2025instability, han2026thompson,chen2026bandit}. 
Our focus is on establishing a more generalized version of stability for a broad class of Stochastic Mirror Descent (SMD) sampling rules. Closest to our work is~\citet{praharaj2026avoidingpriceadaptivityinference} who also consider stability regularized mirror descent in the contextual bandit setting. However, their regularization fails in multi-armed bandit setup.

The EXP3 algorithm, introduced in \cite{Auer1995Gambling,auer2002nonstochastic}, is a canonical method for adversarial multi-armed bandits and can be interpreted as mirror descent on the probability simplex with an entropic regularizer (see, e.g., \citet{cesa2006prediction,bubeck2012regret}). More broadly, online/stochastic mirror descent (and closely related FTRL-type procedures) provides a unifying framework for adversarially robust bandit algorithms~\citep{lattimore2020bandit}. Apart from EXP3, such algorithms include Tsallis-INF
(\citet{zimmert2021tsallis,masoudian2021improved}) and OFTRL (\citet{pmlr-v134-ito21a}) and, have been applied extensively in online
optimization (\citet{abernethy2009competing,audibert2014regret,bubeck2018sparsity,wei2018more}). Despite their strong regret guarantees, the stability properties of these algorithms are not well
understood.

Robustness of bandit procedures to contaminated observations has been studied in several models. A distinct and more adversarial model allows an adversary to corrupt rewards under a corruption budget; this setting was formalized by~\citet{lykouris2018stochasticbanditsrobustadversarial}. Subsequent refinements removed the need to know the corruption level a priori and improved rates and adaptivity~\citep{gupta2019better}. 
In contrast, we establish that a suitably regularized mirror-descent/EXP3-type algorithm can be simultaneously (i) inferentially stable and (ii) robust to sub-$\sqrt{T}$ corruption while maintaining near-optimal regret.
\subsection{Contributions}
In this paper, we develop a unified theory for stability of Stochastic Mirror Descent (SMD) methods in stochastic multi-armed bandits. Sections \ref{sec:problem-setup} and~\ref{sec:algo} introduce the bandit model and the notion of  stability, respectively. In Section~\ref{sec:algo} we review the standard stochastic mirror descent algorithm and highlight the instability of standard EXP3 algorithm. Motivated by this observation , we propose and study a regularized SMD procedure (Algorithm~\ref{alg:st-exp3}). Section~\ref{sec:theory} presents our main stability and inference results. This paper makes three primary contributions which we state below.
\begin{itemize}
    \item \textbf{Stability and Berry-Esseen Bound} Theorem~\ref{thm:berry}\ref{thm:berry-i} establishes stability for Algorithm~\ref{alg:st-exp3}, which in turn implies the asymptotic validity of Wald-type confidence intervals for arbitrary linear functionals of the mean reward vector. Theorem~\ref{thm:berry}~\ref{thm:berry-ii} further strengthens this result by providing a quantitative central limit theorem with an explicit Berry--Esseen bound.
    \item \textbf{Precise regret charactarization:} Theorem~\ref{thm:prec-reg} complements the inference analysis by establishing matching upper and lower bounds on the regret, thereby yielding a precise characterization of the performance of Algorithm~\ref{alg:st-exp3}.
    \item \textbf{Robust Inference} Section~\ref{sec:corr-tolerance} extends the analysis to a corrupted-rewards setting. We show that a suitable modification of the tuning parameter in Algorithm~\ref{alg:st-exp3} enables the algorithm to tolerate a prescribed level of adversarial corruption while preserving asymptotic normality (Theorem~\ref{thm:normality-corruption}). We also derive corresponding regret guarantees in the corrupted setting (Theorem~\ref{regret : corrupted}).
\end{itemize}

\noindent Finally, Section~\ref{simulation} presents numerical experiments that corroborate the theoretical findings and demonstrate the empirical coverage properties of the proposed confidence intervals.

\subsection{Notations}

\noindent For a real-valued random variable $X$, we define $\|X\|_p := \mathbb{E}\!\left[ |X|^p \right]^{1/p}.$ Given a fixed weight vector $w \in \mathbb{R}^d$, we define a weighted norm on $\mathbb{R}^K$, denoted by $\|\cdot\|_{w,\ast}$, as $\|u\|_{w,\ast}^2 := \sum_{i=1}^K w_i u_i^2,$ for any $u \in \mathbb{R}^K.$ For two nonnegative sequences $\{a_n\}$ and $\{b_n\}$, we write $b_n \gg a_n$ if $\frac{b_n}{a_n} \to \infty \quad \text{as } n \to \infty.$ We use $\Delta: = \left\{ p \in \mathbf{R}^K : \sum_a p_a = 1, p_a \geq 0 \right\}$ to denote the the probability simplex in $\mathbf{R}^K$. Throughout, we suppress uniformly lower order terms in inequalities by using the notation $\lesssim$ and $\gtrsim$. Finally, the Kolmogorov distance between two real-valued random variables $X$ and $Y$ is denoted by $\mathrm{d}_{\mathrm{KS}}(X,Y)$ and is defined as
\begin{align*}
\mathrm{d}_{\mathrm{KS}}(X,Y)
:= \sup_{t \in \mathbb{R}} \left| \mathbb{P}(X \le t) - \mathbb{P}(Y \le t) \right|.
\end{align*}

\section{Background and problem formulation}
\label{sec:problem-setup}
This section establishes notation for a multi-armed bandit problem and formalizes the two objectives of interest: constructing valid confidence intervals for a linear functional of the mean-loss vector, and maintaining near minimax-optimal regret guarantees.

\subsection{The bandit problem and inferential objectives}
\label{sec:bandit-inference}
We consider a stochastic multi-armed bandit with $K$ arms. At each round $t = 1,\ldots,T$, the learner selects an arm $A_t \in [K]$. If $A_t = a$, then we observe a random loss $\ell_t$ drawn from the distribution $\mathcal{P}_a$ with mean $\mu_a$. The arm choice $A_t$ may depend on the data observed up to time $t-1$; formally, $A_t \in \mathcal{F}_{t-1} := \sigma\{A_1, \ell_1, \ldots, A_{t - 1}, \ell_{t - 1}\}$, where $\mathcal{F}_{t-1}$ denotes the $\sigma$-field induced by the data observed up to time $t-1$. Let $\mu := (\mu_1,\ldots,\mu_K)$ denote the vector of mean losses, with $\mu_j := \mathbb{E}[\ell_t \mid A_t = j]$ for each $j \in [K]$. Throughout, we adopt a loss-minimization formulation, and we measure the performance of a bandit algorithm via its expected regret
\begin{align}
  \label{eq:regret-defn}
  \regret_T := \mathbb{E}\!\left[ \displaystyle\sum_{t=1}^T \ell_t \right]
  - T\mu^\star, \qquad \text{where} \qquad  \mu^\star := \min_{j \in [K]} \mu_j.
\end{align}

Our objective is to design bandit algorithms that achieve two goals: (a) one can construct an efficient confidence interval for the linear functional $u^\top \mu$, where $u$ is a given vector in $\mathbb{R}^K$; and (b) the algorithm enjoys near-optimal regret guarantees.  
When data is collected via a bandit algorithm, the resulting set of samples is not i.i.d. Concretely, the arm pulled at round $t$, and hence the observed loss, depend on the data observed up to time $t-1$. This dependence on prior data invalidates the classical central limit theorem~\cite{zhang2020inference,deshpande2020accurate}: the empirical arm mean $\widehat{\mu}_{T,a} := (\sum_{t=1}^T \ell_t \cdot 1_{A_t=a})/(\sum_{t=1}^T 1_{A_t=a})$ is not asymptotically normal in general. In the next section, we introduce a set of sufficient conditions on a bandit algorithm under which the classical central limit theorem remains valid. As a consequence, one can use the classical Wald interval~\eqref{eq:CI-coverage} as a valid confidence interval for $u^\top \mu$.

\subsection{A sufficient condition for asymptotic normality}
\label{sec:smd-stability}

\noindent Consider any bandit algorithm $\algo$ that at each round $t$ selects the arm $A_t$ based on past actions and losses. Let $\mathcal{F}_t := \sigma(A_1,\ell_1,\ldots,A_t,\ell_t)$ denote the natural filtration. The arm-sampling distribution of the algorithm at time $t$ is given by 
\begin{subequations}
\begin{align}
  \label{eq:pt-defn}
  p_{t,a} := \mathbb{P}\left(A_t = a \mid \mathcal{F}_{t-1}\right),
  \qquad a \in [K],
\end{align}
and we then define the time-averaged sampling distribution as
\begin{align}
  \label{eq:pbar-defn}
  \overline p_{T,a} := \frac{1}{T}\sum\limits_{t=1}^T p_{t,a}, 
  \qquad a \in [K]. 
\end{align}
\end{subequations}

The following lemma shows that the sample means $\{\widehat{\mu}_{T, a}\}_{a \in [K]}$ are jointly asymptotically normal whenever the time-averaged sampling distribution $\pbar := [\bar{p}_{T,a}]_{a \in [K]}$ converges in ratio to a deterministic probability vector $\pstar$.

\begin{lemma}[Sufficient condition for asymptotic normality]
\label{lem:normality}
Let $\algo$ be any bandit algorithm. Suppose that there exists $\varepsilon > 0$ such that $T\varepsilon \to \infty$ and $\min_a \bar p_{T,a} \geq \varepsilon$, and a deterministic vector $\pstar$ satisfying
\begin{align}
  \label{eq:ratio-conv}
  \frac{\pbara}{p^\star_{T,a}} \inprob 1
  \qquad \text{for all } a \in [K].
\end{align}
Then
\begin{align}
  \label{eqn:normality-direct}
  \frac{1}{\widehat{\sigma}} \cdot  \begin{pmatrix}
    \sqrt{n_{T, 1}}
    \bigl(\widehat{\mu}_{T, 1} - \mu_1\bigr) \\[4pt]
    \vdots \\[4pt]
    \sqrt{n_{T, K}}
    \bigl(\widehat{\mu}_{T, K} - \mu_K\bigr)
  \end{pmatrix}
  \indist \mathcal{N}(0, I_K).
\end{align}
where $\widehat{\sigma}$ is any consistent estimator of the common noise standard deviation $\sigma := \sqrt{\mathrm{Var}\left(\ell_t \mid A_t = 1 \right)}$. 
\end{lemma}

A few remarks regarding Lemma~\ref{lem:normality} are in order. The lower bound $\varepsilon > 0$ in the lemma is permitted to depend on the number of rounds $T$; we suppress this dependence in order to simplify the presentation. In our algorithm, we set $\varepsilon = 1/\sqrt{T}$ (see the parameter choices in~\eqref{eq:param-choices}).  
Next, we point out that condition~\eqref{eq:ratio-conv} is a simple refinement of the stability condition of Lai and Wei~\citep{laiwei82,khamaru2024inference,han2024ucb}. In particular, if the algorithm $\mathcal{A}$ is deterministic given $\mathcal{F}_{t-1}$ --- that is, $p_{t,a} = \mathbf{1}\{A_t = a\}$ --- and $n_{T, a}$ denotes the number of times arm $a$ is pulled in $T$ rounds, then $\pbara = n_{T, a}/T$ and $p^\star_{T,a} = n^\star_{T, a}/T$, so that condition~\eqref{eq:ratio-conv} reduces to $n_{T, a}/n^\star_{T, a} \inprob 1$, which is precisely the stability condition of~\citet{laiwei82}.
Condition~\eqref{eq:ratio-conv}, phrased in terms of $\bar{p}_{T}$, is better suited to our purposes than the Lai--Wei condition on the arm-pull counts $\{n_{a,T}\}_{a \in [K]}$, because the sequential mirror-descent-style algorithms that we discuss (see Section~\ref{sec:algo}) afford precise control over the arm-selection probabilities $\{p_{t}\}_{t \geq 1}$. Within our framework, we instead verify the stronger $L^1$ condition
\begin{align}
  \label{eq:L1-conv}
(L^1\texttt{-stability}) \qquad   \mathbb{E}\left|
    \frac{\pbara}{p^\star_{T,a}} - 1
  \right| \to 0
  \qquad \text{for all } a \in [K].
\end{align}
This stronger $L^{1}$ control allows us to derive a quantitative central limit theorem (Berry-Esseen bound) and matching non-asymptotic upper and lower bounds on the regret. With this set-up in place, we now turn to the bandit algorithms that we analyze.

\newpage 
\section{A Regularized Stochastic Mirror Descent Algorithm}
\label{sec:algo}
We now develop a bandit algorithm that satisfies condition~\eqref{eq:L1-conv} while maintaining near-optimal regret. Identifying the regret-minimization problem as an online optimization problem, we first explain why the classical EXP3 algorithm~\citep{vovk1990aggregating,littlestone1994weighted,Auer1995Gambling} fails to satisfy condition~\eqref{eq:L1-conv}, and we then motivate a regularized variant of EXP3 that remedies this shortcoming.

\medskip

Let $f : \Delta \to \mathbb{R}$ denote the linear objective
\begin{align}
  \label{eq:linear-obj}
  f(p) := \langle \mu, p \rangle
  \;=\; \sum_{a=1}^K \mu_a p_a,
\end{align}
whose minimum over the probability simplex $\Delta$ is $\mu^\star = \min_{a \in [K]} \mu_a$, attained at any distribution supported on the set $\opticlass$ of optimal arms. Since $p_{t,a} := \prob(A_t = a \mid \mathcal{F}_{t-1})$, the conditional expected loss at round $t$ satisfies $\mathbb{E}[\ell_t \mid \mathcal{F}_{t-1}] = f(p_t)$. Writing $\pbar := \tfrac{1}{T}\sum_{t=1}^T p_t$ for the time-averaged sampling distribution and using linearity of expectation, we have
\begin{align}
  \label{eq:regret-as-opt}
  \regret_T
  \;=\; \mathbb{E}\!\left[\sum_{t=1}^T \ell_t\right]
  - T\mu^\star
  \;=\; T\cdot\mathbb{E}\bigl[f(\pbar) - \mu^\star\bigr]
  \;=\; T\cdot\mathbb{E}\bigl[
    f(\pbar) - \min_{p \in \Delta} f(p)\bigr].
\end{align}
Regret minimization is therefore an online linear optimization problem over the simplex $\Delta$: minimizing $\regret(T)$ is equivalent to minimizing the time-averaged loss $f(\pbar)$. The EXP3 algorithm~\citep{Auer1995Gambling,auer2002nonstochastic} solves this problem via stochastic mirror descent (SMD)~\citep{nemirovski1983} on the function $f$, using the negative-entropy mirror map. The resulting algorithm achieves $\regret_T = O(\sqrt{TK\log K})$, which is sublinear in the horizon $T$~\citep{lattimore2020bandit}.

\medskip

\subsection{Instability of Exp3 and its variants}
A sublinear regret guarantee does not, however, imply that the time-averaged distribution $\pbar$ converges to any particular point in $\Delta$. The source of the difficulty is the geometry of $f$: since $f$ is linear, its set of minimizers over $\Delta$ is the face
\begin{align}
  \label{eq:optimal-face}
  \mathcal{F}^\star :=
  \bigl\{p \in \Delta : p_a = 0 \;
  \forall a \notin \opticlass\bigr\},
\end{align}
a convex set of dimension $|\opticlass| - 1$. Whenever there is more than one optimal arm, this set contains infinitely many points, so the minimizer of $f$ over $\Delta$ is not unique. As a result, the average $\bar{p}_T$ need not converge in ratio to any deterministic $p^\star_T$, and condition~\eqref{eq:L1-conv} may fail. 

We formalize this intuition in the following lemma, which shows that any bandit algorithm --- possibly randomized, as in EXP3 --- whose regret is at most $c_0\sqrt{T}$ cannot satisfy condition~\eqref{eq:L1-conv}. Here, by \emph{randomized} we mean that, given the history $\mathcal{F}_{t-1}$, the arm $A_t$ is drawn from a distribution using external randomness; in contrast, UCB-type algorithms select $A_t$ deterministically given $\mathcal{F}_{t-1}$.

\begin{lemma}\label{lemma:instability}
    Any (possibly randomized) bandit algorithm with regret upper bounded by $c_0\sqrt{T}$ cannot satisfy condition~\eqref{eq:L1-conv}. Here $c_0$ is a constant independent of $T$. 
\end{lemma}


The regret of EXP3 is upper bounded by $\sqrt{KT\log K}$, which is of order $\sqrt{T}$ for fixed $K$. Consequently, EXP3 cannot satisfy condition~\eqref{eq:L1-conv}. See Appendix~\ref{sec:Proof-of-Lemma-instability} for a proof. One natural remedy is to add a strictly convex perturbation to $f$, thereby collapsing the optimal face to a unique minimizer; the precise form of this regularizer is dictated by condition~\eqref{eq:L1-conv}, which we discuss in the next section.

\medskip
\subsection{The log-barrier regularizer}
\label{sec:log-barrier}
The appropriate regularization is determined by the geometry of condition~\eqref{eq:L1-conv}, which concerns the ratio $\pbara/p^\star_{T,a}$ rather than the difference $|\pbara - p^\star_{T,a}|$. A natural choice is the log-barrier regularizer $\psi(p) := -\sum_a\ln p_a$. The Bregman divergence associated with $\psi$ is the Itakura--Saito divergence~\citep{itakura1968analysis}, given by
\begin{align}
  \label{eq:IS-defn}
  IS(p, q) := \sum_{i=1}^K\!\left[
    \frac{p_i}{q_i} - \ln\frac{p_i}{q_i} - 1\right],
\end{align}
We establish later (see the proof of Theorem~\ref{thm:berry}) that $\mathbb{E}[IS(p, q)] \rightarrow 0$ implies $\sum_a \Exs |p_a/q_a - 1| \rightarrow 0$, which is precisely the control required by condition~\eqref{eq:L1-conv}. By contrast, a Euclidean penalty $R(p) = \|p\|_2^2$ or an entropic penalty $R(p) = -\sum_a p_a\log p_a$ controls only the difference $|\pbara - p^\star_{T,a}|$ and provides no control over arms for which $p^\star_{T,a}$ vanishes, which is insufficient to establish condition~\eqref{eq:L1-conv}.  In our algorithm, we add the following modified log-barrier regularizer 
\begin{align}
  \label{eqn:regularized-loss}
  R_\varepsilon(p) := -\sum_{i=1}^K \ln p_i
  + \frac{1}{\varepsilon}\sum_{i=1}^K p_i
\end{align}
to the linear loss function $f$ with regularization-weight $\lambda > 0$, and run stochastic mirror descent on the regularized loss
\begin{align}
  \label{eq:reg-obj}
  f_{\lambda,\varepsilon}(p)
  := \langle\mu,p\rangle + \lambda R_\varepsilon(p)
\end{align}
over the truncated simplex
$\DE := \{p \in \mathbb{R}^K : p_j \ge \varepsilon \;
\forall j,\; \sum_j p_j = 1\}$.
The term $\tfrac{1}{\varepsilon}\sum_i p_i$ is constant on the simplex $\Delta$ and leaves the minimizer unchanged; this modification ensures that $\nabla R_\varepsilon \ge 0$ on $\DE$ and simplifies the subsequent analysis.

\medskip
\noindent\textbf{Mirror map.}
For the Bregman projection step, we allow any mirror map from the Tsallis-entropy family~\citep{tsallis1988possible}, given by
\begin{align}
  \label{defn:mirror-map-alpha}
  \phi_\alpha(p) :=
  \begin{cases}
    -\sum_{i=1}^K
    \dfrac{p_i^{\alpha} - \alpha p_i - (1-\alpha)}
    {\alpha(1-\alpha)},
    & 0 < \alpha < 1,\\[6pt]
    -\sum_{i=1}^K (\log p_i - p_i + 1),
    & \alpha = 0,\\[6pt]
    \;\sum_{i=1}^K (p_i\log p_i - p_i + 1),
    & \alpha = 1.
  \end{cases}
\end{align}
The case $\alpha = 1$ corresponds to the negative entropy used by standard EXP3, whereas $\alpha = 0$ recovers the log-barrier. The choice of $\alpha$ affects the Bregman geometry and hence the regret bound, but not the inferential guarantees, which hold for all $\alpha \in [0,1]$.

At each round, the learner projects the dual iterate onto truncated simplex 
$\DE$ via the Bregman divergence $D_{\phi_\alpha}$, draws an arm from the resulting
$p_t$, and updates the dual iterate with a mirror-descent step of size $\eta$ using the loss estimate
\begin{align}
   \widetilde\ell_t := \widehat\ell_t
+ \lambda \nabla R_\varepsilon(p_t) \qquad \text{where} \qquad   \widehat\ell_{t,j} = \frac{\ell_t\mathbf{1}\{A_t=j\}}{p_{t,j}}
\end{align}
which is unbiased for the gradient $\nabla f_{\lambda,\varepsilon}(p_t)$ conditionally on $\mathcal{F}_{t - 1}$. The complete procedure is stated as Algorithm~\ref{alg:st-exp3}.

\begin{algorithm}
  \caption{Regularized-EXP3}\label{alg:st-exp3}
  \begin{algorithmic}[]
    \State{\textbf{Inputs}: (a) horizon $T$; (b) step size
      $\eta$; (c) regularizer weight $\lambda$;
      (d) simplex floor $\varepsilon$;
      (e) mirror map $\phi_\alpha$.\\
    \textbf{Initialize}: $z_1 = (1/K,\ldots,1/K)^\top$.}
    \For{$t = 1, \ldots, T$}
      \State Set $p_t = \argmin \limits_{p\in\DE}
        D_{\phi_\alpha}(p, z_t)$.
      \State Draw $A_t \sim p_t$. Observe $\ell_t$.
      \State Set $\widehat\ell_{t,j}
        := \ell_t\mathbf{1}\{A_t=j\}/p_{t,j}$ \;\;\;\; 
        and \;\;\;\;  $\widetilde\ell_t
        := \widehat\ell_t
        + \lambda\nabla R_\varepsilon(p_t)$.
      \State Set $z_{t+1} = \argmin_{p}\bigl\{
        \eta\langle\widetilde\ell_t, p\rangle
        + D_{\phi_\alpha}(p, p_t)\bigr\}$.
    \EndFor
    \State{\textbf{Output}: $\pbar = \frac{1}{T}\sum_{t=1}^T
      p_t$.}
  \end{algorithmic}
\end{algorithm}

\section{Theoretical guarantees}
\label{sec:theory}
In this section, we present our two main results. Section~\ref{sec:inference} establishes a quantitative central limit theorem for Algorithm~\ref{alg:st-exp3}. In Section~\ref{sec:regret}, we present a \emph{precise} non-asymptotic characterization of the regret, with matching upper and lower bounds. Throughout this section, we fix the parameter choices:

\noindent 
\begin{align}
  \label{eq:param-choices}
  \lambda := \frac{\gamma_T\log T}{\sqrt{T}},
  \qquad
  \eta := \frac{1}{\sqrt{T}},
  \qquad
  \varepsilon := \frac{h_T}{\sqrt{T}},
\end{align}
where $\gamma_T\to\infty$ and
$\gamma_T\log T/h_T\to 0$ as $T\to\infty$.
All of our guarantees are governed by a single rate $\Psi_T$, given by
\begin{align}
  \label{eqn:psi-defn}
  \Psi_T:= (2 \cdot 15^{1/3}) \left(
    \frac{3K C^\star_{K,T}(\alpha)}{\hTo  }
    + \frac{\gamma_T(\log T)^2}{h_T^2}
  \right)^{1/3}, \quad \text{where} \ C^\star_{K,T}(\alpha) := \begin{cases}
      \log K, \  \text{if} \  \alpha \in  [1/3,1] \\
      \frac {\log T} 2, \  \text{if} \  \alpha \in  [0, 1/3].     
  \end{cases}
\end{align}
As we show below, $\Psi_T$ simultaneously controls the accuracy of the Gaussian approximation and the tightness of our regret characterization. We make the following assumptions on the losses and the noise:

\begin{assumption}\label{assump:main}
\leavevmode
\begin{enumerate}[label=(A\arabic*), leftmargin=2.5em]
  \item \label{assump:noise}
    The noise
    $\epsilon_{a,t}$ satisfies
    $\mathbb{E}[\epsilon_{a,t} \mid \mathcal{F}_{t-1},
    A_t] = 0$.
  \item \label{assump:bd-losses}
    The loss distribution of each arm is supported in
    $[0,1]$.
\end{enumerate}
\end{assumption}

\noindent Our main results are naturally expressed in terms of the solution to the following regularized problem discussed in Section~\ref{sec:log-barrier}: 
\begin{subequations}
\begin{align}
\label{eq:pop-loss-soln}
    p^\star_{T} = \argmin_{p \in \DE} \;  \left\lbrace f_{\epsilon, \lambda}(p) := 
     \langle\mu,p\rangle + \lambda R_\varepsilon(p)
    \right\rbrace
\end{align}
The solution of the last problem has the following closed form solution 
\begin{align}
  \label{eq:pstar-KKT}
  p^\star_{T,a} = \max\!\left(\varepsilon,\;
    \frac{\lambda}{\Delta_a + \mu^\star - \nu^\star}\right),
  \quad \text{where~$\nu^\star$~is the unique solution of} \quad
  \sum_{a=1}^K \max\!\left(\varepsilon,\;
    \frac{\lambda}{\Delta_a + \mu^\star - \nu}\right) = 1,
\end{align}
\end{subequations}
\noindent where $\Delta_a = \mu_a - \mu^\star$. With this set-up in place, we are now ready to state our main results.

\subsection{A quantitative central limit theorem}
\label{sec:inference}
Our first result verifies that Algorithm~\ref{alg:st-exp3}
satisfies the iterate convergence
condition~\eqref{eq:L1-conv} of Lemma~\ref{lem:normality},
and strengthens the asymptotic conclusion of that lemma to
a non-asymptotic Berry--Esseen bound.
 
\begin{theorem}\label{thm:berry}
Under Assumption~\ref{assump:main} and the parameter
choices~\eqref{eq:param-choices},
Algorithm~\ref{alg:st-exp3} satisfies:
\begin{enumerate}[label=(\roman*)]
  \item \label{thm:berry-i}
  $\displaystyle\;\mathbb{E}\left|
  \frac{\overline p_{T,a}}{p^\star_{T,a}}-1\right|
  \;\leq\;  \Psi_T, \;\;  \text{for all arms} \;\; a \in [K]$.
  \item \label{thm:berry-ii}
  $ \max \limits_{a \in [K]} \displaystyle\;d_{\mathrm{KS}}\!\left(
  \frac{\sqrt{n_{a,T}}}{\sigma} \cdot \,(\widehat{\mu}_{a,T}-\mu_a),\,Z
  \right)\;\leq\;C \left(\Psi_T^{1/3} + \frac{1}{T^{1/5}\hTto^{4/5}} \right),$
\end{enumerate}
where $Z\sim\mathcal{N}(0,1)$ and $C$ is a universal
constant independent of $T$ and $K$.
\end{theorem}

The two parts of Theorem~\ref{thm:berry} play distinct roles. Part~(i) is the \emph{algorithmic} result: it verifies the iterate
convergence condition~\eqref{eq:L1-conv} with an explicit
rate; the proof rests on the strong convexity of the
regularized objective with respect to the Itakura--Saito
divergence. Part~(ii) is the \emph{statistical}
consequence: a non-asymptotic Berry--Esseen bound,
obtained by combining part~(i) with the martingale
central limit theorem of~\cite{mourrat2013rate}.
In particular, part~(ii) sharpens
Lemma~\ref{lem:normality} from an asymptotic statement
to a quantitative one.

An immediate consequence of part~(i), together with Lemma~\ref{lem:normality}, is the asymptotic validity of Wald-type confidence intervals for any linear functional $u^\top\mu$ --- for instance, individual arm means ($u = e_a$) or pairwise differences ($u = e_a - e_b$). In particular, for any fixed nonzero vector $u\in\mathbb{R}^K$ and any target level $\alpha_0\in(0,1)$,
\begin{align}
  \label{eq:CI-coverage}
  \lim_{T\to\infty}\prob\left(
    u^\top\mu \in
    \left[\,u^\top\widehat{\mu}_T \;\pm\;
    z_{1-\alpha_0/2}\sqrt{\textstyle\sum_{a=1}^K
    \frac{u_a^2\,\widehat{\sigma}^2}{n_{a,T}}}
    \,\right]\right) = 1-\alpha_0.
\end{align}

\subsection{A Precise Regret Analysis}
\label{sec:regret}
While Algorithm~\ref{alg:st-exp3} guarantees inferential validity, a natural question concerns the cost in regret. Theorem~\ref{thm:berry} shows that the CLT rate $\Psi_T$ improves as $\gamma_T$ and $h_T$ grow; conversely, larger values of $\gamma_T$ and $h_T$ strengthen the regularization and raise the exploration floor, thereby inflating the regret. Our next result quantifies this trade-off \emph{exactly}: rather than an upper bound alone, we characterize the regret with matching upper and lower bounds, relative to the following idealized benchmark:

\begin{align}
\label{eqn:ideal-reg-defno}
  \regret^\star(T) := \sum_{a\in[K]}
  \Delta_a\cdot T\,p^\star_{T,a},
\end{align}
where $p^\star_T = (p^\star_{T,1},\ldots,p^\star_{T,K})$
is defined in~\eqref{eq:pstar-KKT}. In words, $\regret^\star(T)$ is the regret the algorithm would incur if it sampled arms from the probability vector $p^\star_T$ at every round.  

\begin{theorem}
\label{thm:prec-reg}
Under the conditions of Theorem~\ref{thm:berry}, the regret of Algorithm~\ref{alg:st-exp3} satisfies 
\begin{align*}
  \left|\frac{\regret(T)}{\regret^\star(T)}-1\right|
  \;\leq\; C \left(\Psi_T + \sqrt{\frac{\log T}{\sqrt{T}}}  + \frac{2 \sqrt{T}}{h_T T^{\hTto}}  \right),
\end{align*}
where $\regret^\star(T)$ is defined in equation~\eqref{eqn:ideal-reg-defno}, and $C$ is a universal constant independent of $T$ and $K$.  
\end{theorem}

We highlight several features of Theorem~\ref{thm:prec-reg}. First, it is entirely \emph{non-asymptotic}: the bound holds for every finite $T$. Second, it imposes no condition on the suboptimality gaps: the gaps $\Delta_a$ may be arbitrarily small and may even vanish with $T$. Third, the number of arms $K$ is permitted to grow with the horizon $T$. To the best of our knowledge, this is the first non-asymptotic characterization of the regret of a stochastic-mirror-descent-style algorithm by matching upper and lower bounds with identical leading constants.

Finally, the benchmark $\regret^\star(T)$ is defined implicitly through the dual variable $\nu^\star$ and does not admit a simple closed-form expression for finite $T$. The following lemma bounds it explicitly and, when combined with Theorem~\ref{thm:prec-reg}, yields a worst-case regret guarantee that can be compared directly with existing algorithms.

\begin{lemma}
   \label{lemma:actual regret bd}
Suppose Assumption \ref{assump:bd-losses} are in force. Then, Algorithm \ref{alg:st-exp3} with $\phi_\alpha(\cdot)$ as the mirror map satisfies the following regret bound
\begin{align*}
    \regret(T)\le \begin{cases}
        \left(4\log T \sqrt{K} +2 \gamma_T^2\log^2 T \sqrt{K} + 2\hTto \sqrt{K} \right)\cdot \sqrt{KT} ,&\alpha\in\left[0,\frac{1}{3}\right)\\
        \\
        \left(4\log K \sqrt{K} + 2\sqrt{K}\gamma^2_T\log^2 T +2 \hTto \sqrt{K} \right)\cdot \sqrt{KT} ,&\alpha\in\left[\frac{1}{3},1\right]
    \end{cases}
\end{align*}
$\forall\; \alpha \in [0,1]$. 
\end{lemma}

\noindent Compared with vanilla EXP3, which achieves regret $O(\sqrt{KT\log K})$, the regret of Algorithm~\ref{alg:st-exp3} is larger by a factor of at most $O\left(\max\left\{\gamma^2_T\log^2 T,h_T\right\} \right)$. The parameter $\gamma_T$ may grow arbitrarily slowly; taking $\gamma_T = \log\log T$ and $h_T = (\log T)^2$, for instance, we conclude that the cost of inferential validity is at most polylogarithmic in $T$. This quantifies one of the central messages of the paper: \emph{the same regularization that enables valid inference inflates the regret by only logarithmic factors.}

\noindent

\section{Inference with corrupted samples}
\label{sec:corr-tolerance}
Thus far, we have discussed bandit inference with uncorrupted data, and we have argued that regularized EXP3 belongs to a family of algorithms --- including UCB-type methods~\citep{auer2002finite} and Thompson sampling~\citep{thompson1933likelihood} --- that allow for efficient inference while maintaining near-optimal regret guarantees. One problem with standard UCB-type algorithms is that they are intrinsically brittle to adversarial corruption. Once the cumulative corruption exceeds a logarithmic order in the horizon $T$, an adversary can systematically bias the empirical means so that the learner continues to favor suboptimal arms, thereby incurring regret linear in $T$. This failure mechanism is demonstrated by \citet{jun2018adversarialattacksstochasticbandits}, who show that several canonical stochastic bandit algorithms break down under moderate, persistent corruptions. In response, \citet{lykouris2018stochasticbanditsrobustadversarial} develop corruption-aware procedures, and \citet{gupta2019better} sharpen the attainable rates and matching lower bounds, thereby formalizing the inadequacy of naive optimism in corrupted environments.

More broadly, in online learning, the feedback may be corrupted before it is revealed to the learner. Such contamination invalidates the arguments that underlie both regret and statistical guarantees, and therefore necessitates algorithms whose performance degrades gracefully with the magnitude of corruption (see, for example, \citet{gupta2019better, lykouris2018stochasticbanditsrobustadversarial}).

We argue next that the regularized mirror-descent perspective introduced in the preceding sections can be applied not only to corruption-robust learning, but also to \emph{valid inference} under corrupted feedback. The key step is to establish condition~\eqref{eq:L1-conv} under a corruption budget permitted to grow with $T$. Formally, at each round $t$, there is an underlying loss vector $\widetilde{\ell_t} \in [0,1]^K$. After observing $\widetilde{\ell_t}$, an adversary produces a corrupted vector $\widetilde{\ell_t}^{\mathrm{c}} \in [0,1]^K$. We quantify the overall level of corruption by
\begin{align}
\label{defn:lvl-corruption}
C_T \;:=\; \sum_{t=1}^{T}\mathbb{E}\!\left[\left\lVert \widetilde{\ell_t}^{\mathrm{c}} - \widetilde{\ell_t} \right\rVert_{\infty}\right],
\end{align}
where the expectation is taken over the randomness of the losses and any internal randomization of the algorithm (and, if applicable, the adversary).

The following theorem shows that if $C_T = o(\sqrt{T})$, then, with an appropriate choice of parameters, Algorithm~\ref{alg:st-exp3} is stable, and consequently the empirical arm means are asymptotically normal.

\begin{theorem}
    \label{thm:normality-corruption}
     Suppose that $C_T \leq K T^\beta$ for some $\beta \in (0,\tfrac{1}{2})$, and set $\setpz=\tfrac{1}{\sqrt{T}}$ and $\varepsilon=\pen=\tfrac{1}{\sqrt{K}}\, T^{-(\frac{1}{2}-\beta)}$. Then, for any mirror map $\phi_\alpha$ from~\eqref{defn:mirror-map-alpha}, Algorithm~\ref{alg:st-exp3} satisfies condition~\eqref{eq:L1-conv}. Moreover, let $\muhat_{a,T}$ and $\sigmahat^2_{a,T}$ denote, respectively, the sample mean and sample variance of the arm-$a$ rewards observed up to time $T$. Then
    \begin{align*}
        \frac{\sqrt{n_{a,T} }}{\sigmahat_{a,T}} \cdot \left( \muhat_{a,T} - \mu_a \right) \indist  \Ncal(0, 1)
    \end{align*}
     for all $a \in [K]$.
\end{theorem}
\noindent The proof is given in Section~\ref{proof:normality-corr} of the Appendix. The next lemma establishes a regret bound for regularized EXP3 under corruption.
\begin{lemma}
\label{regret : corrupted}
Under the conditions stated in Theorem~\ref{thm:normality-corruption}, the regret of Algorithm~\ref{alg:st-exp3} in the corrupted samples setup is upper bounded by
\begin{align*}
    \regret(T)\le \begin{cases}
        \left(4\log T  +\frac{\hTo^2\log^2 T}{T^{2\beta}}+2\; T^\beta\right)\cdot K\sqrt{T},&\alpha\in\left[0,\frac{1}{3}\right)\\[8pt]
        \left(4\log K  +\frac{\hTo^2\log^2 T}{T^{2\beta}}+2\; T^\beta\right)\cdot K\sqrt{T},&\alpha\in\left[\frac{1}{3},1\right]
    \end{cases}
\end{align*}
\end{lemma}
We defer the proof to Section~\ref{proof:regret-corr}.
\section{Illustrative simulations}
\label{simulation}
In this section, we present the numerical simulations for Algorithm~\ref{alg:st-exp3} for a multi-armed Bernoulli Bandit problem. We choose $\setpz=\frac{1}{\sqrt{T}}$, $\pen=\frac{\log^2 T}{\sqrt{KT}}$ and $\vars=\frac{\log(T)^3}{\sqrt{T}}$.

First, let us consider the scenario where we have a unique optimal arm: we set $\mu=(0.9,0.3,0.1)^\top$ and $T=10^5$ averaged over $1000$ independent experiments. We choose $\alpha=1$, that is, the mirror map is negative entropy.
\begin{figure}[htbp]
    \centering
    \hfill
    \begin{minipage}[t]{0.3\linewidth}
        \centering
        \includegraphics[trim=0 30 0 50, clip, width=\linewidth]{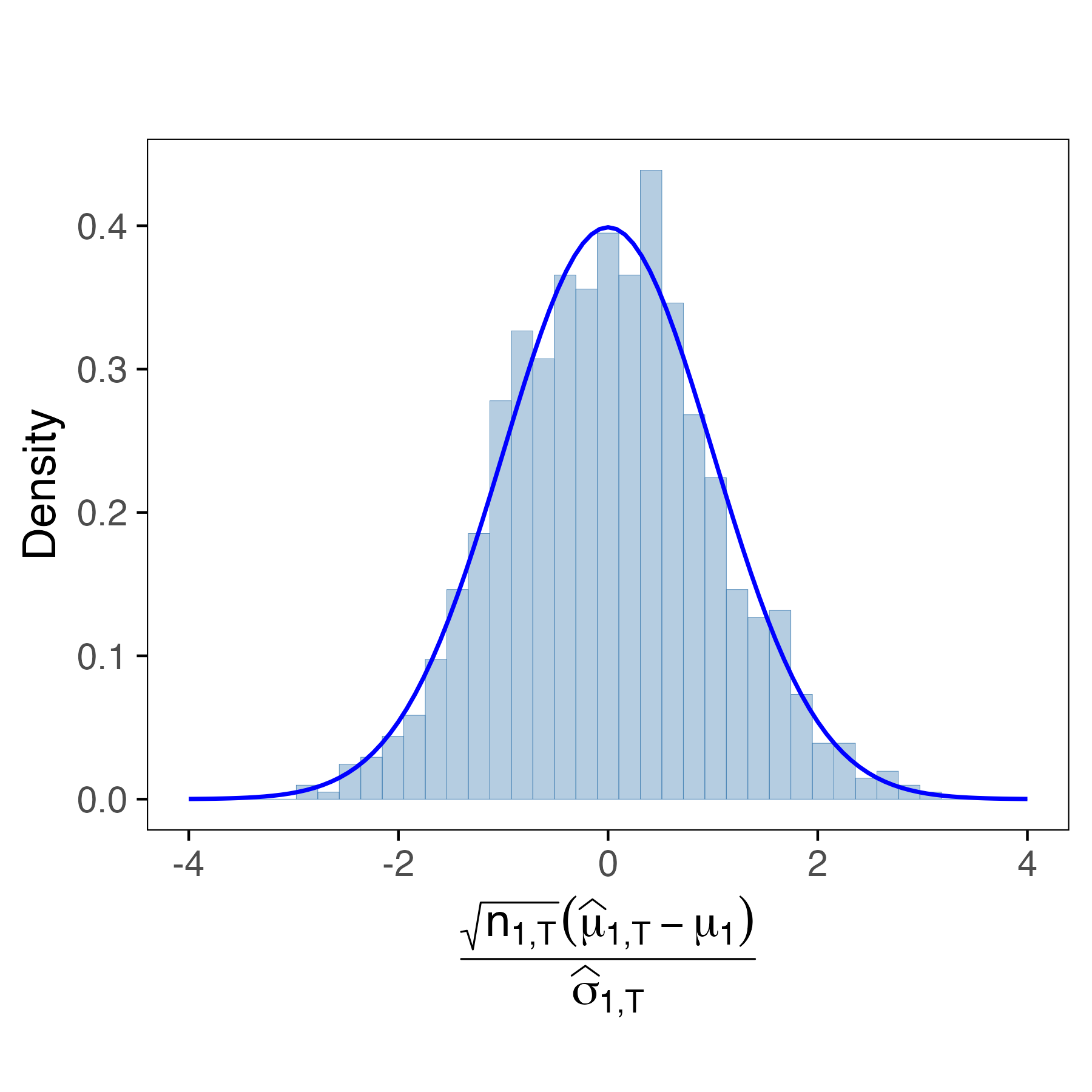}
    \end{minipage}\hfill
    \begin{minipage}[t]{0.3\linewidth}
        \centering
        \includegraphics[trim=0 30 0 50, clip, width=\linewidth]{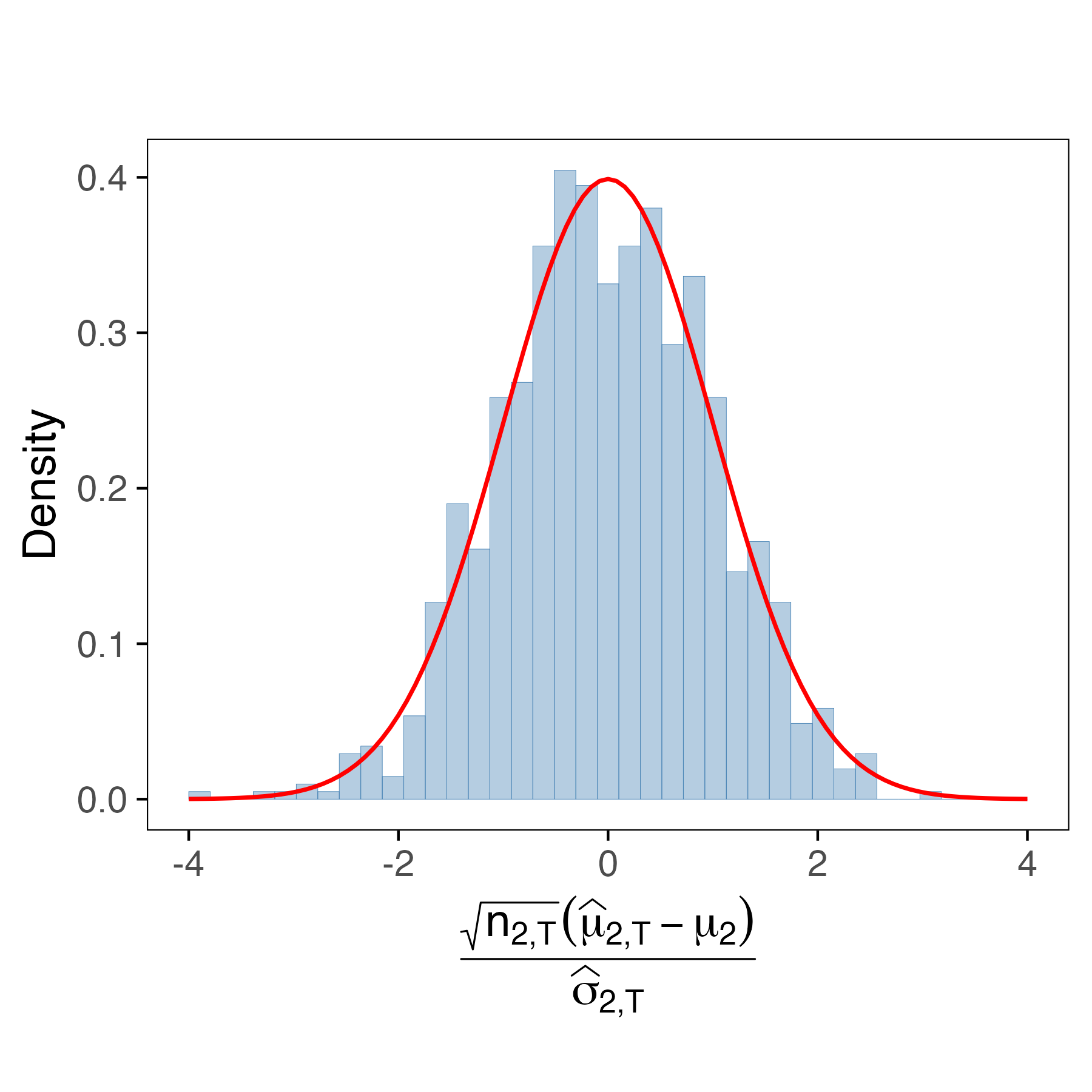}
    \end{minipage}\hfill
    \begin{minipage}[t]{0.3\linewidth}
        \centering
        \includegraphics[trim=0 30 0 50, clip, width=\linewidth]{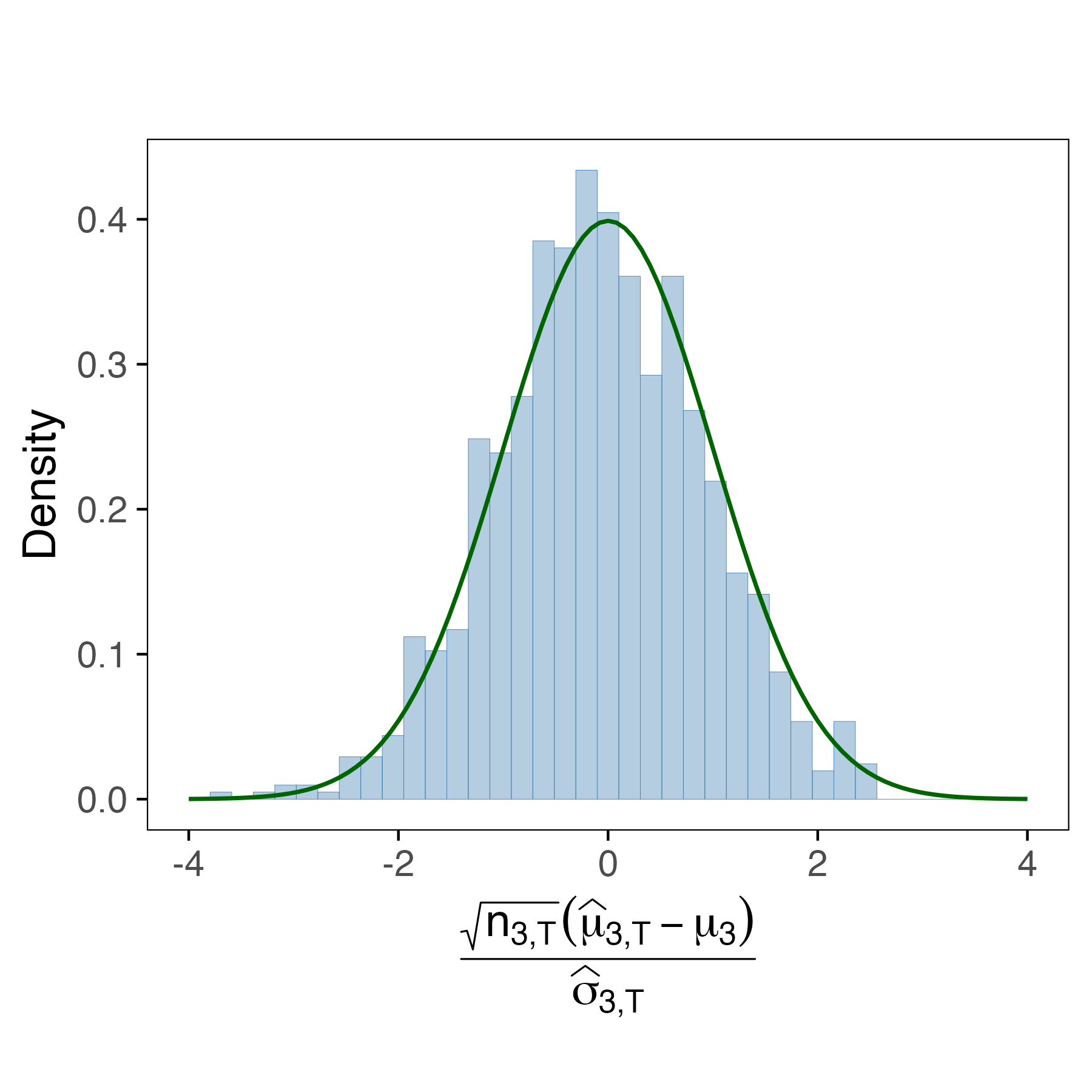}
    \end{minipage}
    \hfill
    \caption{Empirical behavior for Algorithm~\ref{alg:st-exp3} for Bernoulli bandit with $\mu=(0.9,0.3,0.1)^\top$ and $\alpha=1$ : standardized estimation errors \(\sqrt{n_{a,T}}(\widehat{\mu}_{a,T} - \mu_a)/\widehat{\sigma}_{a,T}\) are approximately standard normal.}
    \label{fig:uneq_means_std}
\end{figure}

\noindent Figure~\ref{fig:uneq_means_std} shows the standardized estimation errors against the standard Gaussian density, confirming the asymptotic normality of the sample means when we have a unique optimal arm. We also look at the empirical coverage plots for a sequence of nominal confidence levels $(1-\alpha_0)$ ranging from $0.75$ to $0.99$ --- in Figure~\ref{fig:uneq_means_cov}, the solid blue, red and green curves show the empirical coverage probabilities for the three arms, with shaded bands indicating the standard errors, while the black dashed line represents the target nominal coverage level. Throughout the entire range, the empirical coverage closely tracks the target level for each arm, providing strong evidence for the asymptotic validity of the confidence intervals.

\begin{figure}[htbp]
    \centering
    \hfill
    \begin{minipage}[t]{0.3\linewidth}
        \centering
        \includegraphics[trim=0 0 0 0, clip, width=\linewidth]{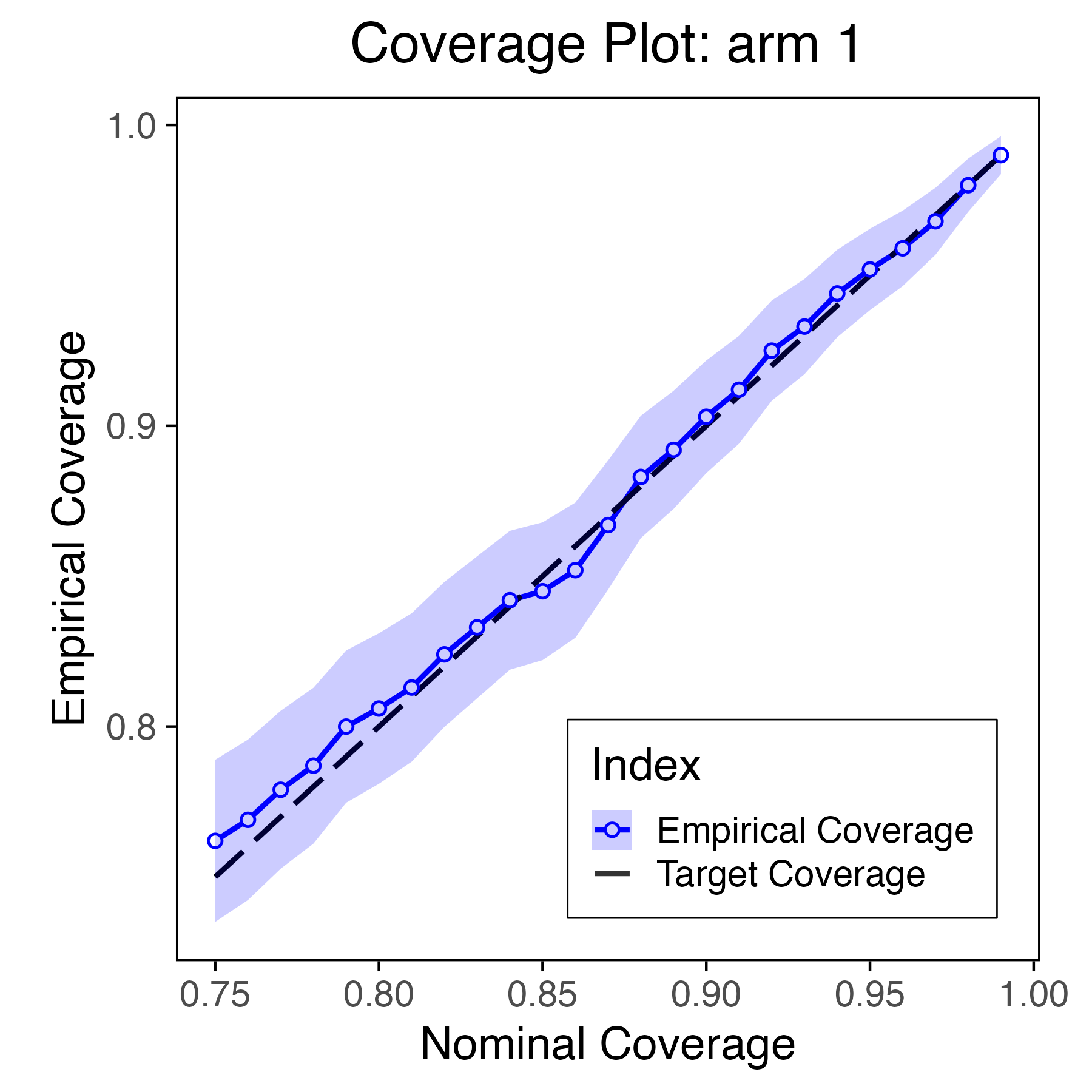}
    \end{minipage}\hfill
    \begin{minipage}[t]{0.3\linewidth}
        \centering
        \includegraphics[trim=0 0 0 0, clip, width=\linewidth]{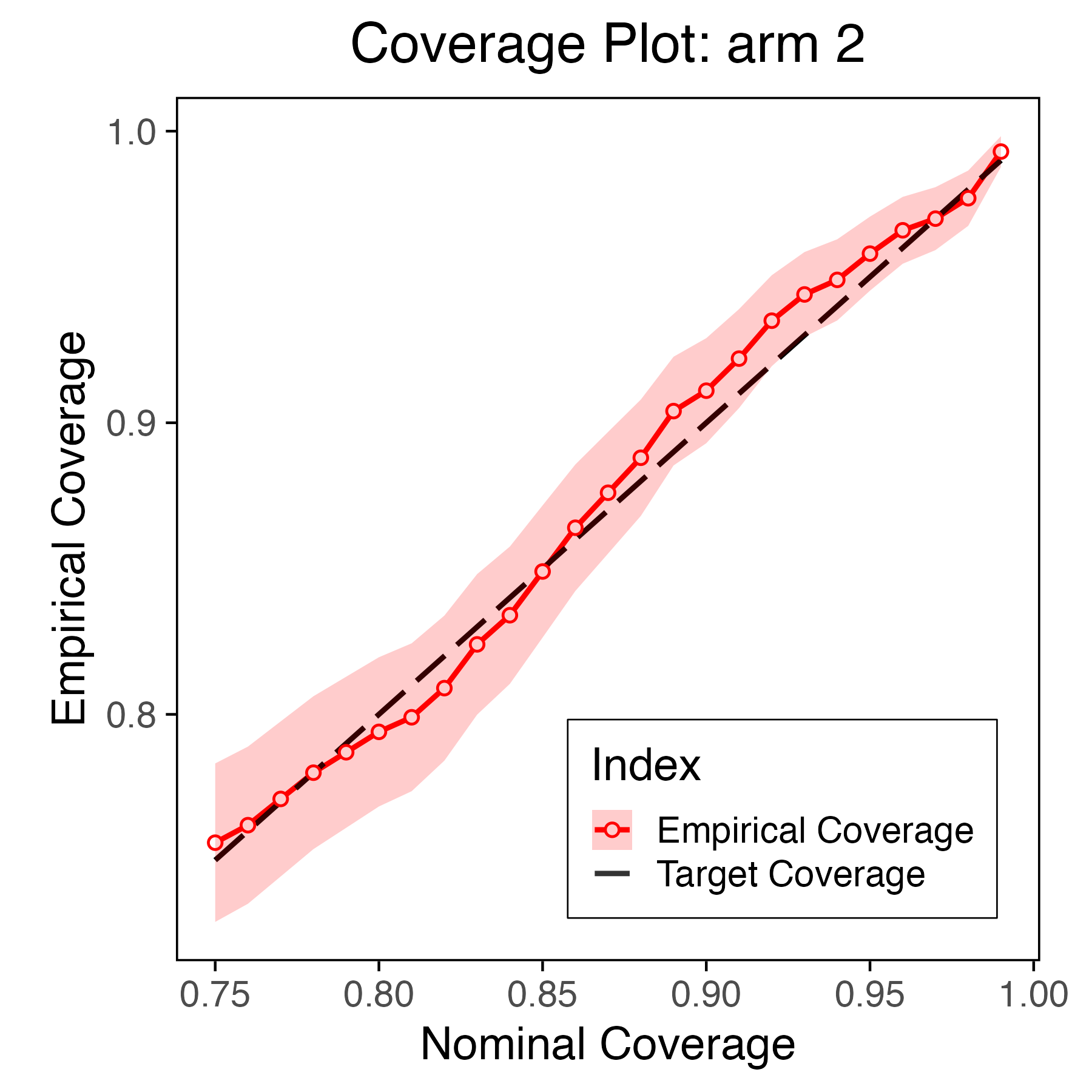}
    \end{minipage}\hfill
    \begin{minipage}[t]{0.3\linewidth}
        \centering
        \includegraphics[trim=0 0 0 0, clip, width=\linewidth]{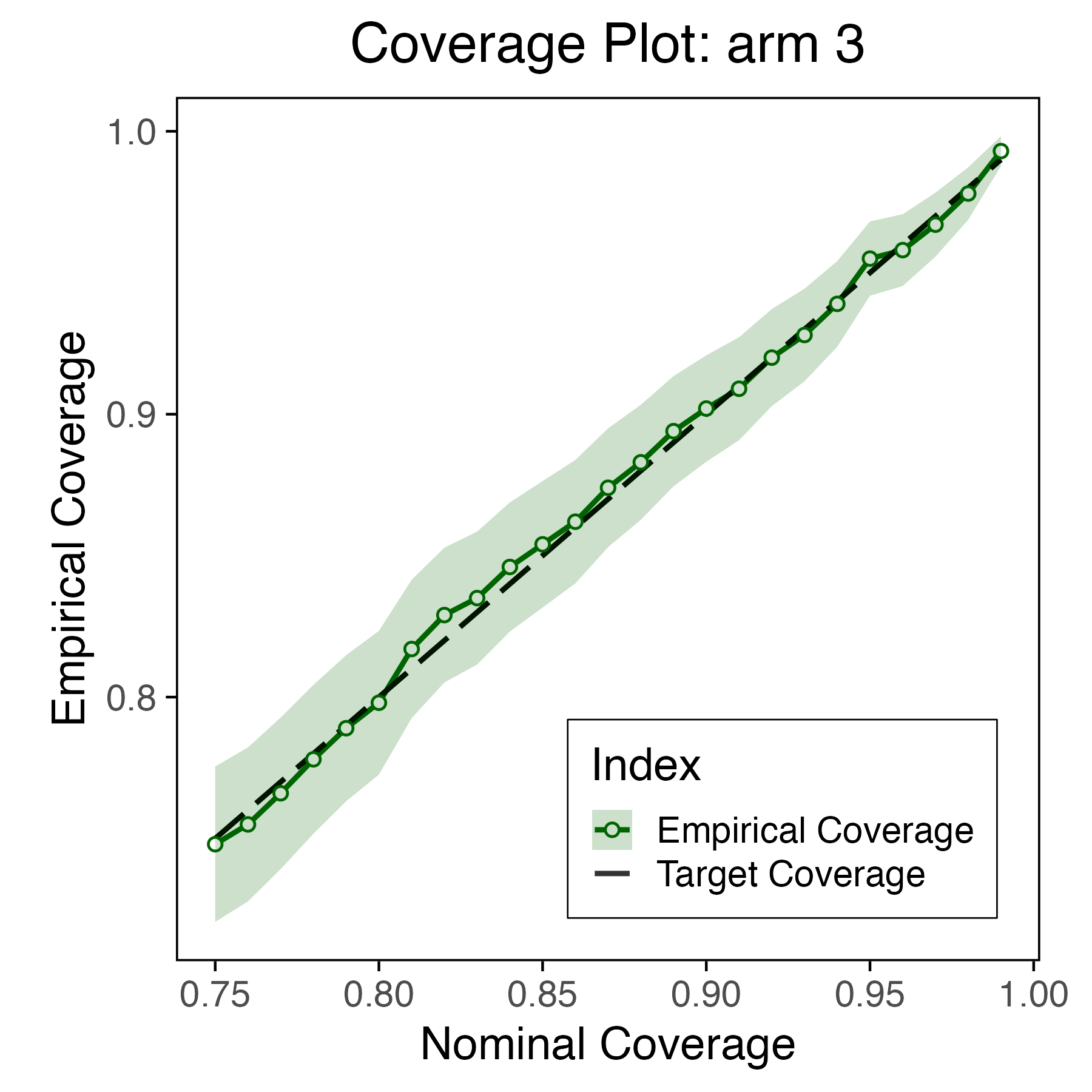}
    \end{minipage}
    \hfill
    \caption{Empirical behavior for Algorithm~\ref{alg:st-exp3} for Bernoulli bandit with $\mu=(0.9,0.3,0.1)^\top$: empirical coverage probabilities nearly aligned with diagonal.}
    \label{fig:uneq_means_cov}
\end{figure}

\noindent Next, we turn to the setting with identical arms -- consider each arm having success probability $0.7$ in a three-armed Bernoulli bandit setup. We take $T=10^5$ and $1000$ independent trials. We consider the mirror map $\phi_\alpha(\cdot)$ from \eqref{defn:mirror-map-alpha} with $\alpha=\frac{1}{2}$.

\begin{figure}[htbp]
    \centering
    \hfill
    \begin{minipage}[t]{0.3\linewidth}
        \centering
        \includegraphics[trim=0 50 0 30, clip, width=\linewidth]{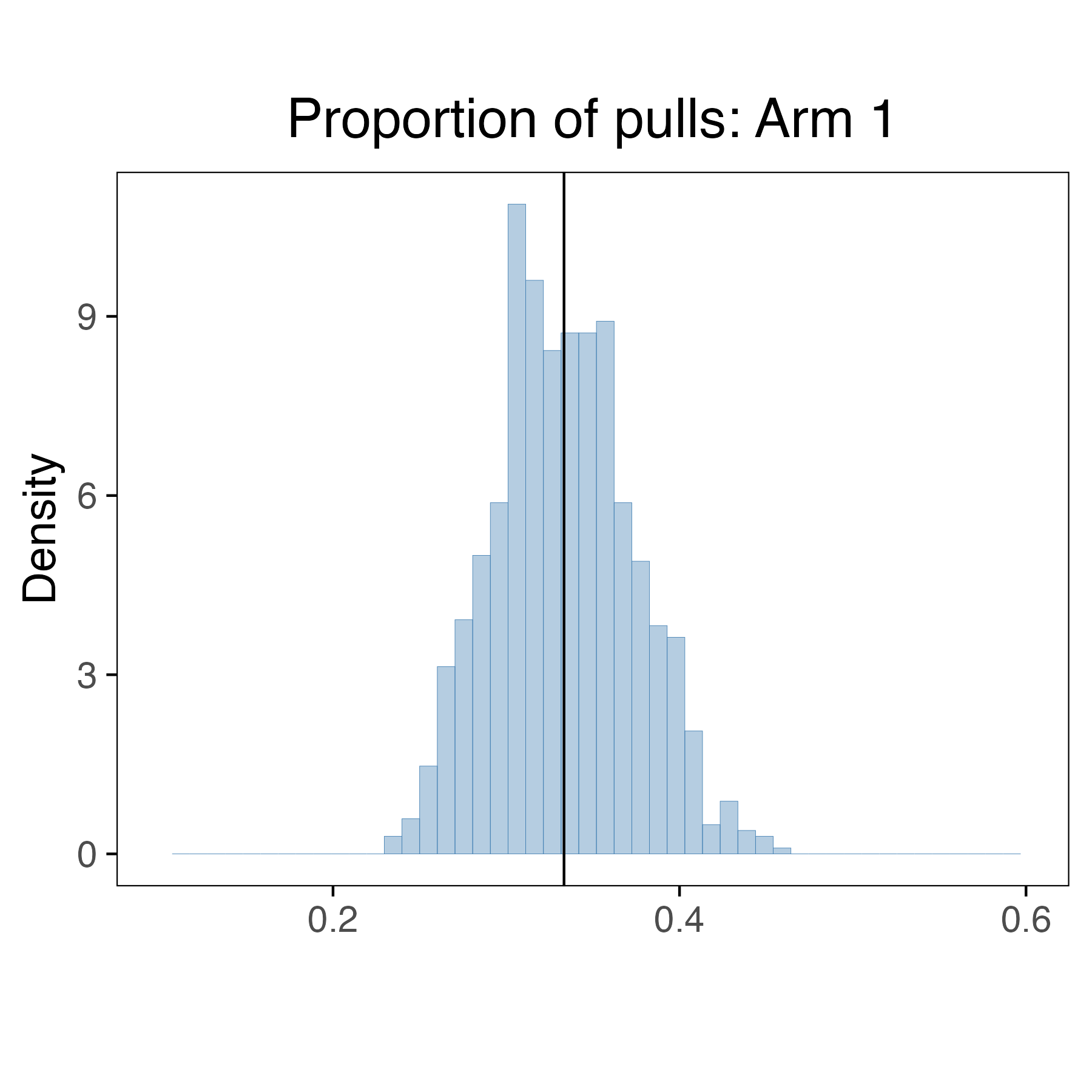}
    \end{minipage}\hfill
    \begin{minipage}[t]{0.3\linewidth}
        \centering
        \includegraphics[trim=0 50 0 30, clip, width=\linewidth]{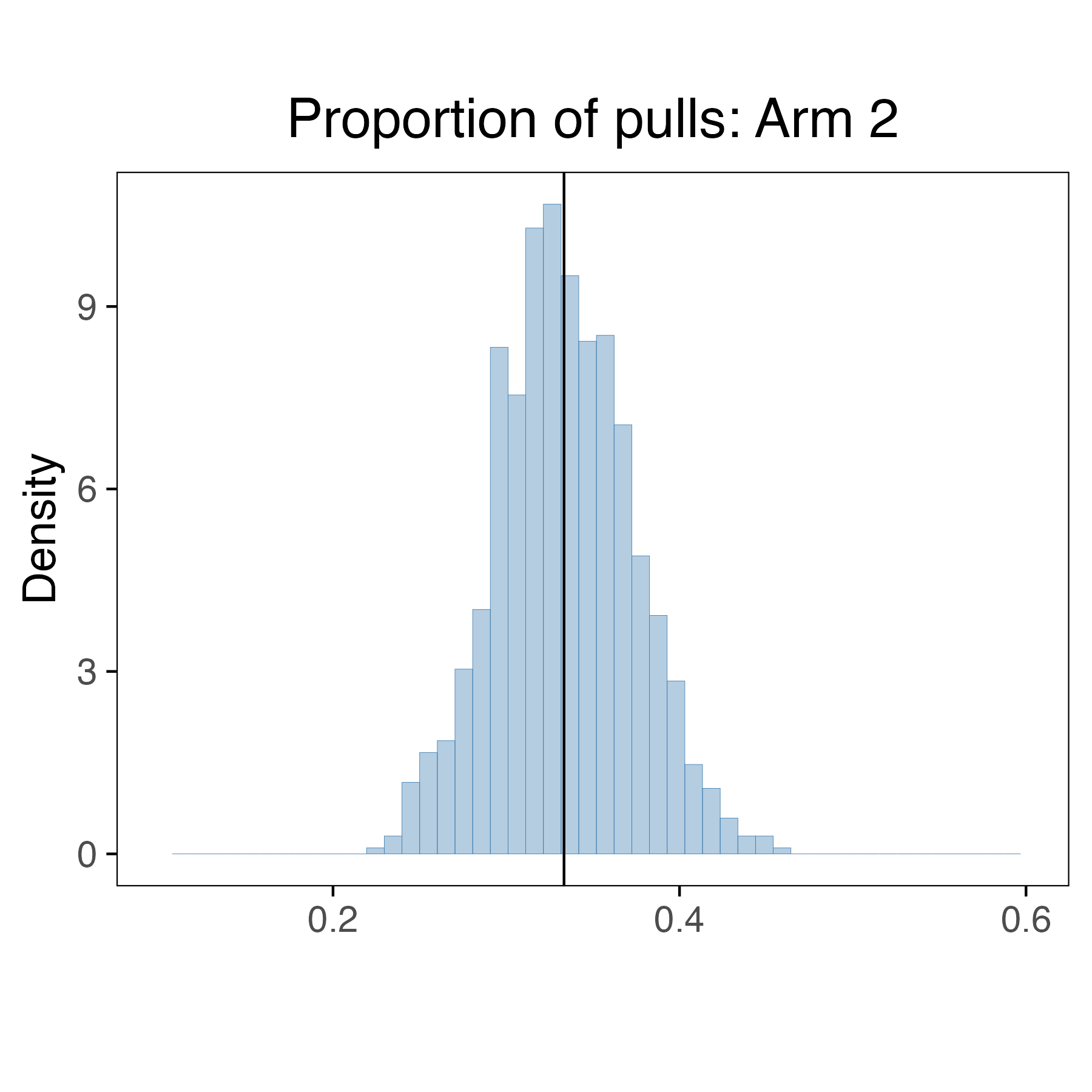}
    \end{minipage}\hfill
    \begin{minipage}[t]{0.3\linewidth}
        \centering
        \includegraphics[trim=0 50 0 30, clip, width=\linewidth]{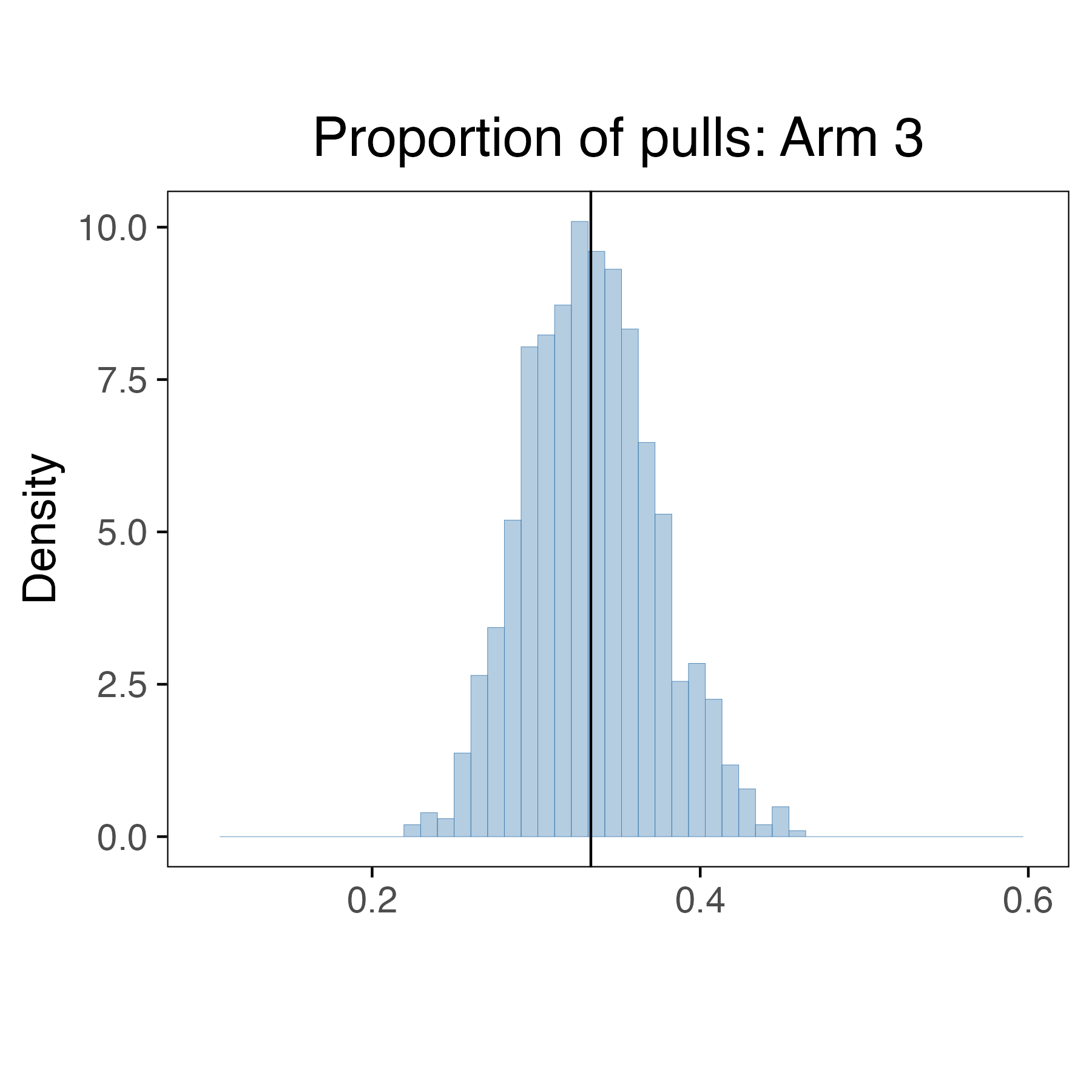}
    \end{minipage}
    \hfill
    \caption{Empirical behavior for Algorithm~\ref{alg:st-exp3} for Bernoulli bandit with $\mu=(0.7,0.7,0.7)^\top$ and $\alpha=1/2$: the proportion of pulls concentrate around $1/3$ for each arm.}
    \label{fig:prop_pulls_eq}
\end{figure}

\begin{figure}[htbp]
    \centering
    \hfill
    \begin{minipage}[t]{0.3\linewidth}
        \centering
        \includegraphics[trim=0 30 0 50, clip, width=\linewidth]{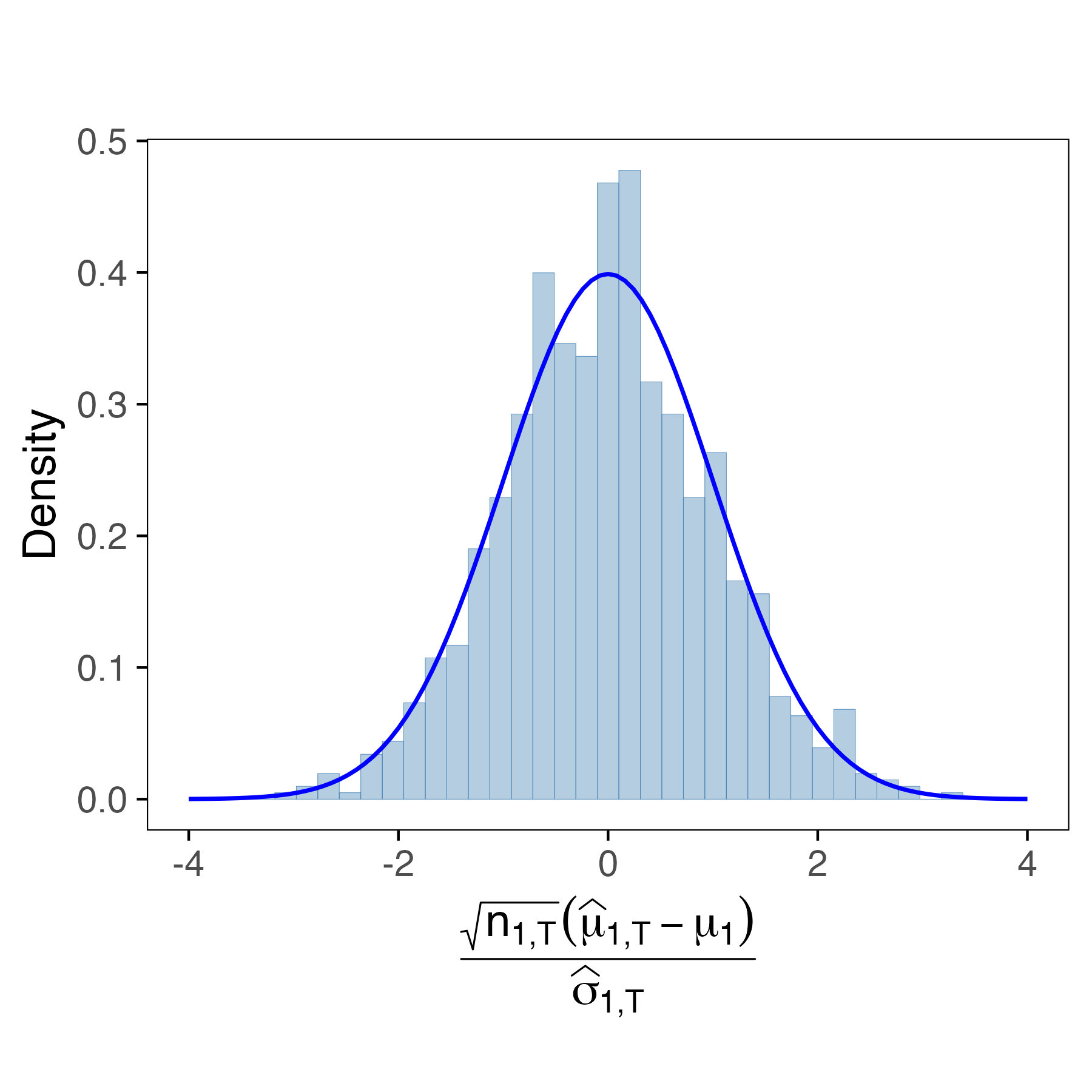}
    \end{minipage}\hfill
    \begin{minipage}[t]{0.3\linewidth}
        \centering
        \includegraphics[trim=0 30 0 50, clip, width=\linewidth]{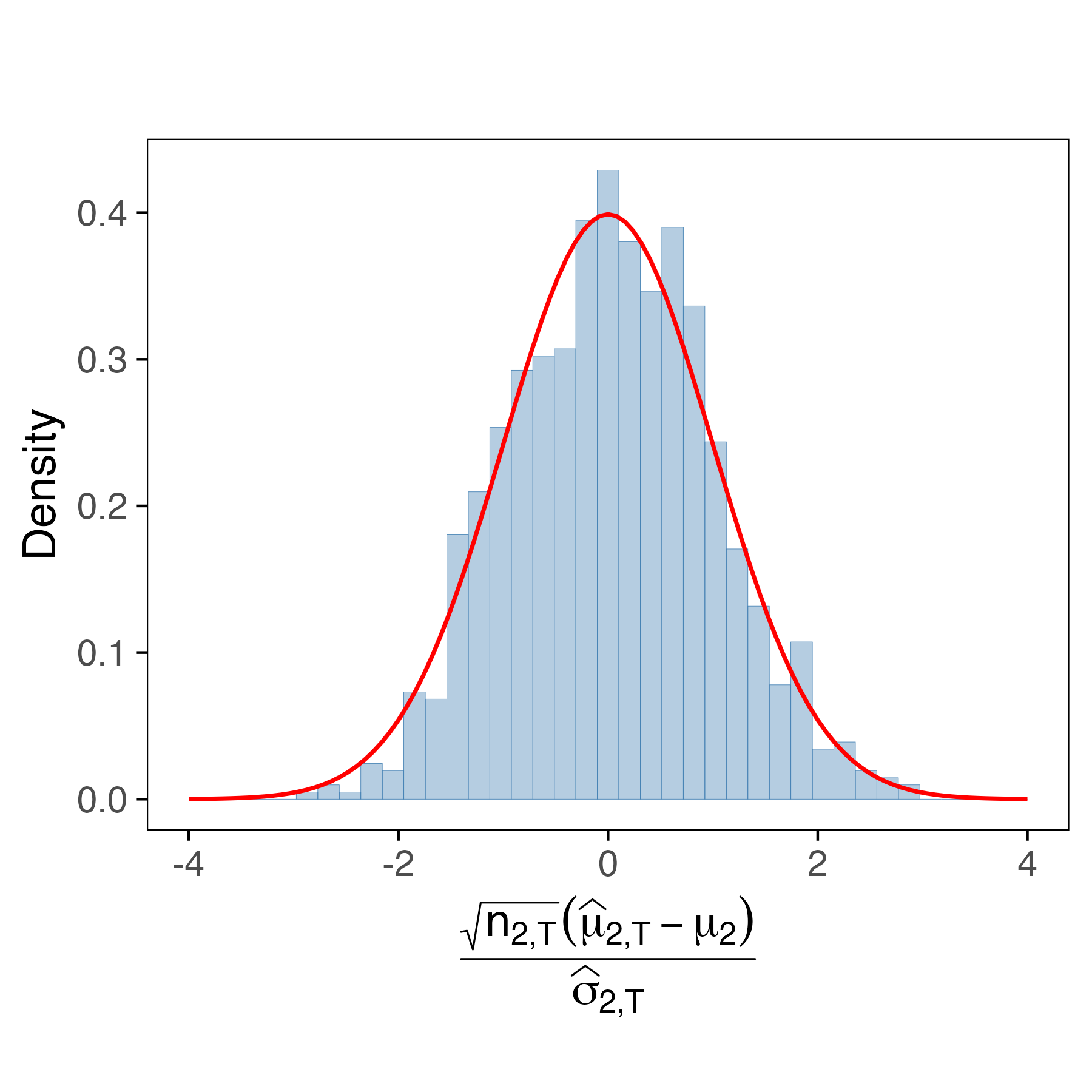}
    \end{minipage}\hfill
    \begin{minipage}[t]{0.3\linewidth}
        \centering
        \includegraphics[trim=0 30 0 50, clip, width=\linewidth]{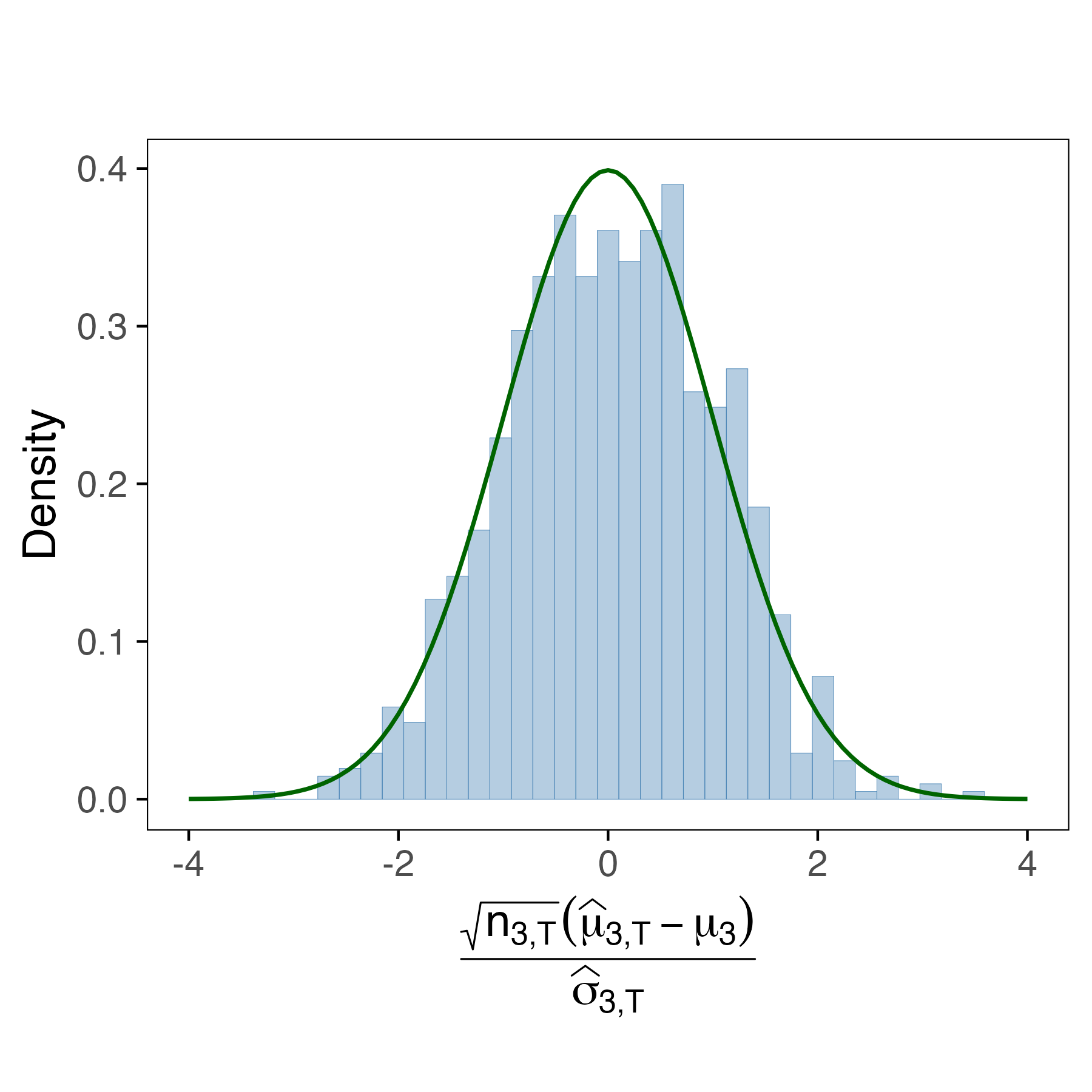}
    \end{minipage}
    \hfill
    \caption{Empirical behavior for Algorithm~\ref{alg:st-exp3} for Bernoulli bandit with $\mu=(0.7,0.7,0.7)^\top$ and $\alpha=1/2$ : standardized estimation errors \(\sqrt{n_{a,T}}(\widehat{\mu}_{a,T} - \mu_a)/\widehat{\sigma}_{a,T}\) are approximately standard normal.}
    \label{fig:eq_means_std}
\end{figure}

\noindent From Figure~\ref{fig:prop_pulls_eq}, it is evident that stability holds for each arm with $n_{a,T}^\star=\frac{T}{3}$ confirming our result in Theorem~\ref{thm:berry}. On the other hand, in Figure~\ref{fig:eq_means_std}, the density of the standardized error matches the standard gaussian density as per our predicted theory.

\section{Proof of Theorem~\ref{thm:berry}}

\noindent We prove the theorem in two parts. First we prove that the proposed Algorithm~\ref{alg:st-exp3} is stable, which is then followed by the proof of the Berry-Esseen theorem.

\subsection*{Part (i): Proof of $L_1$ stability~\eqref{eq:L1-conv} : }

\noindent Let $ f_{\lmep}(x):=\langle \mu,x\rangle+\lambda R_\epsilon(x)$ and define
\begin{align}
    \xopt := \arg \min_{x \in \Delta_{\vars}} f_{\lmep}(x)
\end{align}

\noindent  Consider the following decomposition.
\begin{align*}
    \mathbb E\left[\left|\frac{\overline p_{T,a}}{\xtas}-1\right|\right] \le\mathbb E\left[\frac{\overline p_{T,a}}{\xtas}-\log\left(\frac{\overline p_{T,a}}{\xtas}\right)-1\right]+\mathbb E\left[\left|\log\left(\frac{\overline p_{T,a}}{\xtas}\right)\right|\right].
\end{align*}

\noindent Now observe that,
\begin{align*}
&\mathbb E
\left[
\left|
\log \left(
\frac{\overline p_{T,a}}{\xtas}
\right)
\right|
\right] \\[8pt]
&=
\mathbb E
\left[
\left|
\log \left(
\frac{\overline p_{T,a}}{\xtas}
\right)
\right|
\mathbf 1
\left\{
\frac{\overline p_{T,a}}{\xtas}
\in (1-\delta_T,1+\delta_T)
\right\}
\right]
\nonumber \\[8pt]
&+
\mathbb E
\left[
\left|
\log \left(
\frac{\overline p_{T,a}}{\xtas}
\right)
\right|
\mathbf 1
\left\{
\frac{\overline p_{T,a}}{\xtas}
\notin (1-\delta_T,1+\delta_T)
\right\}
\right]
\\[8pt]
& \overset{(i)}{\le}
|\log(1-\delta_T)|
\vee
\log(1+\delta_T)
\nonumber
+
\log\left(
\frac1{\vars}
\right)
\mathbb P
\left(
\frac{\overline p_{T,a}}{p_{T,a}^{*}}
\notin
(1-\delta_T,1+\delta_T)
\right)
\\[8pt]
&
\overset{}{=}
|\log(1-\delta_T)|
\vee
\log(1+\delta_T)
\nonumber
+
\log\left(
\frac1{\vars}
\right)
\,
\mathbb P
\left(
\left|
\frac{\overline p_{T,a}}{\xtas}-1
\right|
>\delta_T
\right)
\\[8pt]
&\overset{(ii)}{\le}
|\log(1-\delta_T)|
\vee
\log(1+\delta_T)
\nonumber
+
\log\left(
\frac1{\vars}
\right)
\,
\mathbb E
\left[
\frac{\overline p_{T,a}}{\xtas}
-
\log\left(
\frac{\overline p_{T,a}}{\xtas}
\right)
-1
\right]
\left(
\frac4{\delta^2_T}+10
\right)
\\[8pt]
&\le
|\log(1-\delta_T)|
\vee
\log(1+\delta_T)
\nonumber
+
\log\left(
\frac1{\vars}
\right)
\left\{
\E\left[IS(\xT,p^\star_{\lmep})\right]
\right\}
\left(
\frac4{\delta^2_T}+10
\right).
\end{align*}

\noindent
In the above chain of inequalities, $(i)$ holds because both the the weights  $\xtas, \overline{p}_{T,a} \in \Delta_{\vars}$ and $(ii)$ follows from the lemma stated below.
\begin{lemma}\label{lemma:ZT-concet}
    Let $\{Z_T \}$ be a sequence of random variables and define $g(x):=x-\log x-1$. Then for any $\delta>0$ we have
    \begin{align*}
        \mathbb P(|Z_T - 1|>\delta) \leq \mathbb E [g(Z_T)] \cdot \left( \frac{4}{\delta^2} + 10 \right).
    \end{align*}
\end{lemma}

\noindent Therefore,
\begin{align}
\notag
&\mathbb E
\left[
\left|
\frac{\overline p_{T,a}}{\xtas}-1
\right|
\right] \\[8pt] \notag
&\le
\E\left[IS(\xT,p^\star_{\lmep})\right]
\ + \ 
|\log(1-\delta_T)|
\vee
\log(1+\delta_T)
+
\log\left(
\frac1{\vars}
\right)
\left\{
\E\left[IS(\xT,p^\star_{\lmep})\right]
\right\}
\left(
\frac4{\delta^2_T}+10
\right) \\[8pt] 
&
\leq |\log(1-\delta_T)|
\vee
\log(1+\delta_T) 
 \ + \  15 \cdot \log \left(
\frac1{\vars}
\right)
\left\{
\E\left[IS(\xT,p^\star_{\lmep})\right]
\right\} \frac{1}{\delta^2_T}
\end{align}

\noindent The above inequality  holds for any $\delta_T$ converging to zero. We wish to find a choice which provides a tight upper bound. Note that by Taylor expansion we have $\log (1+ \delta_T) = \delta_T \ + \ o(\delta_T)$ and consequently, 
\begin{align}
   |\log(1-\delta_T)|
\vee
\log(1+\delta_T)  
= \delta_T + o(\delta_T)
\end{align}

\noindent Since $o(\delta_T)$ converges to $0$ faster than $\delta_T$ we ignore this additional term. Hence, the upper bound takes the form $\delta_T + x /\delta^2_T$. We choose $\delta_T$ such that this upper bound gets minimized. Define $f(\delta) = \delta + x/\delta^2$ and observe that $f'(\delta) = 0$ when $\delta = (2x)^{1/3}$ and $f''(\delta) > 0$. These observations show that 
\begin{align} \label{eqn:Ltwo-bdto}
\mathbb E
\left[
\left|
\frac{\overline p_{T,a}}{\xtas}-1
\right|
\right] 
\lesssim  \left( 30^{1/3} + \left(\frac{15}{2}\right)^{1/3} \right)  \left( \cdot \log\left(
\frac1{\vars}
\right)
\left\{
\E\left[IS(\xT,p^\star_{\lmep})\right]
\right\}  \right)^{1/3}
\end{align}


\noindent To quantify this upper bound more explicitly, recall that $\pen$ is the regularizer, $\setpz$ is the step size and $\vars$ is the lower bound on the probability simplex $\Delta$.  Now consider Lemma~\ref{lemma:IS-bound} stated below.
\begin{lemma}\label{lemma:IS-bound}
    Let $\pen$ be the regularizer, $\setpz$ be the step size and $\vars$ be the lower bound on the probability simplex. Then the following holds
    \begin{align*}
        \E\left[IS(\xT,p^\star_{\lmep})\right]\le \frac{C_\alpha(K,\varepsilon)}{\setpz\pen T}+\frac{\setpz K}{\pen}+\frac{\setpz \pen}{\vars^2}
    \end{align*}
    where,
\begin{align*}
    C_\alpha(K, \vars) = \begin{cases}
3K \log(\frac{1}{\vars}), & \text{if } \alpha \in \left[ 0, \frac{1}{3}\right), \\[6pt]
3 K \log K, & \text{if } \alpha \in \left[ \frac{1}{3} , 1\right].
\end{cases}
\end{align*}
\end{lemma}

\noindent Therefore, it follows that
\begin{align*}
  & \log \left(\frac{1}{\vars}\right) \left[ \frac{C_\alpha(K, \vars)}{\setpz \pen T} + K\frac{\setpz}{\pen}  + \frac{\setpz \pen}{\varepsilon^2_T}  \right]\\[8pt]
   &\leq \frac{3K\log(1/\vars) C_\alpha(K, \vars)}{\setpz \pen T}
   + K\frac{\setpz \log(1/\vars)}{\pen}  + \frac{\setpz \pen \log(1/\vars)}{\varepsilon^2_T}
\end{align*}

\noindent Now, let us choose $\pen := \hTo \log T/\sqrt{T},$
 $\setpz = 1/\sqrt{T}$ and $\vars = \hTto/\sqrt{T}$. Then it follows that,
 \begin{align*}
    & \frac{3 K \log (\sqrt{T}/\hTto) C_\alpha(K, \vars)}{\hTo \log T }
   + K\frac{ \log(\sqrt{T}/\hTto)}{\hTo \log T}  + \frac{\hTo \log T \log(\sqrt{T}/\hTto)}{ \hTto^2 } \\[8pt]
   &\leq 
   \frac{3K C_\alpha(K, \vars)}{2\hTo  }
   + K\frac{ 1}{\hTo}  + \frac{\hTo (\log T)^2 }{ 2\hTto^2 }
 \end{align*}

\noindent Therefore,
\begin{align} \label{eqn:Reg-bddo}
   \log\left(
\frac1{\vars}
\right)
\left\{
\E\left[IS(\xT,p^\star_{\lmep})\right]
\right\}
\lesssim 
 \frac{3K C_\alpha(K, \vars)}{2\hTo  }
    + \frac{\hTo (\log T)^2 }{2 \hTto^2 }
\end{align}

\noindent Motivated from these calculations  we define
\begin{align*}
    \rate := (2 \cdot 15^{1/3}) \left(
    \frac{3K C^\star_{K,T}(\alpha)}{\hTo  }
    + \frac{\gamma_T(\log T)^2}{h_T^2}
  \right)^{1/3}, \quad \text{where} \ C^\star_{K,T}(\alpha) := \begin{cases}
      \log K, \  \text{if} \  \alpha \in  [1/3,1] \\
      \frac {\log T} 2, \  \text{if} \  \alpha \in  [0, 1/3]     
  \end{cases}
\end{align*}

\noindent Therefore, by combining equations~\eqref{eqn:Ltwo-bdto} and~\eqref{eqn:Reg-bddo} and defining $p^\star = \xtas$ we have our desired result:
\begin{align*}
    \Exs \left[\bigg|\frac{\xbta}{\xtas} - 1 \bigg| \right] \leq \rate
\end{align*}

\subsection*{Part (ii): Proof of quantitative central limit theorem: }

\noindent The proof is as follows. Observe the following decomposition:
\begin{align*}
    \sqrt{n_{a,T}} (\widehat{\mu}_{a,T} - \mu_a)
    &= \frac{1}{\sqrt{n_{a,T}}} \sum^T_{t =1} \varepsilon_{a,t} \indi_{\left\{ A_t = a\right\} } \\[8pt]
    &= \underbrace{\frac{1}{\sqrt{\Exs[T \xbta ]}} \sum^T_{t =1} \varepsilon_{a,t} \indi_{\left\{ A_t = a\right\} }}_{\mathcal{Q}_{T}} \ + \ \underbrace{\left(\sqrt{\frac{\Exs \left[ T \xbta\right]}{ n_{a,T}}} - 1 \right) \frac{1}{\sqrt{\Exs[T \xbta ]}} 
     \sum^T_{t =1} \varepsilon_{a,t} \indi_{\left\{ A_t = a\right\} }}_{\mathcal{D}_{T}}
\end{align*}

\noindent Define $Z_t := \varepsilon_{a,t} \indi_{\left\{ A_t = a\right\} }$. Since the noise $\varepsilon_{a,t}$ is independent of $(A_t, \Fil_{t-1})$ it follows that $(Z_t,\fil_t)$ is a MDS as,
\begin{align*}
    \Exs[Z_t \mid \Fil_{t-1}] 
    = \Exs[ \Exs[\varepsilon_{a,t} \indi_{\left\{ A_t = a\right\} } \mid A_t \ , \Fil_{t-1}] \mid \Fil_{t-1}] 
    = 0.
\end{align*}

\noindent The high-level approach of the proof is as follows. As $\mathcal{Q}_T$ is a sum of martingale difference sequences, in step $1$ we use standard Berry-Esseen results to quantify non-asymptotic error in its Gaussian approximation. In step $2$ we show that $\mathcal{D}_T$ converges to zero and quantify it's rate of convergence. By combining these two results, we obtain the quantitative central limit theorem for $\sqrt{n_{a,T}} (\widehat{\mu}_{a,T} - \mu_a)$. Consider the following lemma.
\begin{lemma} \label{lemma:berry-qt}
  Let $\pen := \hTo \log T/\sqrt{T},$
 $\setpz = 1/\sqrt{T}$ and $\vars = \hTto/\sqrt{T}$ such that $\beta < 1/2$. Then it follows that,
    \begin{align} \label{eqn:berry-qt}
    \dist \left(\mathcal{Q}_T,Z \right)
    \leq C \left(\Psi_T^{1/3} + \frac{1}{T^{1/5}\hTto^{4/5}} \right) 
    \quad \text{where,} 
\end{align}
\begin{align*}
   \rate := (2 \cdot 15^{1/3}) \left(
    \frac{3K C^\star_{K,T}(\alpha)}{4\hTo  }
    + \frac{\gamma_T(\log T)^2}{h_T^2}
  \right)^{1/3}, \quad \text{where} \ C^\star_{K,T}(\alpha) := \begin{cases}
      \log K, \  \text{if} \  \alpha \in  [1/3,1] \\
      \frac{\log T}{2}, \  \text{if} \  \alpha \in  [0, 1/3]     
  \end{cases}
\end{align*}
\end{lemma}

\noindent Now, fix a positive sequence $\kapo$  whose explicit choice will be made later in the proof. Define event  $ \Eone := \left\{ |\resiT| \leq \kapo \right\}$ and note that on event $\Eone$ we have,
\begin{align}\label{eqn:ratto-sand}
    \cuty - \kapo \leq \sqrt{n_{a,T}} (\widehat{\mu}_{a,T} - \mu_a) \leq \cuty + \kapo
\end{align}

\noindent  We use  equation~\eqref{eqn:ratto-sand} to derive lower and upper bounds on $\Prob(\stao \leq x)-\Phi(x)$, which are free of $x$. 
\begin{align*}
    \Prob(\sqrt{n_{a,T}} (\widehat{\mu}_{a,T} - \mu_a) \leq x) - \Phi(x)
    &\geq  \Prob(\sqrt{n_{a,T}} (\widehat{\mu}_{a,T} - \mu_a) \leq x, \Eone) - \Phi(x)\\[8pt]
    &\geq  \Prob(\cuty + \kapo \leq x,\  \Eone) - \Phi(x)\\[8pt]
    &\geq  \Prob(\cuty \leq x -\kapo) - \Phi(x) - \Prob(\Eone^c )
\end{align*}

\noindent From the last inequality we have:
\begin{align*}
    &\Prob(\sqrt{n_{a,T}} (\widehat{\mu}_{a,T} - \mu_a) \leq x) - \Phi(x)\\[8pt]
    & \geq \bigg[ \underbrace{\Prob(\cuty \leq x-\kapo) - \Phi(x-\kapo)}_{\opticlass_1} \bigg]
    + \bigg[ \underbrace{\Phi(x-\kapo) - \Phi(x)}_{\opticlass_2} \bigg]
    - \Prob(\Eone^c)
\end{align*}

\noindent From Lemma~\ref{lemma:berry-qt} we know that $\dist \left(\cuty, Z \right) \lesssim  \Psi^{1/3}_T + \frac{1}{T^{1/5}\hTto^{4/5}}$. Then we have $\opticlass_1 \gtrsim  - \Psi^{1/3}_T - \frac{1}{T^{1/5}\hTto^{4/5}}$. Furthermore by Lagrange's mean value theorem we note that
\begin{align*}
   |\Phi(x) - \Phi(x - \kapo)| = \kapo |\phi(x_0)| \quad \text{where,} \quad 
   x_0 \ \text{lies between $x$ and $x-\kapo$}.
\end{align*}
Therefore, as $\sup_x \phi(x) = \phi(0)<1$  we have $\opticlass_2 \gtrsim -\kapo$. Hence, for all $x\in \real$,
\begin{align}
  \Prob(\sqrt{n_{a,T}} (\widehat{\mu}_{a,T} - \mu_a) \leq x) - \Phi(x) \gtrsim - \bigg[\rate^{1/3} + \frac{1}{T^{1/5}\hTto^{4/5}} +\kapo + \Prob(\Eone^c) \bigg]  
\end{align}

\noindent Now for the upper bound, observe that
\begin{align*}
    &\Prob(\sqrt{n_{a,T}} (\widehat{\mu}_{a,T} - \mu_a) \leq x) - \Phi(x)\\[8pt]
    &=\Prob(\sqrt{n_{a,T}} (\widehat{\mu}_{a,T} - \mu_a) \leq x, \ \Eone) + \Prob(\sqrt{n_{a,T}} (\widehat{\mu}_{a,T} - \mu_a) \leq x, \ \Eone^c) - \Phi(x)\\[8pt]
    &\leq  \Prob(\cuty - \kapo \leq x, \ \Eone) + \Prob(\Eone^c ) - \Phi(x)\\[8pt]
    &\leq  \Prob(\cuty \leq x+ \kapo) - \Phi(x) + \Prob(\Eone^c )
\end{align*}

\noindent The last inequality leads to,
\begin{align*}
    &\Prob(\sqrt{n_{a,T}} (\widehat{\mu}_{a,T} - \mu_a) \leq x)  - \Phi(x)\\[8pt]
    & \leq \bigg[ \underbrace{\Prob(\cuty \leq x+\kapo) - \Phi(x+\kapo)}_{\opticlass_3} \bigg]
    + \bigg[ \underbrace{\Phi(x+\kapo) - \Phi(x)}_{\opticlass_4} \bigg]
    + \Prob(\Eone^c )
\end{align*}

\noindent Analogous calculations yield :
\begin{align}
  \Prob(\sqrt{n_{a,T}} (\widehat{\mu}_{a,T} - \mu_a) \leq x)  - \Phi(x) \lesssim \  \Psi^{1/3}_T + \frac{1}{T^{1/5}\hTto^{4/5}}+\kapo + \Prob(\Eone^c )
\end{align}
\noindent Therefore,
\begin{align*}
    \sup_{x\in \real} |\Prob(\sqrt{n_{a,T}} (\widehat{\mu}_{a,T} - \mu_a)  \leq x) - \Phi(x) |
   \ \lesssim  \ \Psi^{1/3}_T +\frac{1}{T^{1/5}\hTto^{4/5}}+\kapo + \Prob(\Eone^c )
\end{align*}

\noindent The following lemma bounds $\Prob(\Eone^c )$ which we use to obtain our final bound.
\begin{lemma} \label{lemma:Eone-comp}
    Let $m_T$ be any sequence of positive real numbers diverging to infinity. Then we have,
    \begin{align} \label{eqn:Eone-comp}
       \Prob(\Eone^c)
     \leq  \frac{2m_T \ \rate}{\kappa_{1,T} } \ + \ \frac{2 \phi(m_T)}{m_T} \ + \  4 \cdot \rate + \ 2 \cdot \rate^{1/3}
    \end{align}
\end{lemma}

\noindent We shall choose $m_T$ such that we obtain the optimal rate. By combining all these previous results we have
\begin{align*}
    \dist \left(\sqrt{n_{a,T}} (\widehat{\mu}_{a,T} - \mu_a),Z \right) 
   \  \leq \ 5 \cdot \Psi^{1/3}_T \ + \ \kappa_{1,T} \ + \   \frac{2m_T \ \rate}{\kappa_{1,T} } \ + \ \frac{2 \phi(m_T)}{m_T} \ + \  4 \cdot \rate 
\end{align*}

\noindent By optimizing over $\kappa_{1,T}$ via AM-GM inequality and ignoring the absolute constants and lower order terms we have
\begin{align*}
    \dist \left(\sqrt{n_{a,T}} (\widehat{\mu}_{a,T} - \mu_a),Z \right) 
   \  \leq \  5 \cdot\Psi^{1/3}_T  \ + \   \sqrt{m_T \ \rate}  \ + \ \frac{ \phi(m_T)}{m_T}  
\end{align*}

\noindent  Observe that since we already have $\Psi^{1/3}_T$ in the upper bound, the rate cannot be improved further. By choosing $m_T = \rate^{-1/3}$ implies that $\sqrt{m_T \rate} = \rate^{1/3}$ and $\frac{ \phi(m_T)}{m_T} = \rate^{1/3}/\exp \{ m_T^2\}$ which is uniformly dominated by $\rate^{1/3}$ and hence we ignore this lower order term. Therefore we conclude that,
\begin{align*}
   \dist \left(\sqrt{n_{a,T}} (\widehat{\mu}_{a,T} - \mu_a),Z \right)  \ \lesssim \ 
   \Psi^{1/3}_T+ \frac{1}{T^{1/5}\hTto^{4/5}}.
\end{align*}

\subsection*{Proof of Lemma~\ref{lemma:berry-qt}}

\noindent Let us first quantify the rate of convergence of $\mathcal{Q}_{T}$. We shall apply following quantitative central limit theorem for martingales (\citet{mourrat2013rate}).

\begin{lemma}\label{lemma:mart-clt}
Let $(Z_t,\mathcal{F}_t)$ be a square-integrable martingale difference sequence. Let $s^2_T:= \sum^T_{t=1} \Exs[Z^2_t]$ and $V^2_T := \dfrac{1}{s^2_T} \sum \Exs[Z^2_t|\mathcal{F}_{t-1}]$. Then for any pair $(p,q) \in [1,\infty)$, there exists some constant $C = C(p,q)>0$ such that,
\begin{align}
    \dist \left(\dfrac{1}{s_T}\sum^T_{t=1}Z_t,Z \right)
    \leq C \left[ \|V^2_T-1 \|^{\frac{p}{2p+1}}_p + \left(\dfrac{1}{s^{2q}_T}\sum^T_{t=1} \| Z_t\|^{2q}_{2q} \right)^{\frac{1}{2q+1}}  \right]
\end{align}
where $Z \sim N(0,1)$. 
\end{lemma}

\noindent Define $Z_t := \varepsilon_{a,t} \indi_{\left\{ A_t = a\right\} }$. Since the noise $\varepsilon_{a,t}$ is independent of $(A_t, \Fil_{t-1})$ it follows that $Z_t$ is a MDS as,
\begin{align*}
    \Exs[Z_t \mid \Fil_{t-1}] 
    = \Exs[ \Exs[\varepsilon_{a,t} \indi_{\left\{ A_t = a\right\} } \mid A_t \ , \Fil_{t-1}] \mid \Fil_{t-1}] 
    = 0.
\end{align*}

\noindent Observe that,
\begin{align*}
  \Exs[Z^2_t \mid \Fil_{t-1}] 
  &= \Exs[\varepsilon^2_{a,t} \indi_{\left\{ A_t = a\right\} } \mid \Fil_{t-1}] \\[8pt]
  &= \Exs[ \ \indi_{\left\{ A_t = a\right\} } \ \Exs[\varepsilon^2_{a,t}  \mid A_t, \Fil_{t-1}] \mid \Fil_{t-1}] \\[8pt]
  &= \Exs[ \indi_{\left\{ A_t = a\right\} } \mid \Fil_{t-1}] \\[8pt]
  &= p_{T,a}
\end{align*}

\noindent Therefore, it follows that $\sum^T_{t=1} \Exs[Z^2_t \mid \Fil_{t-1}] = T \xbta$, and consequently, $s^2_T = \Exs[T \xbta]$. Therefore, we conclude that
\begin{align}
   \dfrac{1}{s_T}\sum^T_{t=1}Z_t = \frac{1}{\sqrt{\Exs[T \xbta ]}} \sum^T_{t =1} \varepsilon_{a,t} \indi_{\left\{ A_t = a\right\} } 
\end{align}

\noindent Hence, we can apply Lemma~~\ref{lemma:mart-clt} can be applied to derive quantitative CLT for $\mathcal{Q}_T$. Now, observe that,
\begin{align*}
    \|V^2_T-1 \|_1 =  \Exs \left[\bigg|\frac{\xbta}{\Exs[\xbta]} - 1 \bigg| \right]
\end{align*}

\noindent We claim that,
\begin{align}\label{eqn:Lto-bd}
    \Exs \left[ \bigg|\frac{\xbta}{\Exs[\xbta]} - 1 \bigg| \right] \leq 2 \cdot \Exs \left[\bigg|\frac{\xbta}{\xtas} - 1 \bigg| \right]
\end{align}

\noindent The justification for the above lemma is as follows:
\begin{align*}
    \Exs \left[\bigg|\frac{\xbta}{\Exs[\xbta]} - 1 \bigg| \right]
     \leq \Exs \left[\bigg|\frac{\xbta}{\xtas} - 1 \bigg| \right] 
     \ + \ \Exs \left[\bigg|\frac{\xbta}{\Exs[\xbta]} - \frac{\xbta}{\xtas} \bigg| \right]
\end{align*}

\noindent Observe the following:
\begin{align*}
   \Exs \left[\bigg|\frac{\xbta}{\Exs[\xbta]} - \frac{\xbta}{\xtas} \bigg| \right]
    &= \Exs \left[ \xbta\bigg|\frac{\xtas-\Exs[\xbta]}{\Exs[\xbta] \xtas}  \bigg| \right] \\[8pt]
   & = \Exs \left[ \frac{\xbta}{\Exs[\xbta]}  \right]\bigg|\frac{\xtas-\Exs[\xbta]}{ \xtas}  \bigg| \\[8pt]
   & = \frac{\Exs[\xbta]}{\Exs[\xbta]} \cdot \bigg|\Exs \left[\frac{\xbta}{\xtas} - 1   \right]\bigg| \\[8pt]
   & \leq \Exs \left[\bigg|\frac{\xbta}{\xtas} - 1 \bigg| \right]
\end{align*}

\noindent Hence by combining equation~\eqref{eqn:Lto-bd} along with Theorem~\ref{thm:berry}~\ref{thm:berry-i} we conclude that,
\begin{align*}
    \|V^2_T-1 \|^{1/3}_1 \ \leq \ (2 \cdot\rate)^{1/3}
\end{align*}

\noindent Now to quantify the second term, observe that
\begin{align*}
    \dfrac{1}{T}\sum^T_{t=1} \| Z_t\|^{2q}_{2q}
    \ = \dfrac{1}{T}\sum^T_{t=1} \Exs[Z^{2q}_t]
    \ = \dfrac{1}{T}\sum^T_{t=1}\Exs[\varepsilon^{2q}_{a,t} \indi_{\left\{ A_t = a\right\} } ] 
    \ \leq \dfrac{1}{T}\sum^T_{t=1}\Exs[\varepsilon^{2q}_{a,t} ] \ \leq M_{a,2q}
\end{align*}

\noindent Hence,
\begin{align*}
    \dfrac{1}{s^{2q}_T}\sum^T_{t=1} \| Z_t\|^{2q}_{2q}
    \ &= \ \frac{T}{\Exs[T \xbta]^{2q}} \cdot \left[ \dfrac{1}{T}\sum^T_{t=1} \| Z_t\|^{2q}_{2q} \right]
    \\[8pt]
    & \overset{(i)}{\leq} \frac{1}{T^{2q-1} \epsilon^{2q}_T} \cdot \left[ \dfrac{1}{T}\sum^T_{t=1} \| Z_t\|^{2q}_{2q} \right]
    \\[8pt]
    &  \overset{(ii)}{\leq } \frac{1}{T^{2q-1} \epsilon^{2q}_T}   \\[8pt]
    & =  \frac{1}{T^{q-1} \hTto^{2q}} 
\end{align*}

\noindent Hence, it follows that
\begin{align*}
    \dist \left(\dfrac{1}{s_T}\sum^T_{t=1}Z_t,Z \right)
    \leq C \left[ (2 \cdot\rate)^{1/3} + \left(\frac{1}{T^{q-1} \hTto^{2q}}  \right)^{\frac{1}{2q+1}}  \right]
\end{align*}

\noindent The above inequalities follow because we only assume martingale noise are uniformly bounded above by $1$. We choose $q=2.$

\section{Proof of Theorem~\ref{thm:prec-reg}}

\noindent We begin the proof by noting the following string of inequalities.
\begin{align}\label{eqn:prec-reg}
\notag
    \bigg| \frac{\regret(T)}{\regret^\star(T)}  - 1\bigg|
    = \bigg| \frac{\sum_{a\in \action} \Delta_a \Exs[n_{a,T}]}{\sum_{a'\in \action} \Delta_{a'} n^\star_{a',T}}  - 1\bigg|
    &= \bigg| \frac{\sum_{a\in \action} \Delta_a \left(\Exs[n_{a,T}] - n^\star_{a,T} \right)}{\sum_{a'\in \action} \Delta_{a'} n^\star_{a',T}} \bigg| \\[8pt] \notag
    & \leq \ \sum_{a \in \action}\left( \frac{\Delta_a  n^\star_{a,T}}{\sum_{a'\in \action} \Delta_{a'} n^\star_{a',T} } \right) \bigg| \Exs \left[ \frac{n_{a_T}}{n^\star_{a,T}}\right] - 1 \bigg| \\[8pt]
    & \leq  \max_{a \in \action} \  \Exs \left[ \bigg| \frac{n_{a_T}}{n^\star_{a,T}} - 1  \bigg| \right] 
\end{align}

\noindent To control the $L_1$ norm, the high level idea of the proof is as follows. We write $n_{a,T}/n^\star_{a,T}$ as the product of $n_{a,T}/T\overline p_{T,a}$ and $\overline p_{T,a}/\xtas$, and control the $L_1$ norm for each of these terms separately. Let us define 
\begin{align}
   \mathcal{J}_{1,T}:= E\left[  \left|\frac{n_{a,T}}{T\overline p_{T,a}}-1\right|\right] \quad \text{and,} \quad  \mathcal{J}_{2,T}:= E\left[  \left|\frac{\overline p_{T,a}}{\xtas}-1\right|\right]
\end{align}

\noindent The lemma below states that it suffices to control $ \mathcal{J}_{1,T}$ and $ \mathcal{J}_{2,T}$ separately.
\begin{lemma}\label{lemma:jojto-upp-bdd}
    Let $\mathcal{J}_{1,T}$ and $\mathcal{J}_{2,T}$ be as defined above. Then we have
    \begin{align*}
       E\left[  \left|\frac{n_{a,T}}{T\xtas}-1\right|\right] 
       \leq  \mathcal{J}_{1,T} + \mathcal{J}_{2,T} 
    \end{align*}
\end{lemma}

\noindent Now to control $\mathcal{J}_{1,T}$ we shall apply the following concentration inequality.
\begin{lemma}\label{lemma:napt-concent}
    Let $n_{a,T}$ and $\overline{p}_{T,a}$ be as defined above. Then for any $\delta > 0$ we have
    \begin{align}
        \prob \left( \left| \frac{n_{a,T}}{T\overline{p}_{T,a}}-1 \right| > \delta\right) 
        \leq 2 \cdot \exp \left\{ -(\sqrt{1 + \delta} - 1)^2 T \vars \right\}     
    \end{align}
\end{lemma}

\noindent Let $\delta_{2,T} > 0$ be a real sequence converging to $0$, whose explicit form will be determined later. Now,
\begin{align}
\notag
&  \mathbb E
\left[
\left|
 \frac{n_{a,T}}{T\overline{p}_{T,a}}-1
\right|
\right] \\[8pt]
&=  
\mathbb E
\left[
\left|
\ \frac{n_{a,T}}{T\overline{p}_{T,a}}-1
\right|
\mathbf 1
\left\{
\frac{n_{a,T}}{T\overline{p}_{T,a}}
\in (1-\delta_{2,T},1+\delta_{2,T})
\right\}
\right]
\nonumber +
\mathbb E
\left[
\left|
 \frac{n_{a,T}}{T\overline{p}_{T,a}}-1
\right|
\mathbf 1
\left\{
\frac{n_{a,T}}{\overline{p}_{T,a}}
\notin (1-\delta_{2,T},1+\delta_{2,T})
\right\}
\right] \\[8pt]\notag
&
\overset{(*)}{\leq} \delta_{2,T} \ + \ \frac{1}{\varepsilon_T} \Prob \left( \frac{n_{a,T}}{T\overline{p}_{T,a}}
\notin (1-\delta_{2,T},1+\delta_{2,T}) \right) \\[8pt]
&\overset{(**)}{\leq}
 \delta_{2,T} \ + \ \frac{2}{\varepsilon_T} \cdot \exp \left\{ -(\sqrt{1 + \delta_{2,T}} - 1)^2 T \vars \right\}     
\end{align}

\noindent In the above chain of inequalities, $(*)$ follows from that fact that both $n_{a,T}$ and $T \overline{p}_{T,a}$ are bounded by $T$ and $\overline{p}_{T,a} \geq \vars$ whereas, $(ii)$ follows from Lemma ~\ref{lemma:napt-concent}. Let
$(\sqrt{1 + \delta_{2,T}} - 1)^2 = \log T/\sqrt{T}$ and observe that by substituting $\vars = \hTto/\sqrt{T}$ the above equation gets simplified as,
\begin{align*}
   \delta_{2,T} \ + \ \frac{2 \sqrt{T}}{h_T T^{\hTto}} 
\end{align*}

\noindent As $h_T \to \infty$ we highlight the fact that $\sqrt{T}/T^{\hTto} \to 0$. Now since $(\sqrt{1 + \delta_{2,T}} - 1)^2 = \log T/\sqrt{T}$ we must have 
\begin{align*}
    \sqrt{1 + \delta_{2,T}} = 1 + \sqrt{\frac{\log T}{\sqrt{T}}} \quad \text{which implies that,} \ \delta_{2,T} = \left( 1 + \sqrt{\frac{\log T}{\sqrt{T}}}\right)^2-1
\end{align*}

\noindent Further simplification yields that $\delta_{2,T} = 2\sqrt{\frac{\log T}{\sqrt{T}}} +  \frac{\log T}{\sqrt{T}}$ and hence we have,
\begin{align}
  \mathcal{J}_{1,T}
\leq 
2\sqrt{\frac{\log T}{\sqrt{T}}} +  \frac{\log T}{\sqrt{T}} + \frac{2 \sqrt{T}}{h_T T^{\hTto}}
\end{align}

\noindent Now, Theorem~\ref{thm:berry-i} states that,
 \begin{align}
     \Exs \left[\bigg|\frac{\xbta}{\xtas} - 1 \bigg| \right] \leq \rate
 \end{align}
 where 
 \begin{align*}
    \Psi_T:= (2 \cdot 15^{1/3}) \left(
    \frac{3K C^\star_{K,T}(\alpha)}{\hTo  }
    + \frac{\gamma_T(\log T)^2}{h_T^2}
  \right)^{1/3}, \quad \text{where} \ C^\star_{K,T}(\alpha) := \begin{cases}
      \log K, \  \text{if} \  \alpha \in  [1/3,1] \\
      \frac {\log T} 2, \  \text{if} \  \alpha \in  [0, 1/3].     
  \end{cases}
 \end{align*}

\noindent Therefore, we have
\begin{align*}
    \left|\frac{\regret(T)}{\regret^\star(T)}-1\right|
  \;\leq\;\Psi_T + 2\sqrt{\frac{\log T}{\sqrt{T}}} +  \frac{\log T}{\sqrt{T}} + \frac{2 \sqrt{T}}{h_T T^{\hTto}}
\end{align*}

\noindent along with equation~\eqref{eqn:prec-reg} we obtain the desired result. Now, once we characterize the idealized regret $\regret^\star(T)$ we are done. Recall that $\regret^\star(T) = \sum_{a \in \action} \Delta_a n^\star_{a,T}$ where  $n^\star_{a,T} = T \xopt$ and,
\begin{align*}
   \xopt = \arg \min_{\rho \in \Delta_{\vars}} \left[  \langle \mu, \rho \rangle + \pen R(\rho) \right].
\end{align*}

\noindent Our first observation is that $ \xopt$ has a closed form solution. We are interested in the optimization solution with a strictly convex objective function with constraints $\sum_a{\rho_a} = 1$ and $\rho \geq \vars$. By considering the Lagrangian and applying the KKT conditions we have:
\begin{align}\label{eqn:KKT-main}
    \grad_\rho \left[ \left\{ \langle \mu, \rho \rangle - \pen \sum_a \log \rho_a \right\} + \nu \left( 1 - \sum_{a} \rho_a \right) + \sum_{a}\alpha_a(\vars - \rho_a) \right] = 0.
\end{align}

\noindent Further simplification yields that for every arm $a \in \action$,
\begin{align*}
    \rho_a = \frac{\pen}{\mu_a - \nu - \alpha_a} \quad \text{where $\nu$ is defined as the solution of $\sum_a \rho_a = 1$,} 
\end{align*}
and $\alpha_a(\vars - \rho_a) = 0$, $\alpha_a \geq 0$ and $\rho_a \geq \vars$. Therefore whenever $\rho_a > \vars$ we have $\alpha_a = 0$ which implies that,
\begin{align*}
    \rho_a = \begin{cases}
        \frac{\pen}{\mu_a - \nu}, \quad  \text{if} \  \rho_a > \vars, \\[8pt]
        \vars, \quad \text{otherwise}
    \end{cases}
\end{align*}

\noindent A careful look indicates that we are rewrite the solution $\rho_a$ more concisely by observing that when $\rho_a = \vars$ $\alpha_a \geq 0$ and hence if $\mu_a - \nu > 0$ then,
\begin{align*}
    \frac{\pen}{\mu_a - \nu} \leq \frac{\pen}{\mu_a - \nu - \alpha_a} = \vars.
\end{align*}

\noindent Alternately, if $\mu_a - \nu < 0,$ then $\frac{\pen}{\mu_a - \nu} \leq  \vars$ holds trivially. Therefore, from these observations along with the definition of $\xopt$  we conclude that:
\begin{align}\label{eqn:xstar-charac}
    \xtas = \max \left(\vars,  \frac{\pen}{\mu_a - \nu^\star_T} \right)
    \quad \text{where $\nu^\star_T$ is the solution of}\  \sum_a  \max \left(\vars,  \frac{\pen}{\mu_a - \nu} \right) = 1.
\end{align}

\section{Discussion}
\label{sec:Discussion}
This work addresses a fundamental tension in stochastic multi-armed bandit problems: adaptive exploration strategies that minimize regret often violate the stability conditions required for valid statistical inference, thereby invalidating standard inferential procedures. Our results show that this tension can be resolved through appropriate regularization. More precisely, we show that a large class of penalized EXP3-type (stochastic mirror descent) algorithms simultaneously achieves stability  and minimax-optimal regret up to logarithmic factors. For these algorithms, we further establish a quantitative central limit theorem and a precise characterization of the regret.

Beyond stability, our analysis reveals a robustness property: the proposed algorithm tolerates up to $o\left(T^{1/2}\right)$ adversarially corrupted rewards while preserving both its regret guarantees and inferential validity. In marked contrast, competing stable algorithms---most notably UCB-based methods---incur linear regret under as few as $O(\log T)$ corruptions. These findings indicate that stability need not entail fragility to data contamination, a property of practical relevance in noisy environments.

Several directions merit further investigation. First, extending this framework to structured settings such as contextual or non-stationary bandits remains an open problem. Second, characterizing the minimal regularization required to achieve the threefold objective of adaptivity, stability, and robustness would sharpen our theoretical understanding. Finally, these results suggest that, for applications requiring valid inference, stability should be treated as a core algorithmic requirement on a par with regret minimization.



\bibliography{arXiv}

\appendix


\newpage

\section{Proofs of Key Lemmas}

\noindent This section contains the proof of all imporant lemmas used in the main section of the paper.

\subsection{Proof of Lemma~\ref{lem:normality}}

\noindent Fix an arm $a$. Let $\mathcal F_{t-1}$ denote the history before action
selection at time $t$, and define
\[
p_{t,a}:=\Prob(A_t=a\mid \mathcal F_{t-1}).
\]
\noindent Suppose that $\bar p_{T,a}:=\frac1T\sum_{t=1}^T p_{t,a}$
satisfies
\[
\frac{\bar p_{T,a}}{p_{T,a}^\star}\xrightarrow{\prob} 1
\]
for some constant $p_{T,a}^\star>0$. Assume that conditional on
$A_t=a$ and $\mathcal F_{t-1}$,
\[
\E[\ell_t\mid A_t=a,\mathcal F_{t-1}]=\mu_a,
\qquad
\Var(\ell_t\mid A_t=a,\mathcal F_{t-1})=1.
\]
Assume also that the rewards are uniformly bounded, say $0\le \ell_t\le 1$.
Let
\[
n_{a,T}:=\sum_{t=1}^T \mathbf 1\{A_t=a\},
\qquad
\widehat\mu_{a,T}
:=
\frac{1}{n_{a,T}}\sum_{t=1}^T \mathbf 1\{A_t=a\}l_t.
\]
Then
\[
\sqrt{n_{a,T}}\left(\widehat\mu_{a,T}-\mu_a\right)
\xrightarrow{d} N(0,1).
\]

\noindent Now, define $X_{t,a}:=\mathbf 1\{A_t=a\}(l_t-\mu_a).$
Then $\widehat\mu_{a,T}-\mu_a
=
\frac{1}{n_{a,T}}\sum_{t=1}^T X_{t,a},$
and therefore,
\[
\sqrt{n_{a,T}}\left(\widehat\mu_{a,T}-\mu_a\right)
=
\frac{1}{\sqrt{n_{a,T}}}\sum_{t=1}^T X_{t,a}.
\]

\noindent First, we show that $\{X_{t,a},\mathcal F_t\}_{t\ge 1}$ is a martingale
difference sequence. Indeed,
\begin{align*}
\Exs[X_{t,a}\mid \mathcal F_{t-1}]
&=
\Exs\left[
\mathbf 1\{A_t=a\}(\ell_t-\mu_a)
\mid \mathcal F_{t-1}
\right] \\
&=
\Exs[1\{A_t=a\} \Exs[\ell_t-\mu_a\mid A_t=a,\mathcal F_{t-1} ]\mid \mathcal F_{t-1}] \\
&=0.
\end{align*}
Moreover,
\begin{align*}
\Exs[X_{t,a}^2\mid \mathcal F_{t-1}]
&=
\E\left[
\mathbf 1\{A_t=a\}(\ell_t-\mu_a)^2
\mid \mathcal F_{t-1}
\right] \\
&=
\Exs[1\{A_t=a\} \Exs[(\ell_t-\mu_a)^2\mid A_t=a,\mathcal F_{t-1} ]\mid \mathcal F_{t-1}] \\
&=p_{t,a}.
\end{align*}
Thus the sum of conditional variances is
\[
V_{T,a}
:=
\sum_{t=1}^T \E[X_{t,a}^2\mid \mathcal F_{t-1}]
=
\sum_{t=1}^T p_{t,a}.
\]
By assumption,
\[
\frac{V_{T,a}}{Tp_{T,a}^\star}
=
\frac{\bar p_{T,a}}{p_{T,a}^\star}
\xrightarrow{\prob} 1
\]
In particular, since $p_{T,a}^\star \geq \vars$ and $T\vars \to \infty$  we have $V_{T,a}\to\infty$ in probability. 
Now we verify the conditional Lindeberg condition required by the martingale
CLT: for any arbitrary $\delta>0$,
\[
L_{T,a}(\delta)
:=
\frac{1}{V_{T,a}}
\sum_{t=1}^{T}
\E\!\left[
  X_{t,a}^{2}
  \,\mathbf{1}\!\left\{|X_{t,a}|>\delta\sqrt{V_{T,a}}\right\}
  \;\Big|\;
  \mathcal{F}_{t-1}
\right]
\xrightarrow{\prob} 0
\quad\text{as }T\to\infty.
\]

\noindent Fix $\delta>0$. Note that for each $t\in [T]$ we have $p_{t,a} \geq \varepsilon_t \geq \vars$. Therefore, $\bar p_{T,a} \geq \vars$ which implies that $V_{T,a} \geq T\vars$. Hence, $V_{T,a} \to \infty$ as $T \to \infty$ almost surely. Therefore, for any fixed $\delta$, there exists $T_0$ such that for all $T \geq T_0$ we must have $\mathbf{1}\!\left\{|X_{t,a}|>\delta\sqrt{V_{T,a}}\right\} = 0$. Therefore, the Lindeberg condition is satisfied. We shall apply the following of version martingale central limit theorem~\citep{hall2014martingale}.
\begin{lemma}
    Suppose  $(Z_{T,t},\fil_{T,t})$ is a martingale difference sequence satisfying the Lindeberg condition. Let $V_{T,t} = \sum^{T}_{t=1} \Exs[Z^2_{T,t} \mid \Fil_{T,t-1}]$ and assume that $V_{T,t} \xrightarrow{\prob} \eta$ where $\prob(\eta > 0) = 1$. Then we have
    \begin{align*}
        \frac{1}{\sqrt{V_{T,t}}}\sum^{T}_{t=1} Z_{T,t} \xrightarrow{d} \mathcal{N}(0,1).
    \end{align*}
\end{lemma}

\noindent Therefore, by substituting $Z_{T,t} = X_{t,a}/\sqrt{T p^\star_{T,a}}$ we have
\begin{align*}
    \frac{1}{\sqrt{V_{T,a}}} \sum_{t=1}^T X_{t,a} \xrightarrow{d}
\mathcal{N}(0,1)
\end{align*}

\noindent It remains to show that $n_{a,T}/V_{T,a}\xrightarrow{\prob}1$.
Recall that
\[
n_{a,T}-V_{T,a}
=
\sum_{t=1}^T D_t,
\qquad
D_t:=\mathbf 1\{A_t=a\}-p_{t,a}.
\]
Then $(D_t,\mathcal F_t)$ is a martingale difference sequence and
$|D_t|\le 1$. Moreover,
\[
\sum_{t=1}^T
\Var(D_t\mid \mathcal F_{t-1})
=
\sum_{t=1}^T p_{t,a}(1-p_{t,a})
\le
\sum_{t=1}^T p_{t,a}
=
V_{T,a}.
\]

\noindent Fix $\varepsilon>0$. For any deterministic $M>0$,
\begin{align*}
\Prob\left(
\left|n_{a,T}-V_{T,a}\right|>\varepsilon V_{T,a}
\right)
&\le
\Prob(V_{T,a}\le M) +
\Prob\left(
\left|\sum_{t=1}^T D_t\right|>\varepsilon V_{T,a},
\;V_{T,a}>M
\right).
\end{align*}
On the event $\{V_{T,a}>M\}$, the condition
$|\sum_{t=1}^T D_t|>\varepsilon V_{T,a}$ implies
$|\sum_{t=1}^T D_t|>\varepsilon M$. Also the predictable quadratic variation
satisfies
\[
\sum_{t=1}^T
\Var(D_t\mid \mathcal F_{t-1})
\le V_{T,a}.
\]

\noindent Now consider the following lemma~\citep{tropp2011freedman}.

\begin{lemma}[Freedman's inequality]
Let $(Y_t,\mathcal F_t)_{t\ge 1}$ be a martingale difference sequence with
$Y_t\le b$ almost surely for all $t$. Define
\[
S_T:=\sum_{t=1}^T Y_t,
\qquad
W_T:=\sum_{t=1}^T \E[Y_t^2\mid \mathcal F_{t-1}].
\]
Then, for any $x>0$ and $v>0$,
\[
\Prob\left(S_T\ge x,\; W_T\le v\right)
\le
\exp\left\{
-\frac{x^2}{2(v+bx/3)}
\right\}.
\]
Consequently, if $|Y_t|\le b$ almost surely, then
\[
\Prob\left(|S_T|\ge x,\; W_T\le v\right)
\le
2\exp\left\{
-\frac{x^2}{2(v+bx/3)}
\right\}.
\]
\end{lemma}

\noindent We apply Freedman's inequality to
\[
S_T:=\sum_{t=1}^T D_t
=
n_{a,T}-V_{T,a}.
\]
Here $|D_t|\le 1$ and
\[
W_T:=\sum_{t=1}^T \E[D_t^2\mid\mathcal F_{t-1}]
=
\sum_{t=1}^T p_{t,a}(1-p_{t,a})
\le V_{T,a}.
\]

\noindent Fix $\varepsilon>0$ and $M>0$. Then
\begin{align*}
\Prob\left(\left|\frac{n_{a,T}}{V_{T,a}}-1\right|>\varepsilon\right)
&\le
\Prob(V_{T,a}\le M)
+
\Prob\left(|S_T|>\varepsilon V_{T,a},\; V_{T,a}>M\right).
\end{align*}
For the second term, decompose according to dyadic intervals:
\[
\{V_{T,a}>M\}
=
\bigcup_{j=0}^{\infty}
\left\{2^jM<V_{T,a}\le 2^{j+1}M\right\}.
\]
On the event $\{2^jM<V_{T,a}\le 2^{j+1}M\}$,
\[
|S_T|>\varepsilon V_{T,a}
\quad\Longrightarrow\quad
|S_T|>\varepsilon 2^jM,
\]
and also
\[
W_T\le V_{T,a}\le 2^{j+1}M.
\]
Therefore, by Freedman's inequality with
\[
x=\varepsilon 2^jM,
\qquad
v=2^{j+1}M,
\qquad
b=1,
\]
we obtain
\begin{align*}
\Prob\left(
|S_T|>\varepsilon V_{T,a},\;
2^jM<V_{T,a}\le 2^{j+1}M
\right)
&\le
2\exp\left\{
-\frac{\varepsilon^2 2^{2j}M^2}
{2\left(2^{j+1}M+\varepsilon 2^jM/3\right)}
\right\}\\[8pt]
&\qquad=
2\exp\left\{
-\frac{\varepsilon^2 2^jM}
{2(2+\varepsilon/3)}
\right\}.
\end{align*}
Hence
\[
\Prob\left(|S_T|>\varepsilon V_{T,a},\; V_{T,a}>M\right)
\le
2\sum_{j=0}^{\infty}
\exp\left\{
-\frac{\varepsilon^2 2^jM}
{2(2+\varepsilon/3)}
\right\}.
\]
The right-hand side tends to $0$ as $M\to\infty$. Since
$V_{T,a}\xrightarrow{\prob}\infty$, we also have
\[
\Prob(V_{T,a}\le M)\to 0
\]
for every fixed $M$. Therefore, taking first $T\to\infty$ and then
$M\to\infty$ yields
\[
\Prob\left(\left|\frac{n_{a,T}}{V_{T,a}}-1\right|>\varepsilon\right)
\to 0.
\]
Thus
\[
\frac{n_{a,T}}{V_{T,a}}\xrightarrow{\prob}1.
\]

\subsubsection*{Extension to joint CLT}

Assume the single-arm result holds for every $a\in[K]$: that is,
Since $\hat{\sigma} \inprob 1$, Slutsky's lemma reduces the claim to
$\sigma^{-1}v_T\indist\N(0,I_K)$, where
\[
  v_T
  :=
  \bigl(
    \sqrt{n_{1,T}}\,(\widehat\mu_{1,T}-\mu_1),\;
    \dots,\;
    \sqrt{n_{K,T}}\,(\widehat\mu_{K,T}-\mu_K)
  \bigr)^\top.
\]
By the Cram\'er--Wold device it suffices to show, for every fixed
$\lambda\in\R^K$,
\[
  \frac{\lambda^\top v_T}{\|\lambda\|}\indist\N(0,1).
\]

\noindent Recall that $X_{t,a}:=\mathbf{1}\{A_t=a\}(l_t-\mu_a)$.
By exchanging the order of summation we have,
\[
  \lambda^\top v_T
  =
  \sum_{a=1}^K\frac{\lambda_a}{\sqrt{n_{a,T}}}
  \sum_{t=1}^T X_{t,a}
\]

\noindent Let us define $ Y_t:=\sum_{a=1}^K\frac{\lambda_a}{\sqrt{T p^\star_{T,a}}}\,X_{t,a}$.
Since $\E[X_{t,a}\mid\F_{t-1}]=0$ for each $a$ (single-arm lemma),
$\{Y_t,\F_t\}$ is itself a martingale difference sequence.
Furthermore, because $A_t$ selects exactly one arm at each time, the indicators
$\mathbf{1}\{A_t=a\}$ and $\mathbf{1}\{A_t=b\}$ cannot both equal $1$,
so $X_{t,a}X_{t,b}=0$ almost surely for $a\ne b$. The cross-terms in
$Y_t^2$ therefore vanish conditionally:
\begin{align*}
  \E[Y_t^2\mid\F_{t-1}]
  &=
  \sum_{a=1}^K\frac{\lambda_a^2}{T p^\star_{T,a}}
  \,\E[X_{t,a}^2\mid\F_{t-1}]
  +
  \underbrace{
    \sum_{a\ne b}\frac{\lambda_a\lambda_b}{\sqrt{T p^\star_{T,a}T p^\star_{T,a}}}
    \E[X_{t,a}X_{t,b}\mid\F_{t-1}]
  }_{=\,0}\\[4pt]
  &=
  \sum_{a=1}^K\frac{\lambda_a^2\,p_{t,a}}{T p^\star_{T,a}},
\end{align*}
where we used $\E[X_{t,a}^2\mid\F_{t-1}]=p_{t,a}$.
Summing over $t$,
\[
  V_T
  :=\sum_{t=1}^T\E[Y_t^2\mid\F_{t-1}]
  =\sum_{a=1}^K\frac{\lambda_a^2 V_{T,a}}{T p^\star_{T,a}},
  \qquad V_{T,a}:=\textstyle\sum_{t=1}^T p_{t,a}.
\]

\noindent In the previous section we have established that $n_{a,T}-V_{T,a}=O_p(\sqrt{V_{T,a}})$.
Dividing both sides by $V_{T,a}$ and using $V_{T,a}\ge T\varepsilon\to\infty$
in probability,
\[
  \frac{n_{a,T}}{V_{T,a}}
  =
  1+\frac{n_{a,T}-V_{T,a}}{V_{T,a}}
  =
  1+O_p\!\left(\frac{1}{\sqrt{V_{T,a}}}\right)
  \inprob 1.
\]

\noindent Hence,
\[
  \frac{V_T}{\|\lambda\|^2}
  =\sum_{a=1}^K\frac{\lambda_a^2}{\|\lambda\|^2}
  \cdot\frac{V_{T,a}}{T p^\star_{T,a}}
  \inprob
  \sum_{a=1}^K\frac{\lambda_a^2}{\|\lambda\|^2}=1,
\]
so $V_T\inprob\|\lambda\|^2$.
 
\noindent Furthermore, analogous arguments as the last section imply that the Lindeberg condition is also satisfied. Therefore, we have established that
\[
\widetilde v_T
:=
\left(
\frac{1}{\sqrt{T p_{T,1}^\star}}
\sum_{t=1}^T X_{t,1},
\ldots,
\frac{1}{\sqrt{T p_{T,K}^\star}}
\sum_{t=1}^T X_{t,K}
\right)^\top
\xrightarrow{d}
\mathcal N(0,I_K).
\]

\noindent It remains to replace the deterministic normalizers
$\sqrt{T p_{T,a}^\star}$ by the realized sample sizes
$\sqrt{n_{a,T}}$. Recall that
\[
\frac{n_{a,T}}{V_{T,a}}
\xrightarrow{\prob}
1,
\qquad
\frac{V_{T,a}}{T p_{T,a}^\star}
=
\frac{\bar p_{T,a}}{p_{T,a}^\star}
\xrightarrow{\prob}
1.
\]
Hence, by Slutsky's lemma,
\[
\frac{n_{a,T}}{T p_{T,a}^\star}
=
\frac{n_{a,T}}{V_{T,a}}
\cdot
\frac{V_{T,a}}{T p_{T,a}^\star}
\xrightarrow{\prob}
1.
\]
Applying the continuous mapping theorem yields
\[
\sqrt{\frac{T p_{T,a}^\star}{n_{a,T}}}
\xrightarrow{\prob}
1,
\qquad
a\in[K].
\]

\noindent Define the diagonal matrix
\[
D_T
:=
\operatorname{diag}
\left(
\sqrt{\frac{T p_{T,1}^\star}{n_{1,T}}},
\ldots,
\sqrt{\frac{T p_{T,K}^\star}{n_{K,T}}}
\right).
\]
Since each diagonal entry converges in probability to one, we have
\[
D_T
\xrightarrow{\prob}
I_K.
\]

\noindent Observe that
\begin{align*}
D_T \widetilde v_T
&=
\left(
\sqrt{\frac{T p_{T,1}^\star}{n_{1,T}}}
\cdot
\frac{1}{\sqrt{T p_{T,1}^\star}}
\sum_{t=1}^T X_{t,1},
\ldots,
\sqrt{\frac{T p_{T,K}^\star}{n_{K,T}}}
\cdot
\frac{1}{\sqrt{T p_{T,K}^\star}}
\sum_{t=1}^T X_{t,K}
\right)^\top \\
&=
\left(
\frac{1}{\sqrt{n_{1,T}}}
\sum_{t=1}^T X_{t,1},
\ldots,
\frac{1}{\sqrt{n_{K,T}}}
\sum_{t=1}^T X_{t,K}
\right)^\top \\
&=: v_T .
\end{align*}

\noindent Since
\[
\widetilde v_T
\xrightarrow{d}
\mathcal N(0,I_K)
\qquad\text{and}\qquad
D_T
\xrightarrow{\prob}
I_K,
\]
the multivariate Slutsky theorem implies
\[
v_T
=
D_T\widetilde v_T
\xrightarrow{d}
\mathcal N(0,I_K).
\]

\noindent Therefore,
\[
\left(
\sqrt{n_{1,T}}(\widehat\mu_{1,T}-\mu_1),
\ldots,
\sqrt{n_{K,T}}(\widehat\mu_{K,T}-\mu_K)
\right)^\top
\xrightarrow{d}
\mathcal N(0,I_K),
\]
which completes the proof.

\subsection{Proof of Lemma~\ref{lemma:instability}}
\label{sec:Proof-of-Lemma-instability}

We call a bandit algorithm \emph{symmetric} if the trajectory of both the arms are identical in distribution when reward distribution for all the arms are identical. It follows that for any symmetric bandit algorithm, 
\begin{align*}
    \Exs[n_{1,T}] =  \Exs[n_{2,T}] = \frac{T}{2}.
\end{align*}

\noindent In this section, we prove the following lemma:
\begin{lemma*}
    Any randomized, symmetric bandit algorithm with regret upper bounded by $c_0\sqrt{T}$ cannot satisfy condition~\eqref{eq:L1-conv} in Bernoulli environment. Here $c_0$ is a constant independent of $T$. 
\end{lemma*}

\noindent The proof of this lemma follows the same structural argument as Theorem~2 of \cite{chen2026bandit}, 
extended here to encompass the general class of \emph{stochastic} algorithms, and allows for bounded rewards which are essential for mirror-descent style algorithm like EXP3. To accommodate 
randomized decision rules, we augment the probability space to carry an i.i.d.\ sequence 
$\{U_t\}_{t=1}^{T}$ of $\mathrm{Uniform}[0,1]$ random variables, independent of all arm rewards, 
serving as the internal randomization seed of the algorithm. Now for each arm $a$ define $\textbf{X}_{a,T} := (X_{a,1},\ldots, X_{a,T})$ and let the vector of all potential outcomes be $\textbf{X}_{T}:= (\textbf{X}_{1,T},\textbf{X}_{2,T})$.  At each round $t$, the arm selection 
takes the form
\[
A_t = \pi_t(U_1, A_1, X_1, \ldots, A_{t-1}, X_{t-1}, U_t),
\]
where $\pi_t$ is a measurable decision rule and $X_s$ denotes the reward observed at round $s$. 
The cumulative allocation to arm $1$ is then the random variable
\[
n_{1,T} \coloneqq \sum_{t=1}^{T} \mathbf{1}\{A_t = 1\},
\]
whose stochasticity arises jointly from the reward observations and the internal randomization 
$\{U_t\}_{t=1}^{T}$. Now instantiate the environment with Bernoulli rewards where the mean reward for the sub-optimal arm and optimal arm are $1/2$ and $1/2 + \Delta_0$ respectively, where the choice of $\Delta_0$ will be made later.
We claim that there exists a constant $c_0 > 0$ such that, for all 
sufficiently large $T$,
\begin{equation}\label{eq:var-lower}
    \sqrt{\mathrm{Var}_{\Delta_0}(n_{1,T})} > c_0\, T.
\end{equation}

\noindent Since the bandit algorithm being considered is minimax optimal, it follows that $\regret_T(\frac{3c_1}{\sqrt{T}}) \leq c_1\sqrt{T}$. Let $g(\Delta)$ denote the expected number of arm pulls of the sub-optimal arm  $\Exs_{\Delta}[n_{1,T}]$.  Then from regret decomposition we know that,
\begin{align*}
    \regret_T \left(\frac{3c_1}{\sqrt{T}} \right) = \frac{3c_1}{\sqrt{T}} g \left(\frac{3c_1}{\sqrt{T}} \right) \leq c_1\sqrt{T}.
\end{align*}
\noindent The above inequality further implies that  $g(3c_1/\sqrt{T}) \leq T/3$. Define $\Delta_0 = 3c_1/\sqrt{T}$ observe that, $g(0) - g(\Delta_0) \geq T/6$ Consider the following string of inequalities
\begin{align}\label{eqn:insta-low-bddo}
\notag
    \sqrt{\Var_{\Delta_0}[n_{1,T}]} 
    = \sqrt{\Exs_{\Delta_0} \left[\left(n_{1,T} - g(\Delta_0) \right)^2 \right]}
    &\geq \sqrt{\Exs_{\Delta_0} \left[\left(n_{1,T} - g(\Delta_0) \right)^2 \  \mathbf 1
\left\{
n_{1,T} > \frac{g(0) + g(\Delta_0)}{2}
\right\}  \right]} \\[8pt] \notag
& \geq \frac{T}{12} \sqrt{ \prob_{\Delta_0}
\left(
n_{1,T} > \frac{g(0) + g(\Delta_0)}{2}
\right)} \\[8pt]
&\geq \frac{T}{12} \sqrt{ \prob_{\Delta_0}
\left(
n_{1,T} > \frac{5T}{12}
\right)}
\end{align}

\noindent As the event $\left\{n_{1,T} > \frac{5T}{12} \right\}$ is measurable with respect to the sigma field jointly generated by $\textbf{X}_T,\textbf{U}_T$ we have,
\begin{align*}
  \prob_{\Delta_0}
\left(
n_{1,T} > \frac{5T}{12}
\right)
= \Exs_{\Delta_0} \left[  \mathbf 1
\left\{
n_{1,T} > \frac{5T}{12}
\right\} \right]
= \Exs \left[  \mathbf 1
\left\{
n_{1,T} > \frac{5T}{12}
\right\} \frac{f^{(1)}_{\Delta_0}(\textbf{X}_T) \cdot f^{(2)}(\textbf{U}_T)}{f^{(1)}_{0}(\textbf{X}_T) \cdot f^{(2)}(\textbf{U}_T)}\right]
\end{align*}

\noindent As we are considering Bernoulli rewards, we have
\begin{align*}
    \frac{f^{(1)}_{\Delta_0}(\textbf{X}_T)}{f^{(1)}_{0}(\textbf{X}_T)}
    =  \frac{\prod^T_{t=1}(\frac{1}{2} + \Delta_0)^{X_{1,t}}(\frac{1}{2} - \Delta_0)^{1-X_{1,t}}}{\left(\frac{1}{2} \right)^{T}}
    = \left( 1 + \frac{6c_1}{\sqrt{T}} \right)^{\sum^T_{i=1} X_{1,i}} \left( 1 -\frac{6c_1}{\sqrt{T}} \right)^{(T -\sum^T_{i=1} X_{1,i})}
\end{align*}

\noindent Further simplification yields:
\begin{align}
\notag
  \prob_{\Delta_0}
\left(
n_{1,T} > \frac{5T}{12}
\right)
&=  \Exs \left[  \mathbf 1
\left\{
n_{1,T} > \frac{5T}{12}
\right\} \left( 1 + \frac{6c_1}{\sqrt{T}} \right)^{\sum^T_{i=1} X_{1,i}} \left( 1 -\frac{6c_1}{\sqrt{T}} \right)^{(T -\sum^T_{i=1} X_{1,i})}   \right] \\[8pt] \notag
&\geq \Exs \left[  \mathbf 1
\left\{
n_{1,T} > \frac{5T}{12}
\right\} \mathbf 1
\left\{
\sum^T_{i=1} X_{1,i} \leq T - \sqrt{T}
\right\}
\left( 1 + \frac{6c_1}{\sqrt{T}} \right)^{\sum^T_{i=1} X_{1,i}} \left( 1 -\frac{6c_1}{\sqrt{T}} \right)^{(T -\sum^T_{i=1} X_{1,i})}  \right] \\[8pt] \notag
& \geq \Exs \left[  \mathbf 1
\left\{
n_{1,T} > \frac{5T}{12}
\right\} \mathbf 1
\left\{
\sum^T_{i=1} X_{1,i} \leq T - \sqrt{T}
\right\}
   \right]  \left( 1 - \frac{6c_1}{\sqrt{T}} \right)^{\sqrt{T}} \\[8pt]
  &=  \prob \left(  
n_{1,T} > \frac{5T}{12}
, \ 
\sum^T_{i=1} X_{1,i} \leq T- \sqrt{T} 
   \right)  \left( 1 -\frac{6c_1}{\sqrt{T}} \right)^{\sqrt{T}}
\end{align}

\noindent It follows that,
\begin{align*}
    \liminf_{T \to \infty} \prob_{\Delta_0}
\left(
n_{1,T} > \frac{5T}{12}
\right)
\geq  \liminf_{T \to \infty} \prob \left(  
n_{1,T} > \frac{5T}{12}
, \ 
\sum^T_{i=1} X_{1,i} \leq T 
   \right) \times e^{-6c_1}
\end{align*}

\noindent Define events  $A_T := \left\{ n_{1,T} > \frac{5T}{7}  \right\}$ and  $\left\{\sum^T_{i=1} X_{1,i} \leq T -\sqrt{T} \right\}$. By an application of the central limit theorem, $\prob(B_T) \to 1$.  Now observe that,
\begin{align*}
    \frac{T}{2} = \Exs[n_{1,T}] 
    &= \Exs \left[n_{1,T} \ \mathbf 1
\left\{
n_{1,T} > \frac{5T}{7}
\right\} +  n_{1,T} \ \mathbf 1
\left\{
n_{1,T} \leq \frac{5T}{12}
\right\} \right] \\[8pt]
& \leq \frac{5T}{7}\prob
\left(
n_{1,T} \leq \frac{5T}{12}
\right) + T \  \prob
\left(
n_{1,T} > \frac{5T}{12}
\right) \\[8pt]
& = \frac{5T}{12} + \frac{7T}{12} \prob
\left(
n_{1,T} > \frac{5T}{12}
\right)
\end{align*}

\noindent By rearranging the terms in the above inequality we have
\begin{align}\label{eqn:insta-low-bddto}
   \prob
\left(
A_T
\right)  \geq \frac{1}{7}
\end{align}

\noindent Therefore,
\begin{align*}
   \liminf_{T \to \infty} \prob(A_T \cap B_T) = \liminf_{T \to \infty} [\prob(A_T) + \prob( B_T) - \prob(A_T \cup B_T)]
    \geq \liminf_{T \to \infty}[\prob(A_T) + \prob( B_T) - 1] \geq \frac{1}{7}
\end{align*}

\noindent Hence by combining equations~\eqref{eqn:insta-low-bddo} and~\eqref{eqn:insta-low-bddto} we have
\begin{align}\label{eqn:insta-low-bdte}
     \sqrt{\Var_{\Delta_0}[n_{1,T}]} \geq c_0 T \quad \text{where} \ c_0 = \frac{1}{12 \sqrt{7}}
\end{align}

\noindent Now we claim that Algorithm~\ref{alg:st-exp3} does not satisfy Lai and Wei's stability condition, which implies that~\ref{thm:berry}~\ref{thm:berry-i} is also not satisfied, as it is a stronger condition. If possible, suppose the algorithm satisfies sampling stability. There must exist deterministic scalars $n_T \to \infty$ such that
\[
\frac{n_{1,T}}{n^\star_{1,T}} \xrightarrow{p} 1.
\]

\noindent Now fix any $\gamma \in (0,1)$. By convergence in probability, for all large $T$,
\[
\prob_{\Delta_0} \left(|n_{1,T} - n^\star_{1,T}| \geq \gamma n^\star_{1,T}\right) \leq \gamma.
\]
Using the bound $0 \leq n_{1,T} \leq T$ and $n^\star_{1,T} \leq 2T$, we obtain
\[
\Exs_{\Delta_0} \!\left[(n_{i,T} - n^\star_{1,T})^2\right]
\leq (\gamma n^\star_{1,T})^2 + (T + n^\star_{1,T})^2\, \prob_{\Delta_0} \left(|n_{1,T} - n^\star_{1,T}| \geq \gamma \ n^\star_{1,T}  \right)
\leq \left(4\gamma^2 + 9\gamma \right) T^2.
\]
Application of Jensen's inequality implies that $\left(\Exs_{\Delta_0}[n_{1,T}] - n^\star_{1,T} \right)^2 \leq \Exs_{\Delta_0}[(n_{1,T} - n^\star_{1,T})^2]$, so
\[
\mathrm{Var}(n_{1,T})
\leq \Exs_{\Delta_0}\!\left[(n_{1,T} - n^\star_{1,T})^2\right] + \left(\Exs_{\Delta_0}[n_{1,T}] - n^\star_{1,T}\right)^2
\leq 2\,\Exs_{\Delta_0} \!\left[(n_{1,T} - n^\star_{1,T})^2\right]
\leq 2\!\left(4\gamma^2 + 9\gamma\right) T^2.
\]
Now, since $\gamma > 0$ is arbitrary, this implies $\mathrm{Var}_{\Delta_0}[n_{1,T}]/ T^2 \to 0$, which contradicts equation~\eqref{eqn:insta-low-bdte}.

\subsection*{Proof of Lemma~\ref{lemma:actual regret bd}}
\label{proof:regret-main}

    \noindent Define $p^\star$ to be the optimal policy in the unregularized setting, that is, $p^\star_j:=\frac{1}{|I|}\I\{j\in I\}$. Note that $\mu^\star=\langle \mu, p^\star\rangle$. For all comparator $y\in\DE$ and $T\ge K$ we claim that
    \begin{align}
        \label{eq:almost regret bd}
        \E\left[\langle \mu,\xT-y\rangle\right]\le \frac{C_\alpha(K,T)}{\sqrt{T}}+\frac{K}{\sqrt{T}}+\frac{2K \gamma_T^2 \log^2 T }{\sqrt{T}}
    \end{align}

   \noindent  Suppose claim~\eqref{eq:almost regret bd} stated above is correct.  Define $p_{\vars}$ to be a $K$ dimensional vector in the probability simplex such that it is equal to $\vars$, for $K - |I|$ coordinates, and otherwise equals $\frac{1-(K-|I|)\vars}{|I|}$ for renormalization purpose. Then observe that,
   \begin{align*}
       ||p^\star-p_{\vars}||_1=\vars\cdot(K-|I|)+\left(\frac{1}{|I|}-\frac{1-(K-|I|)\vars}{|I|}\right)\cdot|I|=2(K-|I|)\vars \le 2K\vars.
   \end{align*}
   Plugging $y=p_{\vars} \in\DE$ in claim~\eqref{eq:almost regret bd}, we have
    \begin{align*}
        \E[\langle\mu,\xT-p^\star\rangle]&=\E[\langle\mu,\xT-p_{\vars}\rangle]+\E[\langle\mu,p_{\vars}-p^\star\rangle]\\
        &\le\E[\langle\mu,\xT-p_{\vars}\rangle]+||\mu||_\infty\cdot||p^\star-p_{\vars}||_1\\
        &\le\E[\langle\mu,\xT-p_{\vars} \rangle]+\frac{2K \hTto}{\sqrt{T}}\\
        &\le\frac{C_\alpha(K,T)+K+2K \gamma_T^2 \log^2 T+2K\hTto}{\sqrt{T}}
    \end{align*}

    \noindent For $\alpha\in\left[0,\frac{1}{3}\right)$, observe that $C_\alpha(K,T)+K\le 4K\log T$. On the other hand, for $\alpha\in\left[\frac{1}{3},1\right]$, we have $C_\alpha(K,T)+K\le4K\log K$. Thus we have,
    \begin{align*}
    \regret(T)\le \begin{cases}
        \left(4\log T+2 \gamma_T^2\log^2 T + 2\hTto \right)\cdot K\sqrt{T} ,&\alpha\in\left[0,\frac{1}{3}\right)\\
        \\
        \left(4\log K+2\gamma^2_T\log^2 T +2 \hTto \right)\cdot K\sqrt{T} ,&\alpha\in\left[\frac{1}{3},1\right]
    \end{cases}
\end{align*}
where,
\begin{align*}
    C_\alpha(K, \vars) = \begin{cases}
3K \log(\frac{1}{\vars}), & \text{if } \alpha \in \left[ 0, \frac{1}{3}\right), \\[6pt]
3 K \log K, & \text{if } \alpha \in \left[ \frac{1}{3} , 1\right].
\end{cases}
\end{align*}
    provided claim~\eqref{eq:almost regret bd} holds. To justify it, note that
    \begin{align*}
       R(y)-R(\xT)= \sum_{i=1}^K\left[-\log(y_i)+\log(\overline{p}_{T,i})\right] \le-\sum_{i=1}^K\log(y_i)\le K\log(1/\vars) 
    \end{align*}
     
\noindent Thus, from equation~\eqref{eqn:ffunc-upp-bdd} it follows
    \begin{align}
        \label{eq:regret general y upper bd}
        \E[\langle\mu,\xT-y\rangle]&\le \frac{C_\alpha(K,\vars)}{\setpz T}+\setpz K+\frac{\setpz \pen^2}{\vars^2}+\pen K\log\left(\frac{1}{\vars}\right)
    \end{align}
    Plugging the aforementioned choices of $\setpz,\pen$ and $ \vars$, we obtain
    \begin{align*}
        &\E[\langle\mu,\xT-y\rangle]\\
        \le&\; \frac{C_\alpha(K,T)}{\sqrt{T}}+\frac{K}{\sqrt{T}}+\frac{1}{\sqrt{T}}\cdot\frac{\gamma_T^2 \log^2 T}{\hTto^2}+\frac{K \gamma_T \log^2 T }{\sqrt{T}}\\
        &\leq \frac{C_\alpha(K,T)}{\sqrt{T}}+\frac{K}{\sqrt{T}}+\frac{2K \gamma_T^2 \log^2 T }{\sqrt{T}}
    \end{align*}
   \noindent Hence, our proof is complete.

\subsection{Proof of Lemma~\ref{lemma:ZT-concet}}
    First observe that $g$ is non-negative with $1$ being the only root. Secondly, for $z\in[1/2,2]$, we have $g(z)\ge(z-1)^2/4$. Now, fix $\delta>0$.
    \begin{align*}
        \prob\left(|Z_T-1|>\delta\right)&=\prob\left(|Z_T-1|>\delta,Z_T\in\left[\frac{1}{2},2\right]\right)+\prob\left(|Z_T-1|>\delta,Z_T\notin\left[\frac{1}{2},2\right]\right)\\
        &\le\prob\left((Z_T-1)^2>\delta^2,Z_T\in\left[\frac{1}{2},2\right]\right)+\prob\left(Z_T\notin\left[\frac{1}{2},2\right]\right)\\
        &\le\prob\left(g(Z_T)\ge\frac{(Z_T-1)^2}{4}>\frac{\delta^2}{4},Z_T\in\left[\frac{1}{2},2\right]\right)+\prob\left(g(Z_T)\ge\min\left\{g\left(\frac{1}{2}\right),g(2)\right\}\right)\\
        &\le\prob\left(g(Z_T)>\frac{\delta^2}{4}\right)+\prob\left(g(Z_T)>\frac{1}{10}\right)\\
        &\le\E\left[g(Z_T)\right]\cdot\left(\frac{4}{\delta^2}+10\right)
    \end{align*}
    
\noindent Thus, the lemma holds.

\subsection{Proof of Lemma~\ref{lemma:IS-bound}}

We begin the proof by recalling that the standard log-barrier regularizer is $\psi(p) := -\sum_a\ln p_a$. Observe that we can rewrite $f_{\lmep}$ as $f_{\lmep}(x)=\pen \psi(x)+\langle\alpha,x\rangle$ for some $\alpha\in\R^K$, since $\frac{1}{\vars} \sum_a\rho_a$ is constant over the simplex. Now,
    \begin{align*}
        f_{\lmep}(x)-f_{\lmep}(y)&=\pen\left(\psi(x)-\psi(y)\right)+\langle\alpha,x-y\rangle\\
        &=\pen\left[D_\psi(x,y) + \langle\nabla\psi(y),x-y\rangle\right] + \langle\alpha,x-y\rangle\\
        &\overset{(*)}{=}\pen\; IS(x,y)+\langle\alpha+\pen\nabla\psi(y),x-y\rangle\\
        &=\pen\; IS(x,y)+\langle\nabla f_{\lmep}(y),x-y\rangle
    \end{align*}
   Equality $(*)$ holds since $\nabla \left[\psi(y)\right]_j=-\frac{1}{y_j}$ implies that,
    \begin{align*}
        D_{\psi}(x,y)&=\psi(x)-\psi(y)-\langle\nabla\psi(y),x-y\rangle\\
        &=\sum_{j=1}^K\left[-\ln(x_j)+\ln(y_j)+\frac{x_j-y_j}{y_j}\right]\\
        &=\sum_{j=1}^K\left[\frac{x_j}{y_j}-\ln\left(\frac{x_j}{y_j}\right)-1\right]\\
        &=IS(x,y)
    \end{align*}
    
\noindent Now by plugging $y=x_{\lmep}$, it follows using \textit{first-order optimality condition} that for $x\in\DE$ and $\pen>0$,
    \begin{align}\label{eqn:f-to-IS}
        f_{\lmep}(x)-f_{\lmep}(p^\star_{\lmep})\ge\pen\; IS(x,p^\star_{\lmep})
    \end{align}

\noindent Therefore, it suffices to prove the upper bound on $\Exs[ f_{\lmep}(x)-f_{\lmep}(p^\star_{\lmep})].$ Consider the following lemma.
\begin{lemma}
    \label{lem:master eqn}
     For any $y\in\DE$,
    \begin{align*}
        \E\left[f_{\lmep}(\xT)-f_{\lmep}(y)\right]&\le \E\left[\frac{D_{\phi_\alpha}(y,x_1)}{\setpz T}+\frac{\setpz}{T}\sum_{t=1}^T \left|\left|\widehat{\ell}_t\right|\right|_{x_t,\alpha}^2+\frac{\setpz\pen^2}{T}\sum_{t=1}^T \left|\left|\nabla R(x_t)\right|\right|_{x_t,\alpha}^2\right]
    \end{align*}
    where $||v||_{x_t,\alpha}^2:=\sum_{i=1}^K v_i^2 x_{t,i}^{(2-\alpha)}$ and $\alpha\in[0,1]$.
\end{lemma}

\noindent We make the following claim:
    \begin{align}
        \label{claim:bregamn-bd}
        D_{\phi_\alpha}\left(y,x_1\right)\le
        \begin{cases}
            K\log\left(\frac{1}{\varepsilon}\right),&\alpha=0\\
            3K\log\left(\frac{1}{\varepsilon}\right),&0<\alpha<\frac{1}{3}\\
            3K\log\left(K\right),&\frac{1}{3}\le\alpha<1\\
            \log\left(K\right),&\alpha=1
        \end{cases}
    \end{align}

\noindent Assume the claim is correct.  Next,
    \begin{align}\label{eqn:lnorm-ctrl}
        \E\left[\left|\left| \widehat{\ell}_t\right|\right|_{x_t,\alpha}^2\right]\le\E\left[\sum_{i=1}^K \widehat{\ell}^{\; 2}_{t,i}\;x_{t,i}\right]= \E\left[\sum_{i=1}^K \frac{\I\{A_t=i\}\ell_t^2}{x_{t,i}}\right] \le \E\left[\sum_{i=1}^K \frac{\I\{A_t=i\}}{x_{t,i}}\right]=K
    \end{align}
    \\
    We bound the last term as follows:
    \begin{align}\label{eqn:Rgrad-ctrl}
        \E\left[\left|\left|\nabla R(x_t)\right|\right|_{x_t,\alpha}^2\right]&\le\E\left[\sum_{i=1}^K\left(-\frac{1}{x_{t,i}}+\frac{1}{\varepsilon}\right)^2 x_{t,i}\right]\le \frac{1}{\varepsilon^2} 
    \end{align}
    Combining equations~\eqref{claim:bregamn-bd},~\eqref{eqn:lnorm-ctrl} and~\eqref{eqn:Rgrad-ctrl}, we have 
\begin{align}\label{eqn:ffunc-upp-bdd}
     \E\left[f_{\lmep}(\xT)-f_{\lmep}(y)\right]&\le  \frac{C_\alpha(K,\varepsilon)}{\setpz T}  +  \setpz K  +  \frac{\setpz^2 \pen}{\vars^2}
\end{align}

\noindent We conclude from equation~\eqref{eqn:f-to-IS} that once claim~\eqref{claim:bregamn-bd} is proved, our proposed result holds.     

\subsubsection*{Proof of Claim~\eqref{claim:bregamn-bd}}
    
\noindent We begin the proof by recalling from equation~\eqref{defn:mirror-map-alpha} that,
\begin{align}
  \phi_\alpha(p) :=
  \begin{cases}
    -\sum_{i=1}^K
    \dfrac{p_i^{\alpha} - \alpha p_i - (1-\alpha)}
    {\alpha(1-\alpha)},
    & 0 < \alpha < 1,\\[6pt]
    -\sum_{i=1}^K (\log p_i - p_i + 1),
    & \alpha = 0,\\[6pt]
    \;\sum_{i=1}^K (p_i\log p_i - p_i + 1),
    & \alpha = 1.
  \end{cases}
\end{align}

\noindent Since $x_1 = (1/K, \ldots, 1/K)'$, we compute each term in the Bregman divergence
$D_{\phi_\alpha}(y,x_1) = \phi_\alpha(y) - \phi_\alpha(x_1) - \langle \nabla\phi_\alpha(x_1), y - x_1\rangle$ separately. We shall bound  $D_{\phi_\alpha}(y,x_1)$ 

\noindent Let us begin with the case $\alpha = 0$.
Since $\nabla\phi_0(p) = -(p_i^{-1} - 1)_i$, evaluating at $x_1$ gives 
$\nabla\phi_0(x_1) = -(K-1)\mathbf{1}$, which is a constant vector. Hence the 
gradient term vanishes:
\begin{align*}
    \langle\nabla\phi_0(x_1), y - x_1\rangle 
    = -(K-1)\sum_{i=1}^K\left(y_i - \frac{1}{K}\right) = 0.
\end{align*}
The difference of potentials gives,
\begin{align*}
    \phi_0(y) - \phi_0(x_1) 
    &= -\sum_{i=1}^K\left[(\log y_i - y_i + 1) - (\log K^{-1} - K^{-1} + 1)\right]\\
    &= -\sum_{i=1}^K\left[\log y_i + \log K - \left(y_i - \frac{1}{K}\right)\right]
    = -\sum_{i=1}^K[\log y_i + \log K],
\end{align*}
where the last step uses $\sum_i y_i = \sum_i K^{-1} = 1$. Since $y_i \geq \vars$ 
for all $i$, we have $-\log y_i \leq \log(1/\vars)$ for each term, and hence
\begin{align}\label{eqn:d-phizo}
    D_{\phi_0}(y,x_1) = -\sum_{i=1}^K[\log y_i + \log K] \leq K\log\left(\frac{1}{\vars}\right). \qquad
\end{align}

\noindent Since $\nabla\phi_1(p) = (\log p_i)_i$, evaluating at $x_1$ gives 
$\nabla\phi_1(x_1) = -\log(K)\mathbf{1}$, which is a constant vector. Hence the 
gradient term vanishes:
\begin{align*}
    \langle\nabla\phi_1(x_1), y - x_1\rangle 
    = -\log(K)\sum_{i=1}^K\left(y_i - \frac{1}{K}\right) = 0.
\end{align*}
The difference of potentials gives,
\begin{align*}
    \phi_1(y) - \phi_1(x_1) 
    &= \sum_{i=1}^K\left[(y_i\log y_i - y_i + 1) - (K^{-1}\log K^{-1} - K^{-1} + 1)\right]\\
    &= \sum_{i=1}^K\left[y_i\log y_i + K^{-1}\log K - \left(y_i - K^{-1}\right)\right]
    = \sum_{i=1}^K y_i\log y_i + \log K,
\end{align*}
where the last step uses $\sum_i y_i = \sum_i K^{-1} = 1$. Adding the gradient 
term (which vanishes) and noting $x_{1,i} = K^{-1}$:
\begin{align*}
    D_{\phi_1}(y,x_1) 
    = \sum_{i=1}^K y_i\log y_i + \log K 
    = \sum_{i=1}^K y_i\log\left(\frac{y_i}{x_{1,i}}\right).
\end{align*}
Since $\sum_i y_i\log y_i \leq 0$ (as $y_i \in [0,1]$ implies $\log y_i \leq 0$), 
we conclude
\begin{align}\label{eqn:d-phio}
    D_{\phi_1}(y,x_1) = \sum_{i=1}^K y_i\log\left(\frac{y_i}{x_{1,i}}\right) 
    \leq \log K. 
\end{align}

\noindent For $\alpha \in (0,1)$, the $i$-th partial derivative is
\begin{align*}
    \frac{\partial \phi_\alpha}{\partial p_i}(p) = -\frac{p_i^{\alpha-1} - 1}{1-\alpha}.
\end{align*}
Evaluating at $x_1$ gives $\nabla\phi_\alpha(x_1) = -\frac{K^{1-\alpha}-1}{1-\alpha}\mathbf{1}$, a constant vector. Hence,
\begin{align*}
    \langle \nabla\phi_\alpha(x_1), y - x_1\rangle = -\frac{K^{1-\alpha}-1}{1-\alpha}\sum_{i=1}^K\left(y_i - \frac{1}{K}\right) = 0,
\end{align*}
since $\sum_i y_i = \sum_i K^{-1} = 1$, we have
\begin{align*}
    \phi_\alpha(y) - \phi_\alpha(x_1) 
    &= -\frac{1}{\alpha(1-\alpha)}\sum_{i=1}^K\left[(y_i^\alpha - \alpha y_i - (1-\alpha)) - (K^{-\alpha} - \alpha K^{-1} - (1-\alpha))\right]\\
    &= -\frac{1}{\alpha(1-\alpha)}\sum_{i=1}^K\left[y_i^\alpha - K^{-\alpha} - \alpha\left(y_i -  K^{-1}\right)\right]\\
    &= -\frac{1}{\alpha(1-\alpha)}\left[\sum_{i=1}^K(y_i^\alpha - K^{-\alpha}) - \alpha\left(\sum_{i=1}^K y_i - 1\right)\right]\\
    &= -\frac{1}{\alpha(1-\alpha)}\sum_{i=1}^K\left[y_i^\alpha - K^{-\alpha}\right],
\end{align*}
where the last step uses $\sum_i y_i = 1$. Since the gradient term vanishes, we conclude
\begin{align*}
    D_{\phi_\alpha}(y,x_1) = -\frac{1}{\alpha(1-\alpha)}\sum_{i=1}^K\left[y_i^\alpha - K^{-\alpha}\right]. \qquad\qedhere
\end{align*}

    \noindent For $\alpha\in (0,\frac{1}{3})$, observe that
    \begin{align}
    D_{\phi_\alpha}(y,x_1)=-\frac{1}{\alpha(1-\alpha)}\sum_{i=1}^K \left[y_i^\alpha-K^{-\alpha}\right]\le\frac{K}{\alpha(1-\alpha)} \left[\left(K^{-1}\right)^\alpha -\varepsilon^\alpha\right]
    \end{align}
    Applying Mean Value Theorem~\citep[\text{Theorem 5.9}]{rudin1976principles} on functions
    \begin{align*}
        f_1(\alpha)=\left(K^{-1}\right)^\alpha-\varepsilon^\alpha\qquad\text{and}\qquad f_2(\alpha)=\alpha(1-\alpha)
    \end{align*}
    we get, for some $c_{\alpha}\in (0,\alpha)$,
    \begin{align}\label{eqn:d-phi-alto}
        \frac{\left(K^{-1}\right)^\alpha -\varepsilon^\alpha}{\alpha(1-\alpha)}=\frac{f_1(\alpha)-f_1(0)}{f_2(\alpha)-f_2(0)}=\frac{f_1^\prime(c_\alpha)}{f_2^\prime(c_\alpha)} =\frac{\log\left( K^{-1}\right)\cdot K^{-c_\alpha}-\log\varepsilon\cdot\varepsilon^{c_\alpha}}{1-2c_\alpha}\le 3\log\left(\frac{1}{\varepsilon}\right)
    \end{align}
    The last inequality above follows from the observation that $c_{\alpha}<\alpha<\frac{1}{3}<1$. Hence, the claim holds for $\alpha\in\left(0,\frac{1}{3}\right)$. Finally, for $\alpha\in\left[\frac{1}{3},1\right)$, we have
    \begin{align*}
        D_{\phi_\alpha}(y,x_1)=-\frac{1}{\alpha(1-\alpha)}\sum_{i=1}^K \left[y_i^\alpha-K^{-\alpha}\right]\le\frac{3(K^{1-\alpha}-1)}{(1-\alpha)}
    \end{align*}
    By Mean Value Theorem, for some $0<v_\alpha<1-\alpha<1$, we conclude
    \begin{align}\label{eqn:d-phi-alte}
        \frac{K^{1-\alpha}-1}{1-\alpha}=\log K\cdot K^{v_\alpha}\le K \log K
    \end{align}
    Overall, by combining equations~\eqref{eqn:d-phizo},~\eqref{eqn:d-phio},~\eqref{eqn:d-phi-alto} and~\eqref{eqn:d-phi-alte} we conclude that for any $\alpha\in[0,1]$, $D_{\phi_\alpha}(y,x_1)\le C_\alpha(K,\varepsilon)$.

\subsection{Proof of Lemma~\ref{lemma:Eone-comp}}

\noindent We begin the proof from the following observation.
\begin{align}
\notag
    | \mathcal{D}_{T}| \ 
    &\leq \ \bigg| \sqrt{\frac{\Exs \left[ T \xbta\right]}{ n_{a,T}}} - 1 \bigg| \cdot | Q_{T} | \\[8pt] \notag
    & \leq \ \bigg| \frac{\Exs \left[  \xbta\right]}{ \xbta } - 1 \bigg| \times \left( \sqrt{\frac{\Exs \left[ T \xbta\right]}{ n_{a,T}}} + 1\right)^{-1} \cdot | Q_{T} | \\[8pt]
    &\ \leq \ \bigg| \frac{\Exs \left[  \xbta\right]}{ \xbta } - 1 \bigg| \cdot | Q_{T} | 
\end{align}

\noindent Let $m_T$ be any arbitrary sequence of positive real numbers diverging to infinity. Define event $\mathcal{M}(T) := \{ |\mathcal{Q}_T| < m_T \}$  and recall that $\Eone : = \{ | \mathcal{D}_{1,T}| \leq \kappa_{1,T} \}$. We are interested in quantifying the probability $\prob(\Eone^c)$. Observe that,
\begin{align*}
   \prob(\Eone^c) 
   &\leq  \prob(\Eone^c, \MEone ) \ + \  \prob(\Eone^c, \MEone^c ) \\[8pt]
   &\leq \prob \left( \bigg| \frac{\Exs \left[  \xbta\right]}{ \xbta } - 1 \bigg| \cdot | Q_{T} | > \kappa_{1,T}, \MEone \right) + \prob( \MEone^c ) \\[8pt]
   & \leq \prob \left( \bigg| \frac{\Exs \left[  \xbta\right]}{ \xbta } - 1 \bigg|  > \frac{\kappa_{1,T}}{m_T}, \MEone \right) + \prob( \MEone^c ) \\[8pt]
   & \leq \prob \left( \bigg| \frac{\Exs \left[  \xbta\right]}{ \xbta } - 1 \bigg|  > \frac{\kappa_{1,T}}{m_T} \right) + \prob( \MEone^c )
\end{align*}

\noindent Fix a $\delta>0$ and define event $\Etwo := \{ \xbta/\Exs[\xbta] \in (1-\delta, 1+\delta)\}$. Hence we can decompose further as follows
\begin{align*}
  \prob \left( \bigg| \frac{\Exs \left[  \xbta\right]}{ \xbta } - 1 \bigg|  > \frac{\kappa_{1,T}}{m_T} \right) 
  &\leq \prob \left( \bigg| \frac{\Exs \left[  \xbta\right]}{ \xbta } - 1 \bigg|  > \frac{\kappa_{1,T}}{m_T}, \Etwo \right) \ + \ \prob(\Etwo^c) \\[8pt]
  & = \prob \left( \bigg| \frac{\Exs \left[  \xbta\right]}{ \xbta } \bigg|  \bigg|\frac{\xbta}{ \Exs \left[  \xbta\right] }  - 1 \bigg|  > \frac{\kappa_{1,T}}{m_T}, \Etwo \right) \ + \ \prob(\Etwo^c) \\[8pt]
  & \leq  \prob \left(   \bigg|\frac{\xbta}{ \Exs \left[  \xbta\right] }  - 1 \bigg|  > \frac{\kappa_{1,T} \  (1-\delta)}{m_T}, \Etwo \right) \ + \ \prob(\Etwo^c) 
\end{align*}

\noindent By combining the last two equations and applying Markov's theorem, we have
\begin{align} \label{eqn:Eone-compo}
\notag
  \prob(\Eone^c) 
  &\leq \frac{m_T}{\kappa_{1,T} \ (1- \delta)} \Exs \left[ \bigg|\frac{\xbta}{ \Exs \left[  \xbta\right] }  - 1 \bigg| \right] \ + \ \prob(\Etwo^c) \ + \  \prob(\MEone^c) \\[8pt]
  & \leq \frac{m_T \ \rate}{\kappa_{1,T} \ (1- \delta)} \ + \ \prob(\Etwo^c) \ + \  \prob(\MEone^c)
\end{align}

\noindent Now observe that,
\begin{align} \label{eqn:Eone-compto}
    \prob(\Etwo^c) 
    &= \prob \left(\bigg| \frac{\xbta}{\Exs[\xbta]} - 1 \bigg| > \delta \right)
    \leq \frac{1}{\delta} \cdot \Exs \left[\bigg| \frac{\xbta}{\Exs[\xbta]} - 1 \bigg| \right]
    \leq \frac{2}{\delta} \cdot \rate
\end{align}

\noindent The last inequality holds by applying Markov's inequality along with  equation~\eqref{eqn:Lto-bd} and Lemma~\ref{lemma:Ltwo-Reg} in the proof of Lemma~\ref{lemma:berry-qt}. Now,
\begin{align} \label{eqn:Eone-compte}
\notag
    \prob(\MEone^c) 
    &= \prob ( |\mathcal{Q}_T| > m_T) \\[8pt] \notag
    &= \prob ( \mathcal{Q}_T > m_T) + \prob ( \mathcal{Q}_T <- m_T) \\[8pt] \notag
    & \overset{(*)}{\leq} 1-\Phi(m_T) + \Phi(-m_T) \ + \ 2 \cdot \rate^{1/3} \\[8pt] \notag
    &= 2 \left[1-\Phi(m_T) \right] \ + \ 2 \cdot \rate^{1/3} \\[8pt]
    &  \overset{(**)}{\leq} \frac{2 \phi(m_T)}{m_T} \ + \ 2 \cdot \rate^{1/3} 
\end{align}

\noindent Inequality $(**)$ above follows from \emph{Mill's inequality} whereas inequality $(*)$ holds by observing that Lemma~\ref{lemma:berry-qt} implies,
\begin{align*}
    \sup_{x} |\prob(\mathcal{Q}_T < x) - \Phi(x)| \leq \Psi^{1/3}_T \quad \text{and hence,} \quad
    \sup_{x} |\prob(\mathcal{Q}_T > x) - [1-\Phi(x)]| \leq \Psi^{1/3}_T 
\end{align*}

\noindent Therefore, by combining equations~\eqref{eqn:Eone-compo},~\eqref{eqn:Eone-compto} and~\eqref{eqn:Eone-compte} and fixing $\delta = 1/2$ we have,
\begin{align*}
     \Prob(\Eone^c)
     \leq  \frac{2m_T \ \rate}{\kappa_{1,T} } \ + \ \frac{2 \phi(m_T)}{m_T} \ + \  4 \cdot \rate  + \ 2 \cdot \rate^{1/3}
\end{align*}

\subsection{Proof of Lemma~\ref{lemma:jojto-upp-bdd}}

\noindent Let us recall that
\begin{align}
   \mathcal{J}_{1,T}:= E\left[  \left|\frac{n_{a,T}}{\overline p_{T,a}}-1\right|\right] \quad \text{and,} \quad  \mathcal{J}_{2,T}:= E\left[  \left|\frac{\overline p_{T,a}}{\xtas}-1\right|\right]
\end{align}

\noindent Define $\mathcal{Z}_{1,T} := \frac{n_{a,T}}{\overline p_{T,a}}-1$ and $\mathcal{Z}_{2,T} := \frac{\overline p_{T,a}}{\xtas}-1$. Then it follows that,
\begin{align}
    E\left[  \left|\frac{n_{a,T}}{T\xtas}-1\right|\right] 
    = E\left[  \left|\left(\mathcal{Z}_{1,T} + 1 \right) \left(\mathcal{Z}_{2,T} + 1 \right)-1\right|\right] 
    \leq  E\left[  \left|\mathcal{Z}_{1,T} \right|\right] + E\left[  \left|\mathcal{Z}_{2,T} \right|\right] + E\left[  \left|\mathcal{Z}_{1,T} \mathcal{Z}_{2,T} \right|\right]
\end{align}

\noindent Note that $E\left[  \left|\mathcal{Z}_{i,T} \right|\right] = \mathcal{J}_{i,T}$ for $i=1,2$ and hence it suffices to prove that $E\left[  \left|\mathcal{Z}_{1,T} \mathcal{Z}_{2,T} \right|\right]$ is $o(\mathcal{J}_{1,T})$. Consider two positive real valued sequences $\delta_T$ and $\delta_{2,T}$ be the sequences chosen in the proof of Theorem~\ref{thm:berry-i}, and note the following string of inequalities
\begin{align*}
   &E\left[  \left|\mathcal{Z}_{1,T} \mathcal{Z}_{2,T} \right|\right] \\[8pt]
   &\overset{(i)}{\leq} \delta_{2,T} \mathcal{J}_{1,T} + \frac{1}{\vars} E\left[ |\mathcal{Z}_{1,T}| \mathbf 1
\left\{
\frac{n_{T,a}}{T \overline{p}_{T,a}}
\notin (1-\delta_{2,T},1+\delta_{2,T})
\right\}\right] \\[8pt]
&\overset{}{\leq} \delta_{2,T} \mathcal{J}_{1,T} + \frac{\delta_{T}}{\vars} 
\prob\left(
\frac{n_{T,a}}{T \overline{p}_{T,a}}
\notin (1-\delta_{2,T},1+\delta_{2,T})
\right)\\[8pt]
& \quad + \ \frac{1}{\vars^2} 
\prob\left(
\frac{n_{T,a}}{T \overline{p}_{T,a}}
\notin (1-\delta_{2,T},1+\delta_{2,T}), \frac{\overline p_{T,a}}{\xtas}
\notin (1-\delta_T,1+\delta_T) \right)\\[8pt]
&\leq \delta_{2,T} \mathcal{J}_{1,T} + \frac{\delta_{T}}{\vars} 
\prob\left(
\frac{n_{T,a}}{T \overline{p}_{T,a}}
\notin (1-\delta_{2,T},1+\delta_{2,T})
\right) + \frac{1}{\vars^2} 
\prob\left(
\frac{n_{T,a}}{T \overline{p}_{T,a}}
\notin (1-\delta_{2,T},1+\delta_{2,T}) \right)
\end{align*}

\noindent In the above chain of inequalities, $(i)$ follows from that fact that both $n_{a,T}$ and $T \overline{p}_{T,a}$ are bounded by $T$ and hence $|n_{a,T} - T \overline{p}_{T,a}|/T\overline{p}_{T,a}$ is bounded above by $T/T\vars$ and hence $T$ gets eliminated. Now we have shown in the proof of Theorem~\ref{thm:berry-i}, for the choice of $\delta_{2,T} = 2\sqrt{\frac{\log T}{\sqrt{T}}} +  \frac{\log T}{\sqrt{T}}$ we have
\begin{align*}
  \prob\left(
\frac{n_{T,a}}{T \overline{p}_{T,a}}
\notin (1-\delta_{2,T},1+\delta_{2,T})
\right) \leq  \frac{2 }{ T^{\hTto}}  \quad \text{where} \ \vars = \frac{\hTto}{\sqrt{T}}.
\end{align*}

\noindent Hence, it follows that,
\begin{align*}
    E\left[  \left|\mathcal{Z}_{1,T} \mathcal{Z}_{2,T} \right|\right]
    \leq \left(2\sqrt{\frac{\log T}{\sqrt{T}}} +  \frac{\log T}{\sqrt{T}} \right) \mathcal{J}_{1,T} + \rate \frac{2 \sqrt{T} }{h_T T^{\hTto}} + \frac{2 T }{h^2_T T^{\hTto}} 
\end{align*}

\noindent Since these additional terms are of much smaller order compared to $\mathcal{J}_{1,T} + \mathcal{J}_{2,T}$ we ignore them, and conclude that our result holds.

\subsection{Proof of Lemma~\ref{lemma:napt-concent}}

\noindent Let $U_t:=\I\{A_t=a\}$. We have $\E[U_t\;|\;\F_{t-1}]=p_{T,a}$. Let $\beta_T:=\sum_{t=1}^T p_{T,a}$. Now consider the following modified Freedman's inequality.
\begin{lemma}
    \label{lem:freedman_ineq}
    {(Freedman's Inequality, modified)} Let $(X_t,\F_t)_{t\ge0}$ be a martingale with $X_0=0$. Assume $Y_t:=X_t-X_{t-1}\le b$ almost surely for some $b>0$ for all $t$. Let $V_T:=\sum_{t=1}^T \E[Y_t^2|\F_{t-1}]$. Then, for any $F_{T-1}$ measurable $x,v>0$ and for all $\pen\in\left(0,\frac{1}{b}\right)$, the following holds almost surely
    \begin{align*}
        \prob\left(X_T\ge x\text{ and }V_T\le v\Big|\F_{T-1}\right)\le M_{T-1}(\pen)\exp\left\{-\pen x+\frac{\pen^2 v}{1-\pen b}\right\}
    \end{align*}
    where $M_T(\pen)=\exp\left\{\pen X_T-\frac{\pen^2 V_T}{1-\pen b}\right\}$.
\end{lemma}

\noindent In Lemma~\ref{lem:freedman_ineq}, set $X_t=\sum_{i=1}^t (U_i-x_{i,a})=n_{a,T}-\beta_T$ with $X_0=0$. Clearly $\{X_t\}_{t\ge0}$ is a martingale and $Y_t:=X_t-X_{t-1}=U_t-p_{T,a}\le1$. Moreover,
    \begin{align}
        \label{eq:pred cov sum bd}
        V_T=\sum_{t=1}^T\E[Y_t^2|\F_{t-1}]=\sum_{t=1}^T p_{T,a}(1-p_{T,a})\le\sum_{t=1}^T p_{T,a}=\beta_T
    \end{align}
    Pick arbitrary $\varepsilon_1>0$. With $b=1$, $x=\varepsilon_1 \beta_T$ and $v=\beta_T$ in Lemma~\ref{lem:freedman_ineq}, we have, for all $\pen\in(0,1)$, almost surely
    \begin{align*}
        \prob\left(n_{a,T}-\beta_T\ge \varepsilon_1\beta_T, V_T\le \beta_T\Big|\F_{t-1}\right)\le M_{T-1}(\pen)\exp\left\{-\pen \varepsilon_1\beta_T+\frac{\pen^2 \beta_T}{1-\pen}\right\}
    \end{align*}
    Let $\zeta(\pen):=-\pen \varepsilon_1+\frac{\pen^2}{1-\pen}$. For $\pen \in(0,1)$, this is minimized at $\lambda_0=1-\frac{1}{1+\varepsilon_1}$ and $\zeta(\lambda_0)=-\left(\sqrt{1+\varepsilon_1}-1\right)^2$. Therefore, using \eqref{eq:pred cov sum bd}, it follows
    \begin{align*}
        \prob\left(n_{a,T}-\beta_T\ge \varepsilon_1\beta_T\Big|\F_{t-1}\right)&=\prob\left(n_{a,T}-\beta_T\ge \varepsilon_1\beta_T, V_T\le \beta_T\Big|\F_{t-1}\right)\\
        &\le M_{t-1}(\lambda_0)\exp\left\{-\left(\sqrt{1+\varepsilon_1}-1\right)^2\beta_T\right\}\\
        &\le M_{t-1}(\lambda_0)\exp\left\{-\left(\sqrt{1+\varepsilon_1}-1\right)^2 T\varepsilon\right\}
    \end{align*}
    Therefore,
    \begin{align}
        \label{eq:freedman as upper bd}
        \prob\left(\frac{n_{a,T}}{\beta_T}-1\ge \varepsilon_1\Big|\F_{t-1}\right)\le M_{t-1}(\lambda_0)\exp\left\{-\left(\sqrt{1+\varepsilon_1}-1\right)^2 T\varepsilon\right\}
    \end{align}
    Similarly, applying the same argument to $\sum_{i=1}^t -(U_t-p_{T,a})$ yields
    \begin{align}
        \label{eq:freedman as lower bd}
        \prob\left(\frac{n_{a,T}}{\beta_T}-1\le -\varepsilon_1\Big|\F_{t-1}\right)\le M_{t-1}(\lambda_0)\exp\left\{-\left(\sqrt{1+\varepsilon_1}-1\right)^2 T\varepsilon\right\}
    \end{align}
    Combining \eqref{eq:freedman as upper bd} and \eqref{eq:freedman as lower bd}, we have
    \begin{align}
        \label{eq:freedman as bd}
        \prob\left(\left|\frac{n_{a,T}}{\beta_T}-1\right|\ge \varepsilon_1\Big|\F_{t-1}\right)\le 2M_{t-1}(\lambda_0)\exp\left\{-\left(\sqrt{1+\varepsilon_1}-1\right)^2 T\varepsilon\right\}
    \end{align}
    The bounds in \eqref{eq:freedman as upper bd}, \eqref{eq:freedman as lower bd} and \eqref{eq:freedman as bd} hold almost surely. Finally, taking expectation and using the fact that $M_{t}(\lambda_0)$ is a super-martingale, it follows,
    \begin{align*}
        \prob\left(\left|\frac{n_{a,T}}{\beta_T}-1\right|\ge \varepsilon_1\right)&\le 2\E\left[M_{t-1}(\lambda_0)\right]\exp\left\{-\left(\sqrt{1+\varepsilon_1}-1\right)^2 T\varepsilon\right\}\\
        &\le 2\E\left[M_{0}(\lambda_0)\right]\exp\left\{-\left(\sqrt{1+\varepsilon_1}-1\right)^2 T\varepsilon\right\}\\
        &=2\exp\left\{-\left(\sqrt{1+\varepsilon_1}-1\right)^2 T\varepsilon\right\}
    \end{align*}

\subsection{Proof of Lemma~\ref{lem:master eqn}}

\noindent We begin by analyzing the quantity $\langle\setpz \widetilde{\ell}_t,x_t-y\rangle$. Recall from Algorithm~\ref{alg:st-exp3} that $\setpz \widetilde{\ell}_t = \nabla\phi(x_t)-\nabla\phi(z_{t+1})$, which lead to the following string of inequalities:
    \begin{align*}
        \langle\setpz \widetilde{\ell}_t,x_t-y\rangle&=\langle\nabla\phi(x_t)-\nabla\phi(z_{t+1}),x_t-y\rangle\\
        &=D_\phi(y,x_t)+D_\phi(x_t,z_{t+1})-D_\phi(y,z_{t+1})\\
        &\overset{(a)}{\le}D_\phi(y,x_t)+D_\phi(x_t,z_{t+1})-D_\phi(y,x_{t+1})-D_\phi(x_{t+1},z_{t+1})\\
        &\le \left[D_\phi(y,x_t)-D_\phi(y,x_{t+1})\right]+D_\phi(x_t,z_{t+1})
    \end{align*}
    Here, inequality $(a)$ follows by observing $D_\phi(y,z_{t+1})\ge D_\phi(y,x_{t+1})+D_\phi(x_{t+1},z_{t+1})$. Next, we bound $D_\phi(x_t,z_{t+1})$. Now consider the following lemma.
    \begin{lemma}
\label{lem:primal-dual-bregman}
Let $\phi$ be a Legendre function
\citep[\text{Section 26}]{Rockafellar1970}.
Then for all $x,y \in \operatorname{int}(\operatorname{dom}\phi)$,
\[
D_\phi(x,y)
=
D_{\phi^\ast}\bigl(\nabla \phi(y),\, \nabla \phi(x)\bigr).
\]
Furthermore,  if $\phi:\mathfrak{D}_\phi\to\R$ is  twice differentiable, then there exists $\alpha^\ast \in(0,1)$ such that
    \begin{align*}
        D_\phi(x,y)=\frac{1}{2}(x-y)^T\left[\nabla^2\phi(\alpha^\ast x+(1-\alpha^\ast)y)\right](x-y)
    \end{align*}
\end{lemma}

    \noindent Using Lemma~\ref{lem:primal-dual-bregman} we conclude
    \begin{align*}
        D_\phi(x_t,z_{t+1})=\frac{1}{2}\left(\nabla\phi(z_{t+1})-\nabla\phi(x_t)\right)^T \left[\nabla^2{\phi^\ast}(\omega_t)\right]\left(\nabla\phi(z_{t+1})-\nabla\phi(x_t)\right)
    \end{align*}
    where $\phi^\ast(y):=\sup_{x\in R^K_{>0}}\left\{\langle y,x\rangle-\phi(x)\right\}$ is the dual map corresponding to $\phi$. Furthermore, since $\nabla\phi(z_{t+1})\le\nabla\phi(x_t)$, it implies $\omega_t\le \nabla\phi(x_t)$ component wise. For $\alpha\in[0,1]$, choosing $\phi_\alpha$ as the mirror map yields
    \begin{align*}
        \left[\nabla^2{\phi_\alpha^\ast}(\omega_t)\right]_{j,j}\le x_{t,j}^{2-\alpha}\le x_{t,j}
    \end{align*}
    This implies $\nabla^2{\phi_\alpha^\ast}(\omega_t) \preceq\text{Diag}(x_{t,1},\ldots,x_{t,K})$ for any choice of $\alpha\in[0,1]$. Hence,
    \begin{align*}
        D_{\phi_\alpha}(x_t,z_{t+1})&\le \frac{1}{2}\left|\left|\nabla{\phi_\alpha}(x_t)-\nabla{\phi_\alpha}(z_{t+1})\right|\right|_{x_t,\alpha}^2\\
        &=\frac{1}{2}\left|\left|\setpz \widehat{\ell}_t+\setpz\pen\nabla R(x_t)\right|\right|_{x_t,\alpha}^2\\
        &\le \setpz^2\left|\left|\widehat{\ell}_t\right|\right|_{x_t,\alpha}^2+\setpz^2\pen^2\left|\left|\nabla R(x_t)\right|\right|_{x_t,\alpha}^2 
    \end{align*}
    Overall, we arrive at
    \begin{align*}
        \langle\setpz \widetilde{g}_t,x_t-y\rangle\le\left[D_{\phi_\alpha}(y,x_t)-D_{\phi_\alpha}(y,x_{t+1})\right]+\setpz^2\left|\left|\widehat{\ell}_t\right|\right|_{x_t,\alpha}^2+\setpz^2\pen^2\left|\left|\nabla R(x_t)\right|\right|_{x_t,\alpha}^2
    \end{align*}
    On taking expectation, the left hand side simplifies to $\setpz\;\E\left[\langle\nabla f_{\lmep}(x_t),x_t-y\rangle\right]$. Convexity of $f_{\lmep}$ implies $\langle\nabla f_{\lmep}(x_t),x_t-y\rangle$ is lower bounded by $f_{\lmep}(x_t)-f_{\lmep}(y)$. Averaging over $t=1$ to $T$, we get 
    \begin{align}
        \label{eq:master-eqn-gen}
        \frac{1}{T}\sum_{t=1}^T\E\left[f_{\lmep}(x_t)-f_{\lmep}(y)\right]&\le \E\left[\frac{D_{\phi_\alpha}(y,x_1)}{\setpz T}+\frac{\setpz}{T}\sum_{t=1}^T \left|\left|\widehat{\ell}_t\right|\right|_{x_t,\alpha}^2+\frac{\setpz\pen^2}{T}\sum_{t=1}^T \left|\left|\nabla R(x_t)\right|\right|_{x_t,\alpha}^2\right]
    \end{align}
    The lemma follows from the observation that $\frac{1}{T}\sum_{t=1}^T\E\left[f_{\lmep}(x_t)-f_{\lmep}(y)\right]\ge\E\left[f_{\lmep}(\xT)-f_{\lmep}(y)\right]$.

\newpage

\section{Proof of Robust Inference for corrupted samples}
\label{proof:normality-corr}

\noindent This section provides the proof of Theorems~\ref{thm:normality-corruption} and~\ref{regret : corrupted} stated in the main section.

\subsection{Proof of Theorem~\ref{thm:normality-corruption}}
    
  \noindent Observe that for our choice of $\pen,\vars$ and $\setpz$ the original rewards $\ell_t$ satify the central limit theorem:
    \begin{equation}
        \frac{1}{\sqrt{n_{a,T}}}
        \sum_{t=1}^{T}
        \left(\ell_t - \mu_a\right)\,\mathbf{1}\{A_t = a\}
        \;\indist\;
        \mathcal{N}\!\left(0, \sigma_a^2\right),
        \qquad \forall\, a \in \{1,2,\ldots,K\}
        \label{decomposition_outline}
    \end{equation}
    Let us rewrite $\ell^{\mathrm{c}}_t$ as $\ell^{\mathrm{c}}_t = \ell_t + c_t$, where $c_t$ is the added corruption. For arm $a$, note that,
    \begin{align*}
        \frac{1}{\sqrt{n_{a,T}}}\sum_{t=1}^{T}\bigl(\ell^{\mathrm{c}}_t - \mu_a\bigr)\mathbf{1}\{A_t = a\}&=\frac{1}{\sqrt{n_{a,T}}}\sum_{t=1}^{T}\bigl(\ell_t + c_t - \mu_a\bigr)\mathbf{1}\{A_t = a\} \\
        &=\frac{1}{\sqrt{n_{a,T}}}\sum_{t=1}^{T}\bigl(\ell_t-\mu_a\bigr)\mathbf{1}\{A_t = a\}+\frac{1}{\sqrt{n_{a,T}}}\sum_{t=1}^{T}c_t\,\mathbf{1}\{A_t = a\}
    \end{align*}
    Let $C_{a,T}:=\sum_{t=1}^{T}c_t\,\mathbf{1}\{A_t = a\}$. We have that $C_T = o\left(T^\beta\right)$ and $\mathbb{E}\left[\left|C_{a,T}\right|\right] \le C_T$. For $\frac{1}{2}>\beta$, we have $\frac{1}{2}+\beta>2\beta$ and hence by Markov's Inequality, $C_{a,T} = o_P\left(\sqrt{T\;x_{\beta,a}}\right)$. Using \eqref{decomposition_outline} and applying Slutsky's theorem, it follows
    \begin{equation}
        \label{eq:normality-corr}
        \frac{1}{\sqrt{n_{a,T}}}
        \sum_{t=1}^{T}
        \bigl(\ell^{\text{c}}_t - \mu_a\bigr)\,\mathbf{1}\{A_t = a\}
        \;\indist\;
        \mathcal{N}\!\left(0, \sigma_a^2\right),
        \qquad \forall\, a \in \{1,2,\ldots,K\} 
    \end{equation}
    \paragraph{Consistency of $\widehat{\sigma}_{a,T}$}: 
    $\widehat{\sigma}_{a,T} = \frac{1}{n_{a,T}}\sum_{t=1}^T  \bigl(\ell^{\text{c}}_t\bigr)^2\,\mathbf{1}\{A_t = a\}-\bigl(\frac{1}{n_{a,T}}\sum_{t=1}^T  \ell^{\text{c}}_t\,\mathbf{1}\{A_t = a\}\bigr)^2$. From \ref{eq:normality-corr}, we get that, $\frac{1}{n_{a,T}}\sum_{t=1}^T  \ell^{\text{c}}_t\,\mathbf{1}\{A_t = a\} \inprob \mu_a$. Thus, it suffices to show that 
    \[\frac{1}{n_{a,T}}\sum_{t=1}^T  \bigl(\ell^{\;\mathrm{c}}_t\bigr)^2\,\mathbf{1}\{A_t = a\} \inprob \sigma^2_a + \mu_a^2\]
    Observe that, 
    \begin{align*}
        \frac{1}{n_{a,T}}\sum_{t=1}^T  \bigl(\ell^{\;\mathrm{c}}_t\bigr)^2\,\mathbf{1}\{A_t = a\}
        = \frac{1}{n_{a,T}}\sum_{t=1}^T  \bigl(\ell_t^2 + c_t^2 + 2c_t\ell_t\bigr)\,\mathbf{1}\{A_t = a\}.
    \end{align*}
    From \citet[Lemma 3]{laiwei82}, we get that, \[\frac{1}{n_{a,T}}\sum_{t=1}^T  \bigl(\ell_t\bigr)^2\,\mathbf{1}\{A_t = a\} \inprob \sigma_a^2 + \mu_a^2\]
    It is enough to show that, $\frac{1}{n_{a,T}}\sum_{t=1}^T  c_t^2\,\mathbf{1}\{A_t = a\} \inprob0$. Then, by Cauchy-Schwarz Inequality, $\frac{1}{n_{a,T}}\sum_{t=1}^T  \bigl(\ell_tc_t\bigr)\,\mathbf{1}\{A_t = a\}\le\sqrt{\frac{\sum \ell_t^2\mathbf{1}\{A_t = a\}}{n_{a,T}}}\cdot\sqrt{\frac{\sum c_t^2\mathbf{1}\{A_t = a\}}{n_{a,T}}} \inprob 0$. Now, 
    \[\frac{1}{n_{a,T}}\sum_{t=1}^T  c_t^2\,\mathbf{1}\{A_t = a\} \le \frac{2}{n_{a,T}}\sum_{t=1}^T  |c_t|\,\mathbf{1}\{A_t = a\} := M(T) \text{ (say)}\] 
    We thus have $\mathbb{E}[M(T)]\le 2C_T = \mathcal{O}\left(T^{\beta}\right)$ with $\beta < \frac{1}{2}$ (from Theorem~\ref{thm:normality-corruption} assumption) and we deduce $\frac{C_T}{\sqrt{Tx_{\gamma,a}}}\le \frac{C_T}{\sqrt{T\varepsilon}} \to 0.$ The result is obtained by applying Markov’s inequality. This completes the proof.

\subsection{Proof of Lemma~\ref{regret : corrupted}}
\label{proof:regret-corr}

  \noindent  Imitating the proof of Lemma~\ref{lem:master eqn}, we get,
    \begin{align*}
        \langle\setpz \widetilde{\ell}^{\mathrm{c}}_t,x_t-y\rangle\le\left[D_\phi(y,x_t)-D_\phi(y,x_{t+1})\right]+\setpz^2\left|\left|\widehat{\ell}^{\mathrm{c}}_t\right|\right|_{x_t,\alpha}^2+\setpz^2\pen^2\left|\left|\nabla R(x_t)\right|\right|_{x_t,\alpha}^2
    \end{align*}
    This implies
    \begin{align*}
        \langle\setpz \widetilde{\ell}_t,x_t-y\rangle&\le\left[D_\phi(y,x_t)-D_\phi(y,x_{t+1})\right]+\setpz^2\left|\left|\widehat{\ell}^{\mathrm{c}}_t\right|\right|_{x_t,\alpha}^2\\
        &\hspace{4cm}+\setpz^2\pen^2\left|\left|\nabla R(x_t)\right|\right|_{x_t,\alpha}^2 + \setpz\langle\widetilde{\ell}_t-\widetilde{\ell}^{\mathrm{c}}_t,x_t-y\rangle
    \end{align*}
    Next, we take a conditional expectation and bound the additional inner product term as follows
    \begin{align*}
        \langle\widetilde{\ell_t}-\widetilde{\ell_t}^{\mathrm{c}},x_t-y\rangle&\le \|\widetilde{\ell_t}-\widetilde{\ell_t}^{\mathrm{c}}\|_{\infty}\times\|x_t-y\|_{1}\\
        &\le \|\widetilde{\ell_t}-\widetilde{\ell_t}^{\mathrm{c}}\|_{\infty} \times\left(\sum_{i=1}^K(x_{t,i}-\varepsilon)+\sum_{i=1}^K(y_{i}-\varepsilon)\right)\\
        &\le 2\|\widetilde{\ell_t}-\widetilde{\ell_t}^{\mathrm{c}}\|_{\infty}
    \end{align*}
    By comparing with the proof of Lemma~\ref{lemma:actual regret bd} and taking expectation and averaging over $t=1$ to $T$, we arrive at
    \begin{align}
        \label{eq:master-eqn-corr}
        \E\left[f_{\lmep}(\xT)-f_{\lmep}(y)\right]&\le \frac{C_\alpha(K,\varepsilon)}{\setpz T}+\setpz K+\frac{\setpz \pen^2}{\varepsilon^2}+\frac{2 C_T}{T}
    \end{align}
       
    \noindent Note that from \eqref{eq:master-eqn-corr}, we have
\begin{align*}
    \E[\langle\mu,\xT-y\rangle]\le\frac{C_\alpha(K,T)}{\sqrt{T}}+\frac{K}{\sqrt{T}}+\frac{2C_T}{T}+\frac{\hTo^2\log^2 T}{T^{\frac{1}{2}+2\beta}}
\end{align*}
Therefore,
\begin{align*}
    \E[\langle\mu,\xT-p^\star\rangle]&\le\frac{C_\alpha(K,T)}{\sqrt{T}}+\frac{K}{\sqrt{T}}+\frac{\hTo^2\log^2 T}{T^{\frac{1}{2}+2\beta}}+\frac{2K}{T^{1-\beta}}\\
\end{align*}
For $\alpha\in\left[0,\frac{1}{3}\right)$, we have
\begin{align*}
    \regret(T)\le\left(4\log T  +\frac{\hTo^2\log^2 T}{T^{2\beta}}+2\; T^\beta\right)\cdot K\sqrt{T}
\end{align*}
Similarly, for $\alpha\in\left[\frac{1}{3},1\right]$, we get the following bound
\begin{align*}
    \regret(T)\le\left(4\log K  +\frac{\hTo^2\log^2 T}{T^{2\beta}}+2\; T^\beta\right)\cdot K\sqrt{T}
\end{align*}

\newpage

\section{Auxiliary Lemmas}
\label{sec:aux-lem}
\begin{lemma*}
Let $\phi$ be a Legendre function
\citep[\text{Section 26}]{Rockafellar1970}.
Then for all $x,y \in \operatorname{int}(\operatorname{dom}\phi)$,
\[
D_\phi(x,y)
=
D_{\phi^\ast}\bigl(\nabla \phi(y),\, \nabla \phi(x)\bigr).
\]
Furthermore,  if $\phi:\mathfrak{D}_\phi\to\R$ is  twice differentiable, then there exists $\alpha^\ast \in(0,1)$ such that
    \begin{align*}
        D_\phi(x,y)=\frac{1}{2}(x-y)^T\left[\nabla^2\phi(\alpha^\ast x+(1-\alpha^\ast)y)\right](x-y)
    \end{align*}
\end{lemma*}

\begin{proof}
Since $\phi$ is Legendre, the gradient map
\[
\nabla \phi :
\operatorname{int}(\operatorname{dom}\phi)
\to
\operatorname{int}(\operatorname{dom}\phi^\ast)
\]
is bijective with inverse $\nabla \phi^\ast$ \citep[Theorem~26.5]{Rockafellar1970}.
Using \textit{Fenchel--Young} duality,
\[
\phi(x) + \phi^\ast(\nabla \phi(x)) = \langle x, \nabla \phi(x) \rangle,
\]
and substituting this identity for both $x$ and $y$ in the definition of
the Bregman divergence yields the claim.  Fix $x,y\in\mathfrak{D}_\phi$. For $\alpha\in[0,1]$, consider $h(\alpha)=\phi(\alpha x+(1-\alpha)y)-\phi(y)-\langle\nabla \phi(y),\alpha(x-y)\rangle$. Using Taylor's Theorem, we have
    \begin{align*}
        h(1)=h(0)+h'(0)\times(1-0)+\frac{1}{2}h''(\alpha^\ast)\times(1-0)^2
    \end{align*}
    for some $\alpha^\ast\in (0,1)$. Now, the lemma follows from the observations that $h(1)=D_\phi(x,y)$, $h(0)=h'(0)=0$ and $h''(\alpha^\ast)=(x-y)^T\left[\nabla^2\phi(\alpha^\ast x+(1-\alpha^\star)y)\right](x-y)$.

\end{proof}

\begin{lemma*}
    {(Freedman's Inequality, modified)} Let $(X_t,\F_t)_{t\ge0}$ be a martingale with $X_0=0$. Assume $Y_t:=X_t-X_{t-1}\le b$ almost surely for some $b>0$ for all $t$. Let $V_T:=\sum_{t=1}^T \E[Y_t^2|\F_{t-1}]$. Then, for any $F_{T-1}$ measurable $x,v>0$ and for all $\pen\in\left(0,\frac{1}{b}\right)$, the following holds almost surely
    \begin{align*}
        \prob\left(X_T\ge x\text{ and }V_T\le v\Big|\F_{T-1}\right)\le M_{T-1}(\pen)\exp\left\{-\pen x+\frac{\pen^2 v}{1-\pen b}\right\}
    \end{align*}
    where $M_T(\pen)=\exp\left\{\pen X_T-\frac{\pen^2 V_T}{1-\pen b}\right\}$.
\end{lemma*}
\begin{proof}
    For all $y\le b$ and $\pen\in\left(0,\frac{1}{b}\right)$, using \textit{Taylor's Theorem}, it follows
    \begin{align*}
        e^{\pen y}\le 1+\pen y +\frac{\pen^2 y^2}{1-\pen b}
    \end{align*}
    Setting $y=Y_t$ above and taking conditional expectation, we have
    \begin{align*}
        \E\left[e^{\pen Y_t}\Big|\F_{t-1}\right]\le\E\left[1+\pen Y_t +\frac{\pen^2 Y_t^2}{1-\pen b}\Big|\F_{t-1}\right]=1+\frac{\pen^2}{1-\pen b}\E\left[Y_t^2\Big|\F_{t-1}\right]
    \end{align*}
    since $\E[Y_t|\F_{t-1}]=0$. Using, $1+\alpha\le e^\alpha$ for all $\alpha\in\R$, we conclude
    \begin{align}
        \label{eq:freedman sup mg exp bound}
        \E\left[e^{\pen Y_t}\Big|\F_{t-1}\right]\le\exp\left\{\frac{\pen^2}{1-\pen b}\E\left[Y_t^2\Big|\F_{t-1}\right]\right\}
    \end{align}
    Next, we prove that $\{M_T(\pen)\}_{T\ge1}$ is a super-martingale.
    \begin{align*}
        \E\left[M_T(\pen)\Big|\F_{T-1}\right]&=\E\left[\exp\left\{\pen X_T-\frac{\pen^2 V_T}{1-\pen b}\right\}\Big|\F_{T-1}\right]\\
        &=\E\left[\exp\left\{\pen (X_{T-1}+Y_T)-\frac{\pen^2 V_T}{1-\pen b}\right\}\Big|\F_{T-1}\right]\\
        &\overset{(i)}{=}\E\left[\exp\left\{\pen Y_T\right\}\Big|\F_{T-1}\right]\times\exp\left\{\pen X_{T-1}-\frac{\pen^2 V_T}{1-\pen b}\right\}\\
        &\overset{(ii)}{\le}\exp\left\{\frac{\pen^2}{1-\pen b}\E\left[Y_T^2\Big|\F_{T-1}\right]\right\}\times\exp\left\{\pen X_{T-1}-\frac{\pen^2 V_T}{1-\pen b}\right\}\\
        &=\exp\left\{\pen X_{T-1}-\frac{\pen^2}{1-\pen b}\left(V_T-\E[Y_T^2|\F_{T-1}]\right)\right\}\\
        &=\exp\left\{\pen X_{T-1}-\frac{\pen^2 V_{T-1}}{1-\pen b}\right\}\\
        &=M_{T-1}(\pen)
    \end{align*}
    \vspace{-3pt}
    Equality $(i)$ follows from the observation that $X_{T-1}$ and $V_T$ are $\F_{T-1}$-measurable, while inequality $(ii)$ is due to \eqref{eq:freedman sup mg exp bound}.\\
    Now, observe that $\{X_T\ge x, V_T\le v\}$ implies $\left\{M_t(\pen)\ge\exp\left\{\pen x-\frac{\pen^2 v}{1-\pen b}\right\}\right\}$. Since $x,v$ are $\F_{t-1}$-- measurable, by conditional Markov inequality,
    \begin{align*}
        \prob\left(X_T\ge x\text{ and }V_T\le v\Big|\F_{T-1}\right)&\le\prob\left(M_T(\pen)\ge\exp\left\{\pen x-\frac{\pen^2 v}{1-\pen b}\right\}\Big|\F_{T-1}\right)\\
        &\le\frac{\E\left[M_T(\pen)\Big|\F_{T-1}\right]}{\exp\left\{\pen x-\frac{\pen^2 v}{1-\pen b}\right\}}\\
        &\le M_{T-1}(\pen)\exp\left\{-\pen x+\frac{\pen^2 v}{1-\pen b}\right\}
    \end{align*}
    This completes the proof.
\end{proof}

\end{document}